\documentclass{article}

\usepackage{microtype}
\usepackage{graphicx}
\usepackage{subfigure}
\usepackage{booktabs} 
\usepackage{soul}
\usepackage{hyperref}
\usepackage{enumitem}

\usepackage[accepted]{icml2021-arxiv}

\usepackage{amsthm, amssymb}
\usepackage{quoting}
\usepackage{lineno}
\usepackage{comment}

\usepackage[noend]{algpseudocode}
\usepackage{algorithmicx}

\algnewcommand\algorithmicinput{\textbf{Input:}}
\algnewcommand\algorithmicoutput{\textbf{Output:}}
\algnewcommand\algorithmicinitialize{\textbf{initialize }}

\algnewcommand\Input{\item[\algorithmicinput]}%
\algnewcommand\Output{\item[\algorithmicoutput]}%
\algnewcommand\Initialize{\State\algorithmicinitialize}%

\makeatletter
\newcounter{algorithmicH}
\let\oldalgorithmic\algorithmic
\renewcommand{\algorithmic}{%
  \stepcounter{algorithmicH}
  \oldalgorithmic}
\renewcommand{\theHALG@line}{ALG@line.\thealgorithmicH.\arabic{ALG@line}}
\makeatother

\newtheorem{theorem}{Theorem}
\newtheorem{lemma}{Lemma}

\newtheorem{remark}{Remark}

\usepackage{amsthm,amsmath,bbm,amsfonts,amssymb}
\usepackage{dsfont} 
\usepackage{mathtools}
\usepackage{commath}
\mathtoolsset{showonlyrefs=false}

\usepackage{xspace}
\usepackage[T1]{fontenc}
\usepackage[utf8]{inputenc}
\usepackage[usenames,dvipsnames]{xcolor} 

\usepackage{hyperref,url}
\usepackage{cleveref} 

\usepackage[usenames,dvipsnames]{xcolor} 

\usepackage{wrapfig}
\usepackage[export]{adjustbox}
\usepackage{algorithm}

\usepackage{graphicx,tabularx}
\usepackage{multirow,hhline}
\usepackage{booktabs}

\usepackage{capt-of}

\usepackage{enumitem}

\allowdisplaybreaks 

\usepackage{pbox} 

\usepackage[export]{adjustbox}

\usepackage{pifont}
\usepackage{framed}
\usepackage[most]{tcolorbox}
\definecolor{kjgray}{rgb}{.7,.7,.7}

\newtheoremstyle{kjstyle}
{1ex} 
{\topsep} 
{\itshape} 
{} 
{\bfseries} 
{.} 
{.5em} 
{} 

\newtheoremstyle{kjstylenoitalic}
{1ex} 
{\topsep} 
{} 
{} 
{\bfseries} 
{.} 
{.5em} 
{} 

\tcbset{kjboxstyle/.style={title={},breakable,colback=white,enhanced jigsaw,boxrule=1.3pt,sharp corners,colframe=kjgray,boxsep=0pt,coltitle={black},attach title to upper={},left=.8ex,bottom=.4em}}    
\usepackage{mdframed}
\usepackage{lipsum}
\definecolor{kjgray}{rgb}{.7,.7,.7}
\makeatletter



\makeatletter
\renewcommand{\paragraph}{%
  \@startsection{paragraph}{4}%
  {\z@}{0.8em \@plus 1ex \@minus .2ex}{-1em}%
  {\normalfont\normalsize\bfseries}%
}
\makeatother

\newcounter{textcnt}

\newcommand\addtext[1]{%
  \stepcounter{textcnt}%
  \csgdef{text\thetextcnt}{#1}}

\newcounter{colnum}

\newcounter{Jidx}
\newcommand\dsymhelper[2]{
  \addtext{\hyperlink{#1}{#2}}%
  \blue{\hypertarget{#1}{#2}}%
}
\newcommand\dsym[1]{
  \stepcounter{Jidx}
  \xdef\tmpname{Jsym.\theJidx}
  \expandafter\dsymhelper\expandafter{\tmpname}{#1}
} 

\def\ddefloop#1{\ifx\ddefloop#1\else\ddef{#1}\expandafter\ddefloop\fi}
\def\ddef#1{\expandafter\def\csname c#1\endcsname{\ensuremath{\mathcal{#1}}}}
\ddefloop ABCDEFGHIJKLMNOPQRSTUVWXYZ\ddefloop
\def\ddef#1{\expandafter\def\csname b#1\endcsname{\ensuremath{{\boldsymbol{#1}}}}}
\ddefloop ABCDEFGHIJKLMNOPQRSUVWXYZabcdeghijklmnopqrtsuvwxyz\ddefloop  
\def\ddef#1{\expandafter\def\csname h#1\endcsname{\ensuremath{\widehat{#1}}}}
\ddefloop ABCDEFGHIJKLMNOPQRSTUVWXYZabcdefghijklmnopqrsuvwxyz\ddefloop 
\def\ddef#1{\expandafter\def\csname hc#1\endcsname{\ensuremath{\widehat{\mathcal{#1}}}}}
\ddefloop ABCDEFGHIJKLMNOPQRSTUVWXYZ\ddefloop
\def\ddef#1{\expandafter\def\csname t#1\endcsname{\ensuremath{\widetilde{#1}}}}
\ddefloop ABCDEFGHIJKLMNOPQRSTUVWXYZabcdefghijklmnopqrtsuvwxyz\ddefloop
\def\ddef#1{\expandafter\def\csname r#1\endcsname{\ensuremath{\mathring{#1}}}}
\ddefloop ABCDEFGHIJKLMNOPQRSTUVWXYZabcdefghijklmnopqrtsuvwxyz\ddefloop
\def\ddef#1{\expandafter\def\csname tc#1\endcsname{\ensuremath{\widetilde{\mathcal{#1}}}}}
\ddefloop ABCDEFGHIJKLMNOPQRSTUVWXYZ\ddefloop

\DeclareMathOperator*{\argmax}{arg~max}
\DeclareMathOperator*{\argmin}{arg~min}
\DeclareMathOperator{\EE}{\mathbb{E}}

\DeclareMathOperator{\PP}{\mathbb{P}}

\DeclareMathOperator{\one}{\mathds{1}}
\DeclareMathOperator{\Reg}{{\text{Reg}}}

\def\RR{{\mathbb{R}}}

\newcommand*\diff{\mathop{}\!\mathrm{d}}

\newcommand{\sr}[2]{ {\mathop{}\stackrel{#1}{#2}}\mathop{}\, }

\newcommand{\fr}[2]{ { \frac{#1}{#2} }}
\def\lt{\left}
\def\rt{\right}

\def\larrow{\ensuremath{\leftarrow}\xspace} 
 
\def\T{\ensuremath{\top}}  
\def\sig{\ensuremath{\sigma}\xspace}

\def\supp{\ensuremath{\mbox{supp}}\xspace}

\def\dt{{\ensuremath{\delta}\xspace} }
\def\sm{{\ensuremath{\setminus}\xspace} }

\def\hattheta{\ensuremath{\widehat{\theta}}\xspace}

\def\lfl{\lfloor} 
\def\rfl{\rfloor}

\makeatletter
\newcommand{\vast}{\bBigg@{3}}
\newcommand{\Vast}{\bBigg@{4}}
\makeatother

\def\cX{\ensuremath{\mathcal{X}}\xspace} 
\def\cZ{\ensuremath{\mathcal{Z}}\xspace} 
\def\la{{\langle}}
\def\ra{{\rangle}}

\def\lam{\ensuremath{\lambda}}

\def\cC{\ensuremath{\mathcal{C}}\xspace}

\def\KL{\ensuremath{\normalfont{\mathsf{KL}}}}

\newcommand{\wbar}[1]{ {\ensuremath{\overline{#1}}} }

\def\hth{{\hat{ \theta}}}

\def\lam{{\ensuremath{\lambda}\xspace} }

\def\cS{{\ensuremath{\mathcal{S}}}}

\def\cR{\ensuremath{\mathcal{R}}}

\def\th{{\ensuremath{\theta}}}

\def\cd{\cdot}

\def\Bernoulli{{\ensuremath{\mathsf{Bernoulli}}}}
\def\lammin{{\lam_{\min}}}

\newcommand{\mytag}[1]{\tag*{$\del{#1}$}}

\newcommand\msf[1]{{\mathsf{#1}}}
\def\dmu{\dot\mu}

\def\teff{{\ensuremath{t_\mathsf{eff}}}}
\def\sp{{\ensuremath{\mathsf{p}}}}
\def\sq{{\ensuremath{\mathsf{q}}}}
\usepackage[usenames,dvipsnames]{xcolor} 
\usepackage{xspace}

\def\Bernoulli{{\ensuremath{\mathsf{Bernoulli}}}}
\def\lam{\ensuremath{\lambda}}
\def\lammin{{\lam_{\min}}}
\def\cd{\cdot}
\def\dt{{\ensuremath{\delta}\xspace} }
\renewcommand{\th}{\theta}
\renewcommand{\P}{\mathbb{P}}
\newcommand{\E}{\mathbb{E}}
\newcommand{\R}{\mathbb{R}}
\renewcommand{\ln}{\log}
\newcommand{\mc}[1]{\mathcal{#1}}
\newcommand{\N}{\mathbb{N}}

\usepackage[dvipsnames]{xcolor}

\newcommand{\blue}[1]{#1}

\def\dmu{{\dot\mu}}
\allowdisplaybreaks 

\usepackage[toc,page,header]{appendix}
\usepackage{minitoc}

\icmltitlerunning{Improved Confidence Bounds for the Linear Logistic Model and Applications to Bandits}

\begin{document}
\doparttoc 
\faketableofcontents 


\twocolumn[
\icmltitle{Improved Confidence Bounds for the Linear Logistic Model and Applications to Bandits}



\icmlsetsymbol{equal}{*}

\begin{icmlauthorlist}
\icmlauthor{Kwang-Sung Jun}{ua}
\icmlauthor{Lalit Jain}{uw}
\icmlauthor{Blake Mason}{uwm}
\icmlauthor{Houssam Nassif}{amazon}
\end{icmlauthorlist}

\icmlaffiliation{ua}{University of Arizona}
\icmlaffiliation{uw}{University of Washington}
\icmlaffiliation{uwm}{University of Wisconsin}
\icmlaffiliation{amazon}{Amazon Inc}


\icmlkeywords{Machine Learning, ICML}

\vskip 0.3in
]



\printAffiliationsAndNotice{}  

\begin{abstract}
We propose improved fixed-design confidence bounds for the linear logistic model. Our bounds significantly improve upon the state-of-the-art bound by Li et al. (2017) via recent developments of the self-concordant analysis of the logistic loss (Faury et al., 2020). Specifically, our confidence bound avoids a direct dependence on $1/\kappa$, where $\kappa$ is the minimal variance over all arms' reward distributions. In general, $1/\kappa$ scales exponentially with the norm of the unknown linear parameter $\theta^*$. Instead of relying on this worst case quantity, our confidence bound for the reward of any given arm depends directly on the variance of that arm's reward distribution. We present two applications of our novel bounds to pure exploration and regret minimization logistic bandits improving upon state-of-the-art performance guarantees. For pure exploration we also provide a lower bound highlighting a dependence on $1/\kappa$ for a family of instances.
\end{abstract}

\setlength{\abovedisplayskip}{3pt}
\setlength{\belowdisplayskip}{3pt}
\setlength{\abovedisplayshortskip}{3pt}
\setlength{\belowdisplayshortskip}{3pt}

\vspace{-2em}
\section{Introduction}
\vspace{-.5em}


Multi-armed bandits algorithms offer a principled approach to solve sequential decision problems under limited feedback~\citet{thompson33onthelikelihood}. 
In bandit problems, at each time step, an agent chooses an arm to pull from an available pool of arms and and receives an associated reward. 
Under this setting, two major objectives arise: \emph{pure exploration} (aka best-arm identification) where the goal is to identify the arm with the highest average reward; and \emph{regret minimization} where the goal is to maximize the total rewards gained.
Bandit algorithms are widely deployed in industry, with applications spanning news recommendation \citep{li10acontextual}, ads \citep{SawantHVAE18}, online retail \citep{Teo2016airstream}, and drug discovery \citep{kazerouni2019best}. 
In such applications, the agent often has access to a feature vector for each arm.
A common assumption is that the reward is a noisy linear measurement of the underlying feature vector of the arm being pulled.  In other words, the binary reward received from a pull of the arm $x\in \mathbb{R}^d$ is $y_t = x^\T\theta^\ast +\epsilon$, $\theta^\ast$ is a latent parameter vector, and $\epsilon$ is subGaussian noise. 
In this case, there are several algorithms that are near-optimal and/or practical for pure exploration~\cite{xu2018fully,fiez2019sequential} and for regret minimization~\cite{auer2002using,chu11contextual,dani08stochastic,russo14learning}.


Unfortunately, in abundant real-world use-cases the linear reward model is not realistic and instead rewards are binary. For example, the prevalent form of data arising from user interactions is binary click/no-click feedback in the web and e-commerce domains. 
Another example is the problem of learning the best candidate from binary pairwise comparisons, used in matching recommender systems~\cite{biswas2019seeker}. In this setting, the agent has access to a set of items (e.g., shoes), and repeatedly chooses a pair of items to present to the user to choose from. The goal of the agent is to infer the user's favorite shoe.
In this paper, we use the linear logistic model for binary feedback.
In other words, the binary reward received from a pull of the arm $x\in \mathbb{R}^d$ is $y_t \sim \mathsf{Bernoulli}(\mu(x^\T\theta^\ast))$, where $\mu(z) := (1 + e^{-z})^{-1}$ is the logistic link function. 


Existing effective bandit algorithms in the linear feedback setting attempt to estimate $\theta^\ast$ to drive sampling. To do so, they require tight confidence intervals on the estimated mean reward $x^\T\th^*$ of arm $x$. 
To adapt these algorithms to the logistic model, we require confidence intervals that account for the non-linearity introduced by the link function $\mu$. However, there is a lack of tight confidence intervals in this setting. 
Our work builds on previous attempts in this area to a) provide tight confidence intervals, b) adapt existing linear bandit algorithms to the logistic setting. We now detail our contributions:

\textbf{The first variance-dependent fixed design confidence interval for the linear logistic model.} We first consider the \textit{fixed design} setting. Assume we have access to data $(x_s, y_s)_{s=1}^t\subset \mathbb{R}^d\times \{0,1\}$ where the reward $y_s$ is generated according to the logistic model. In addition we assume $y_s$ is conditionally independent from $\{x_i\}_{i=1}^t \setminus \{x_s\}$ given $x_s$.
Let $\hat\th_t$ be the maximum likelihood estimator (MLE) of $\th^*$.
We propose the first fixed design concentration inequalities such that the width: $i)$ scales with the actual variance instead of the worst-case variance $\kappa^{-1}$ that scales exponentially with $\|\theta^*\|$, and $ii)$ is independent of $d$.
Our bound takes the form of
\begin{align*}
    \PP\del{ |x^\T (\hat\th_t - \th^*)| \le O(\| x \|_{H_t^{-1}(\th^{\ast})} \sqrt{\log(t/\dt)}) }\ge 1-\dt,
\end{align*}
where $H_t(\th^{\ast})$ is the Fisher information matrix at $\theta^\ast$ matching the asymptotic bound for the MLE.
\footnote{
    While $\th^*$ appears on the RHS as well, our full theorem shows that $\hth_t$ can be used in place of $\th^*$ with a slightly larger constant factor, although this is not useful in bandit analysis.
} 
By contrast, the bounds by \citet{li2017provably} take a significantly looser form of $ O(\kappa^{-1}\| x \|_{V_t^{-1}} \sqrt{\log(1/\dt)})$ where $V_t$ satisfies $\kappa V_t \preceq H_t(\th^*)$.
Our improvements in fixed design confidence bounds parallel that of \citet{faury2020improved} for adaptive sampling but reduce a $\sqrt{d}$ factor required by adaptive bounds.
Our confidence bound is a fundamental result in statistical learning. It tightly quantifies the amount of information learned from the training set $\{x_s,y_s\}_{s=1}^t$ that transfers to a test point $x$, in a data-dependent non-asymptotic manner and without distributional assumptions on $\{x_s\}_{s=1}^t$.
We present the full theorem and provide detailed comparisons in Section~\ref{sec:ci}.



\textbf{Improved pure exploration algorithms. } 
In Section~\ref{sec:rage} we propose a new algorithm called RAGE-GLM for pure exploration in transductive linear logistic bandits, which is a novel extension of RAGE by \citet{fiez2019sequential}. 
RAGE-GLM significantly improves both theoretical and empirical performance over the state-of-the-art algorithm by \citet{kazerouni2019best}, reducing the sample complexity by a multiplicative factor of $\kappa^{-1}$.
We perform empirical evaluations on a pairwise comparison problem.

\textbf{Novel fundamental limits for pure exploration.}
While the sample complexity of RAGE-GLM does not have $\kappa^{-1}$ in the leading term of $\log(1/\dt)$ where $\dt$ is the target failure rate, it has an \emph{additive} dependence on $\kappa^{-1}$.
In Section~\ref{sec:lower_kappa}, we show that such an additive dependence is necessary via a novel \emph{moderate confidence} lower bound that captures the non-asymptotic complexity of learning and is independent of transportation inequality techniques~\cite{kaufmann2016complexity}. 
Our results also imply that there are settings where $O(e^d)$ samples are necessary even when gaps are large, a phenomena that does not exist for linear rewards.

\textbf{Improved $K$-armed contextual bandits.} 
We employ our confidence bounds to develop improved algorithms for contextual logistic bandits.
The proposed algorithm SupLogistic makes nontrivial extensions over the state-of-the-art algorithm SupCB-GLM by \citet{li2017provably}. The main challenge is to $i)$ handle the confidence width that depends on the unknown $\th^*$, and $ii)$ design a novel sample bucket scheme to fix an issue of SupCB-GLM that invalidates its regret bound.
We show that SupCB-GLM enjoys a regret bound of $\tilde{O}( \sqrt{dT\log(K)})$ (ignoring $o(\sqrt{T})$ terms), which is a significant improvement over SupCB-GLM that has an extra $\kappa^{-1}$ factor, along with improvements in the lower-order terms.
Such an improvement parallels that of \citet{faury2020improved} over UCB-GLM of \citet{li2017provably} where they achieve a regret bound $\tilde O(d\sqrt{T})$ that shaves of the factor $\kappa^{-1}$ from UCB-GLM.
We discuss our improved bounds and provide more detailed comparisons in Section~\ref{sec:suplogistic}.

\vspace{-.5em}
\section{Improved Confidence Intervals for the Linear Logistic MLE}
\label{sec:ci}
\vspace{-.5em}
In this section we consider the fixed design setting. We assume that we have a fixed $\theta^{\ast}\in \mathbb{R}^d$, a set of measurements $\{(x_s,y_s)\}_{s=1}^{t}\subset \mathbb{R}^d\times \mathbb{R}$ where each $y_s\in \{0,1\}$, and \[P(y_s=1) = \mu(x_{s}^\top\theta^*) = \frac{1}{1+e^{-x^{\top}\theta^{\ast}}}.\]

Let $\eta_s = y_s - \mu(x_s^\T \theta^*)$.
Denote by $\dot\mu(z)$ the first order derivative of $\mu(z)$.
Define $\kappa = \min_{x:~ \norm x \le 1} \dot\mu(x^\T \th^*)$.

The maximum likelihood estimate is given by:
\begin{align}\label{eq:mle}
    \hat{\theta} \!=\!  \argmax_{\theta\in \mathbb{R}^d} \sum_{s=1}^t  y_{s}\log\mu(x_s^{\top}\theta)
     \!+\!(1\!-\!y_s)\log(1\!-\!\mu(x_s^{\top}\theta)) 
\end{align}
We also define the Fisher information matrix at $\theta$ as
\begin{align}\label{eq:def-H}
  H_t(\theta) = \sum_{s=1}^t \dot\mu(x_s^{\top}\theta) x_sx_s^{\top} ~.  
\end{align}
We now introduce our improved confidence interval for the linear logistic model under this fixed design setting.

\begin{theorem}\label{thm:concentration}
  Let $\dt \le e^{-1}$. 
  Let $\hat\th_t$ be the solution of Eq.~\eqref{eq:mle} where, for every $s\in[t]$, $y_s$ is conditionally independent from $\{x_i\}_{i=1}^t \setminus \{x_s\}$ given $x_s$ (i.e. the $x_s$'s are a fixed design).
  Fix $x\in\RR^d$ with $\|x\|\le1$. 
  Let $t_{\mathsf{eff}}$ be the number of distinct vectors in $\{x_s\}_{s=1}^t$.
  Define ${\gamma(d)} \!=\! d + \ln(6(2\!+\!\teff)/\dt)$.
  Define the event $\cE_{\mathsf{var}} \!=\! \{\forall x', \fr{1}{\sqrt{2.2}}\,\|x'\|_{H_t(\th^*)^{-1}}
  \!\le\! \|x'\|_{H_t(\hat\th_t)^{-1}} 
  \!\le\! \sqrt{2.2}\, \|x'\|_{H_t(\th^*)^{-1}}\}$.
  If ${\xi^2_t} := \max_{s\in[t]} \|x_s\|^2_{H_{t}(\th^*)^{-1}} \le \fr{1}{\gamma(d)}$,
  then
  \begin{align*}
    \PP\del{\!
        |x^\T\!(\hat\theta_t\! -\! \theta^*)|  \!\le\! 3.5\|x \|_{\! H_{t} (\th^* \! )^{-\! 1}} 
        \textstyle\!\sqrt{\ln\fr{2(2+\teff)}{\dt}}
        , \cE_{\mathsf{var}}
    \!} \!\ge\! 1\!-\!\dt.
  \end{align*}
\end{theorem}

\begin{remark}
  One can see that Theorem~\ref{thm:concentration} implies empirical concentration inequality 
  $|x^\T\!(\hat\theta_t\! -\! \theta^*)|  \!\le\! 5.2\|x \|_{\! H_{t} (
    \hth_t \! )^{-\! 1}} \sqrt{\ln\fr{2(2+\teff)}{\dt}}$.
  This seemingly useful bound is in fact never used in our bandit analysis for technical reasons. Specifically, bandits select arms adaptively, which breaks the fixed design assumption of Theorem~\ref{thm:concentration}, so care is needed for algorithmic design.
\end{remark}
\vspace{-.5em}
Asymptotically, under some conditions we expect for any $x\in \mathbb{R}^d$, $x^{\top}(\hat{\theta} - \theta)\rightarrow N(0, \|x\|^2_{H(\theta^{\ast})^{-1}})$ \cite{lehmann2006theory}. 
Our bound matches this asymptotic rate up to constant factors. 

\textbf{Comparison to previous work.} 
Our theorem is a significant improvement upon \citet{li2017provably}. Their bound depends on $\fr{1}{\kappa}\|x\|_{V_t^{-1}}$, with $V_t := \sum_{s=1}^t x_s x_s^\T$. In general since $\kappa V_t \preceq H_t(\th^*)$, our bound is tighter and depends on the asymptotic variance. For the bound in \cite{li2017provably} to hold, they require that $\lambda_{\min}(V_t) \geq \Omega(d^2\kappa^{-4})$, which we call the \textit{burn-in} condition.
Recall that $\kappa^{-1} = \Theta(\exp(\|\theta^*\|))$ can be large even for moderate $\theta^*$.
In contrast, our bound's burn-in condition does not directly depend on $\kappa^{-1}$, and more importantly, it is possible to satisfy it without $\kappa^{-1}$ samples in certain cases.
For example, we show in our supplementary a case where a sample size of $\mathsf{polynomial}(\|\theta^*\|) \ll \exp(\|\theta^*\|)$ can satisfy the burn-in condition.
For the sake of comparison, we can use the bound $\xi_t^2\leq \kappa^{-1}\lambda_{\min}(V_t)$ to derive an inferior burn-in condition of $\lambda_{\min}(V_t)\geq \Omega(d\kappa^{-1})$.
This is still a strict improvement over \citet{li2017provably}, saving a factor of $d$ and a cubic factor in $\kappa^{-1}$ as well as shaving off their large constant factor of 512.

We now compare our bound with that of \citet{faury2020improved}. 
The proof of Lemma 3 of \citet{faury2020improved}, under the assumption that $\|\th^*\|\le S_*$ and with a proper choice of regularization constant, implies the following confidence bound: w.p. at least $1-\dt$, $\forall t\ge1, \forall x\in \RR^d: \|x\|\le1, $
\begin{align*}
  x^\T (\hat\th_t-\th^*) \le \Theta\lt(\|x\|_{H_t(\theta^*)^{-1}}S_*^{3/2}\sqrt{(d + \ln(t/\dt)})\rt)~.
\end{align*}
While their bound also does not directly depend on $\kappa$ it is an anytime confidence bound that holds for all $x\in \mathbb{R}^d$ simultaneously, and in addition allows for an adaptively chosen sequence of measurements. As a result, their bound suffers an additional factor of $\sqrt{d}$. Furthermore, their bound introduces a factor $S_*^{3/2}$ and requires the knowledge of both $S_*$ and $\kappa$.\footnote{We believe the factor $S_*^{3/2}$ can be removed by imposing  an assumption on $\xi_t$ like ours.}
In contrast, our bound does not directly depend on the confidence width nor require the knowledge of $S_*$ or $\kappa$ though these quantities may be needed to satisfy the burn-in condition. 

\textbf{Tightness of our bound.}
Empirically, we have observed that $\xi_t^2 \le O(1)$ is necessary to control the confidence width as a function of $\|x\|_{H_t(\th^*)^{-1}}$, but have not found a case where $\xi_t^2$ must be smaller than $O(1/d)$; studying the optimal burn-in condition is left as future work.
We believe one can improve $\log(\teff)$ to the metric entropy of the measurements $\{x_s\}_{s=1}^t$.
Note that it is possible to remove the burn-in condition if we derive a \textit{regularized} MLE version of Theorem~\ref{thm:concentration}, but this comes with an extra factor of $\sqrt{d}$ in the confidence width, which is not better than the confidence bound of \citet{faury2020improved}.
%
%
%
%
\vspace{-.5em}
\begin{proof}[Proof Sketch of Theorem~\ref{thm:concentration}]
  The novelty of our argument is to exploit the variance term without introducing $\kappa$ explicitly in the confidence width.
  We follow the main decomposition of \citet[Theorem 1]{li2017provably} but deviate from their proof by: $i)$ employing Bernstein's inequality rather than Hoeffding's to obtain $H_t(\th^*)$  in the bound (as opposed to $\kappa^{-1} V_t$); and $ii)$ deriving a novel implicit inequality on $\max_{s\in[t]} |x_s^\T(\hth_t - \th^*)|$.
  The latter leads to the significant improvements in both $d$ and $\kappa^{-1}$ in the condition on $\xi_t$. 
  

  Let $\alpha_s(\hat\th_t,\th^*) := \frac{\mu(x_s^\T\hat\th_t) - \mu(x_s^\T\th^*)}{x_s^\T(\hat\th_t - \th^*)}$, $\blue{z} := \sum_s^t \eta_s x_s$, and $\blue{G} := \sum_s^t \alpha_s(\hat\theta_t, \theta^*) x_s x_s^\T$. 
  By the optimality condition of $\hat\th_t$, we have:
  \begin{align}
      \begin{split}
    z &= G (\hat\theta_t - \theta^*).
  \end{split}
  \end{align}
  We use the shorthand $H := H_t(\theta^*)$ and define $\blue{E} := G-H$. 
  The main decomposition is
  \begin{align}\label{eq:main-decomp}
      |x^\T(\hat\theta_t - \theta^*) |
      &\!=\! |x^\T (H+E)^{-1} z|
      \\&\!\le\! |x^\T H^{-1}z| + |x^\T H^{-1}E (H+E)^{-1}z|~. \nonumber
  \end{align}
We bound $x^\T H^{-1} z$ by $O(\|x\|_{H^{-1}}\sqrt{\log(1/\dt)})$ which uses Bernstein's inequality and the assumption on $\xi^2_t$.
We control the second term by:
  \begin{align}\label{eq:main-decomp-2}
    &|x^\T H^{-1}E (H+E)^{-1}z| \nonumber
    \\&\le \|x\|_{H^{-1}} \|H^{-1/2} E H^{-1/2}\| \|G^{-1} z\|_{H}
    \\&\stackrel{(a)}{\le} \|x\|_{H^{-1}} \|H^{-1/2} E H^{-1/2}\| (1+D) \|z\|_{H^{-1}}\nonumber
    \\&\stackrel{(b)}{\le} \|x\|_{H^{-1}} \|H^{-1/2} E H^{-1/2}\| (1+D) \cd O(\sqrt{d+\log(1/\dt)})\nonumber
    \\&\stackrel{(c)}{\le} \|x\|_{H^{-1}} D \cd O(\sqrt{d+\log(1/\dt)}), \nonumber
  \end{align}
  where $(a)$ is by introducing $D := \max_{s\in[t]} |x_s^\T(\hth - \th^*)|$ and using the self-concordance property of the logistic loss~\cite{faury2020improved}, $(b)$ is Bernstein's inequality with a covering argument along with the assumption on $\xi_t$, and $(c)$ is by a novel result employing self-concordance and the assumption to show that $E$ can be bounded by $D\cd H$. 
Our key observation is to apply Eq.~\eqref{eq:main-decomp} for every distinct vector $x$  in $\{x_s\}_{s\in[t]}$ and employ $\|x\|_{H^{-1}} \le \xi_t$ to see:
  \begin{align*}
    D &= \max_{s\in[t]} |x_s^\T(\hth - \th^*)|
    \\&\le O(\xi_t \sqrt{\log(1/\dt)}) + \xi_t D \cd O(\sqrt{d+\log(1/\dt)})~.
  \end{align*}
  The assumption on $\xi_t$ implies $\xi_t O(\sqrt{d+\log(1/\dt)}) < 1$. We solve for $D$ to obtain the implicit equation
  \begin{align*}
    D = O(\xi_t \sqrt{1/\dt})~.
  \end{align*}
  Plugging this back into Eq.~\eqref{eq:main-decomp-2} gives $\|x^\T(\hth - \th^*)\| \le  O(\|x\|_{H^{-1}}\sqrt{\log(1/\dt)})$, which holds with probability at least $1-\Theta(\teff\delta)$, as we use the concentration inequality $\Theta(\teff)$ times.
  To turn this into a statement that holds w.p. at least $1-\dt$, we substitute $\dt$ with $\Theta(\dt/\teff)$, concluding the proof. See appendix for the statement on $\cE_{\mathsf{var}}$.
\end{proof}
\vspace{-1.2em}




\vspace{-.5em}
\section{Transductive Pure Exploration}
\label{sec:rage}
\vspace{-.5em}

We now consider the linear logistic pure exploration problem.
Specifically, the learner is given a confidence level $\dt\in(0,1)$ and has access to finite arm subsets $\mc{X}, \mc{Z}\subset \mathbb{R}^d$ but not the problem parameter $\th^*\in\RR^d$.
The goal is to discover $z^{\ast} = \arg\max_{z\in \mathcal{Z}} z^{\top}\theta^{\ast}$ with probability at least $1-\delta$ using as few measurements from $\mc{X}$ as possible. 
This is a generalization of the standard linear pure exploration~\cite{soare2014best} and is referred to as the transductive setting \cite{fiez2019sequential}. 
In each time step $t$, the learner chooses an arm $x_t$, which is measurable with respect to the history $\mathcal{F}_{t-1} = \{(x_s, y_s)_{s< t}\}$, and observes a reward $y_t\in\{0,1\}$. 
It stops at a random stopping time $\tau$ and recommends $\hat{z}\in \mathcal{Z}$, where $\tau$ and $\hat{z}$ are both measurable w.r.t.\ the history $\mc{F}_{\tau}$. We assume that $\P(y_t = 1|x_t, \mc{F}_{t-1}) = \mu(x_t^{\top}\theta^{\ast})$. Let $\P_{\theta}, \E_{\theta}$ denote the probability and expectation induced by the actions and rewards when the true parameter is $\theta$.
Formally, we define a $\delta$-PAC algorithm as follows:
\emph{An algorithm for the logistic-transductive-bandit problem is $\delta$-PAC for $(\mc{X}, \mc{Z})$ if $\P_{\theta}(z_{\tau}\neq z^{\ast})\leq \delta\ ,\forall \theta\in \mathbb{R}^d$.  }

\textbf{Example: Pairwise Judgements.} As a concrete and natural example, consider an e-commerce application where the goal is to recommend an item from a set of items based on relative judgements by the user. For example, the user may be repeatedly shown two items to compare, and must choose one. 
In each round, the system chooses two items $z, z'\in \mc{Z}$, and observes the binary user preference of item $z$ or item $z'$. 
A natural model is to give each item $z\in\mc{Z}$ a score $z^{\top}\theta^{\ast}$, with the goal of finding $z^{\ast} = \max_{z\in \mc{Z}} z^{\top}\theta^{\ast}$. The probability the user prefers item $z$ over $z'$ is modeled by $\P(z > z') = \mu((z - z')^{\top}\theta^{\ast})$. Hence the set of measurement vectors is naturally $\mc{X} = \{z - z':z, z'\in \mc{Z}\}$. 
Note that our setting is a natural generalization of the dueling bandit setting \citep{Yue2012Dueling}. As far as we are aware, this is the first work to propose the dueling bandit problem as a natural extension of the transductive linear bandit setting under a logistic noise model. 





\textbf{Related Work. } 
\citet{soare2014best} first proposed the problem of pure exploration in linear bandits with Gaussian noise when $\mc{X}=\mc{Z}$.
\citet{fiez2019sequential} introduced the general transductive setting, provided the RAGE elimination based method which is the main inspiration for our algorithm. RAGE achieves the lower bound up to logarithmic factors with excellent empirical performance. 
Other works include \citet{degenne2020gamification, karnin2013almost}, which achieve the lower bound asymptotically. 
Finally we mention \citet{katz2020empirical}, which follows a similar approach to \citet{fiez2019sequential} but uses empirical process theory to replace the union bound over the number of arms with a Gaussian width. 

Despite its importance in abundant real-life settings, pure-exploration for logisitic bandits has received little attention. The only work we are aware of is \citet{kazerouni2019best} which defines the problem and provides an algorithm motivated by \citet{xu2018fully}.
However, their theoretical and empirical performance are both far behind our proposed algorithm RAGE-GLM as we elaborate more below. 

\textbf{Notation.}  
Define $\kappa_0 = \min_{x\in \mc{X}} \dot\mu(x^{\top}\theta^{\ast})$, the smallest derivative of the link function among elements in $\mc{X}$ (differing slightly from the previous section). 
Let $\Delta_{\mc{X}} = \{\lambda\in \mathbb{R}^{|\mc{X}|}, \lambda\geq 0, \sum_{x\in \mc{X}} \lambda_x = 1\}$ be the probability simplex over $\mc{X}$. Given a design $\lambda\in \Delta_{\mathcal{X}}$, define: 
 \[H(\lambda, \theta) := \sum_{x\in \mathcal{X}} \lambda_x\dot{\mu}(x^{\top}\theta)xx^{\top},\; A(\lambda) := \sum_{x\in \mathcal{X}} \lambda_x xx^{\top}.\]

\vspace{-.5em}
\subsection{Algorithm}
\vspace{-.5em}
\newcommand{\cone}{3.5}
\newcommand{\rdp}{\text{\textbf{round}}}


\begin{algorithm}[ht]
\caption{RAGE-GLM }
        \label{alg:rage-glm}
	    \begin{algorithmic}[1]
            \Input {$\epsilon$, $\delta$, $\mc{X}$, $\mc{Z}$, $\kappa_0$, effective rounding procedure $\rdp(n, \epsilon, \lambda)$}
            \Initialize{$k=1, \mc{Z}_1 = \mc{Z}, r(\epsilon) = d^2/\epsilon$}
            \State {$\hat\theta_0\leftarrow$ \textbf{BurnIn}($\mc{X}, \kappa_0$)}  \Comment{Burn-in phase}
            \While{$|\mc{Z}_k| > 1$ } \Comment{Elimination phase}
            \State  $\delta_k = \delta/(2k^2\max\{|\mc{Z}|,|\mc{X}|\} (2+|\mc{X}|))$
            \State\label{line:expDesign}$f(\lambda) := \max\Big[ 
                    \gamma(d)\max_{x\in \mc{X}} \|x\|_{H(\lambda, \hat{\theta}_{k-1})^{-1}}^2$,
            \Statex \hspace*{\fill} $2^{2k}\cdot(\cone)^2 
                    \max_{z,z'\in\mc{Z}_k} \|z-z'\|_{H(\lambda, \hat{\theta}_{k-1})^{-1}}^2
                    \Big]$
            \State $\lambda_k = \argmin_{\lambda\in \Delta_{\mc{X}}} f(\lambda)$
            \State{$n_k =  \max\{3(1+\epsilon)f(\lambda_k) \log(1/\delta_k), r(\epsilon)\}$}
            \State{$x_1, \cdots, x_{n_k}\leftarrow \rdp(n,\epsilon,\lambda)$} 
            \State{Observe rewards $y_1, \cdots, y_{n_k}\in \{0,1\}$.}
            \State{Compute the MLE $\hat\theta_k$ with $\{(x_i,y_i)\}_{i=1}^{n_k}$. \Comment{Eq~\eqref{eq:mle}}}
            \State{$\hat{z}_k = \arg\max_{z\in\mc{Z}_k} \hat\theta_k^\top z$}
            \State{$\mc{Z}_{k+1}\leftarrow \mc{Z}_k \setminus \left\{z\in \mc{Z}_k:\hat\theta_k^\top (\hat{z}_k-z) \geq 2^{-k}\right\}$}
            \State{$k\gets k+1$}
			\EndWhile
			\State \Return $\hat{z}_k$
		\end{algorithmic}
	\end{algorithm}

Our RAGE-GLM algorithm (Alg.~\ref{alg:rage-glm}) proceeds in rounds. 
In each round $k$, it maintains a set of active arms $\mathcal{Z}_k$. 
It computes an experimental design that is dependent on $\hat{\theta}_{k-1}$, the estimate of $\theta^{\ast}$ from the previous round, and uses this experimental design to draw $n_k$ samples. 
Then, the algorithm eliminates any arms verified to be sub-optimal using $\hat\th_{k}$. We now go into the algorithmic details.

\textbf{Rounding.} In the $k$-th round, we have $H_k(\theta) := \sum_{s=1}^{n_k}\dot\mu(x_s^{\top}\theta)x_sx_s^{\top}$.\footnote{
  We abuse notation in this section and use the subscript $k$ to denote the round, not the number of samples as in section~\ref{sec:ci}.
}
The algorithm utilizes an efficient rounding procedure to ensure that $H_k(\hat{\theta}_{k-1})$ is within a constant factor of $n_k \cd H(\lambda_k, \hat{\theta}_{k-1})$. 
Given distribution $\lambda_{k}$, tolerance $\epsilon$, and a number of samples $n_k \geq r(\epsilon)$, the procedure $\rdp(n_k, \epsilon, \lambda)$ returns an allocation $\{x_s\}_{s=1}^{n_k}$ such that for any $\theta\in\mathbb{R}^d$, $H_k(\theta)\geq \frac{n_k}{1+\epsilon}H(\lambda, \theta)$. 
Efficient rounding procedures with $r(\epsilon) = d^{2}/\epsilon$ are described in \cite{fiez2019sequential}; see the supplementary for more details. 

\begin{algorithm}[t]
\caption{BurnIn}
    \label{alg:rage-burnin}
	\begin{algorithmic}[1]
        \Input $\mc{X}, \kappa_0$
		\Initialize $\lambda_0 = \arg\min_{\lambda\in \Delta_{\mc{X}}}\max_{x\in \mathcal{X}} \|x\|_{A(\lambda)^{-1}}^{2}$ 
		\State $n_0 = 3(1+\epsilon)\kappa_0^{-1}d\gamma(d)\log(2|\cX|(2+|\cX|)/\delta)$
		\State $x_1, \cdots, x_{n_0}\leftarrow \rdp(n_0, \lambda_0, \epsilon)$
		\State Observe associated rewards $y_1, \cdots, y_{n_0}$.
		\State \Return MLE $\hat{\theta}_0$ \Comment{Use Eq~\eqref{eq:mle}}
	\end{algorithmic}
\end{algorithm}
\textfloatsep=.8em

\textbf{Burn-In Phase.} The burn-in phase computes $\hat{\theta}_0$, an estimate of $\theta^{\ast}$ to be used in the first round. 
To do so, we need to guarantee that $\theta^{\ast}$ is well-estimated in all directions $\mathcal{X}$, i.e., $|x^{\top}(\hat{\theta} - \theta^{\ast})| < 1, \forall x\in\cX$. 
Ensuring this requires that we can employ the confidence interval in Theorem~\ref{thm:concentration}. 
Thus, burn-in Algorithm~\ref{alg:rage-burnin} must ensure that $\max_{x\in \mathcal{X}}\|x\|^2_{(\sum_{s=1}^{n_0}\dot\mu(x_s^\top\theta^{\ast})x_s x_s^{\top})^{-1}} \leq 1/\gamma(d)$. 
As we yet lack information on $\theta^{\ast}$, we take the naive approximation: 
%
 \begin{align*}
  \sum_{s=1}^{n_0}\dot\mu({\theta^{\ast}}^{\top}x_s)x_s x_s^{\top} 
  &\geq \frac{n_0}{1+\epsilon}H(\lambda, \theta) \quad \text{(from rounding)}\\
  &\geq \frac{n_0}{1+\epsilon}\kappa_0 A(\lambda),
\end{align*}
and instead consider the optimization problem $\min_{\lambda\in \mc{X}}\max_{x\in \mc{X}} \|x\|_{A(\lambda)^{-1}}^2$. 
This is a G-optimal experimental design, and has a value of $d$ by the Kiefer-Wolfowitz theorem \citep{soare2014best}. For the burn-in phase we assume we have access to an upper bound on $\kappa_0^{-1}$, which is equivalent to knowing an upper bound on $\|\theta^*\|$.   


\textbf{Experimental Design.}
In each round, line~\ref{line:expDesign} of Algorithm~\ref{alg:rage-glm} optimizes a convex experimental design that minimizes two objectives simultaneously. The main objective
is
\begin{equation}
2^{2k}\min_{\lambda\in \Delta_\mc{X}}\max_{z,z'\in \mathcal{Z}_k} \|z-z'\|^2_{H(\lambda, \hat\theta_{k-1})^{-1}}.
\end{equation}
which ensures the the gap $\theta^{\top}(z^{\ast} - z)$ is estimated to an error of $2^{-k}$ for each $z$. This allows us to eliminate arms whose gaps are significantly larger than $2^{-k}$ in each round, guaranteeing that $\mathcal{Z}_k\subset \mathcal{S}_k := \{z:(z^{\ast}-z)^{\top}\theta^{\ast}\leq 2\cdot 2^{-k}\}$. 

The other component of line~\ref{line:expDesign} minimizes $\max_{x\in \mc{X}}\|x\|_{H(\lambda, \hat{\theta}_{k-1})^{-1}}^2$
similarly to the burn-in phase. This guarantees that we satisfy the conditions of Theorem~\ref{thm:concentration}. 
It additionally guarantees that the estimate $\hat\theta_{k}$ is sufficiently close to $\theta^{\ast}$ for all directions in $\mc{X}$. 
Combining this with self-concordance, $|\ddot{\mu}|\leq \dot\mu$, we show that $H(\lambda_k, \theta^{\ast})$ is within a constant factor of $H(\lambda_k, \hat\theta_k)$ (see the supplementary). 
We stop when $|\mc{Z}_k|=1$ and return the remaining arm.

\begin{theorem}[Sample Complexity]\label{thm:rage-sc}
Take $\delta < 1/e$ and assume $\|z\|\leq 1/2$ for all $z\in \cZ$. 
Define 
\begin{align*}
    \beta_k =  \min_{\lambda\in\Delta_{\mc{X}}}\max\bigg[2^{2k}\max_{z,z'\in \mc{S}_k} \|z-z'\|_{H(\lambda, \theta^{\ast})^{-1}}^2 ,\\\gamma(d)\max_{x}\|x\|_{H(\lam,\theta^{\ast})^{-1}}^2\bigg]
\end{align*}

Algorithm~\ref{alg:rage-glm} returns $z^{\ast}$ with probability greater than $1-3\delta$ in a number of samples no more than 
\begin{align*}
 &O\bigg(\sum_{k=1}^{\lceil\log_2(2/\Delta_{\min})\rceil}\beta_k \log(\max\{|\mc{Z}|, |\mc{X}|\}k^2/\delta)\\ 
& + d(1+\epsilon)\gamma(d)\kappa_0^{-1}\log(|\mc{X}|/\delta)+ r(\epsilon)\log(\Delta_{\min}^{-1})\bigg)
\end{align*}
 where $\Delta_{\min} = \min_{z\neq z^{\ast}\in \mc{Z}} \langle\theta^{\ast}, z^{\ast}-z\rangle$.  
\end{theorem}



\textbf{Interpreting the Upper Bound.} 
Before comparing our bound with prior work, we show concrete examples that show the strength of our sample complexity bound.

\textbf{Example 1. } 
Consider a simple setting where $\mc{Z} = \mc{X}=\{e_1, e_2\}\subset\RR^2$, and $\theta^{\ast} = (r, r-\epsilon)$, for $r\geq 0$. In this case, $\kappa_0^{-1} = \max_{i\in \{1,2\}} \dot\mu(z_i^{\top}\theta^{\ast})^{-1}\leq e^{r}$. Thus in the burn-in phase, we take roughly $\tilde{O}(e^{r})$ samples. 
Now, for small $\epsilon$, the minimizer of $\min_{\lambda\in \Delta_{\mc{X}}} \|e_1 - e_2\|_{H(\theta^{\ast})^{-1}}^2$ places roughly equal mass on $e_1$ and $e_2$, giving an objective value that is roughly bounded by $e^r$. Thus the sample complexity of Algorithm~\ref{alg:rage-glm} is $O(\sum_{k=1}^{\log_2(1/\epsilon)} 2^{2k} e^{r}\log(1/\delta)) \approx \frac{e^{r}}{\epsilon^2}$.
 
Note this problem is equivalent to a standard best-arm identification algorithm with two Bernoulli arms \citep{kaufmann2015complexity}. 
Standard results in Pure Exploration show that a lower bound on this problem is given by the KL-divergence 
  $\mathsf{KL}(\Bernoulli(\mu(\theta^{\top}\!z_1)), \Bernoulli(\mu(\theta^{\top}\!z_2)))^{-1} \!\approx\! \frac{e^r}{2\epsilon^2}$ for sufficiently small $\epsilon$. 
  This shows that our bound is at least no worse than the well-studied unstructured case. 

\textbf{Example 2. } We extend the above setting and consider $\mc{X} = \{e_1, e_2, e_1-e_2\}$, $\mc{Z} = \{e_1, e_2\}$ and the same $\theta^{\ast}$. As above, the burn-in phase requires $\kappa^{-1}\approx e^{r}$ samples. Starting from the first round, our computed experimental design will place all of its samples on the third arm.  In this case, $\min_{\lambda} \|e_1 - e_2\|_{H(\theta^{\ast})^{-1}}^2 = 1/\dot{\mu}(\epsilon)\leq C$, for small $\epsilon$.\footnote{With $H(\theta^{\ast})^{-1}$ interpreted as a pseudo-inverse.} The main term of the sample complexity becomes
\begin{align*}
O\del{\sum_{k=1}^{\log_2(1/\Delta_{\min})} 2^{2k}\log\frac1\delta}
  \leq O\del{\frac{1}{\epsilon^2}\log\frac1\delta}~.
\end{align*}

Hence ignoring burn-in or the additional samples we take in each round to guarantee the confidence interval, the total sample complexity would be $\tilde{O}(\frac{1}{\epsilon^2})$. This is exponentially smaller than in Example 1 and demonstrates the power of an informative arm in reducing the sample complexity. 

On the other hand, the burn-in phase, common to all logistic bandit algorithms based on the MLE, may potentially take a number of samples exponential in $r$. This example demonstrates the need for further work on understanding the precise dependence of $\kappa$ in pure exploration. In Section~\ref{sec:lower_kappa}, we take a first step towards this by showing $\kappa^{-1}$ burn-in is unavoidable in some cases.

\textbf{Comparison to past work.} 
As $\beta_k$ grows exponentially each round, the first element in the maximum for $\beta_k$ dominates our sample complexity. Focusing on this term while ignoring logarithmic terms and the burn-in samples, the sample complexity in Theorem~\ref{thm:rage-sc} scales as
\begin{align*}\label{eq:pure-exp-ub}
    \rho* := \sum_{k=1}^{\log_2(2/\Delta_{\min})} 2^{2k}\min_{\lambda\in \Delta_{\mc{X}}} \max_{z,z'\in \mc{S}_k} \|z-z'\|_{H(\theta^{\ast})}^2~.
\end{align*}

Importantly, this depends on $H(\theta^{\ast})$ instead of a loose bound based on $\kappa^{-1}$. In the supplementary we show
\[\rho*\leq \log\left(\frac{2}{\Delta_{\min}}\right)\min_{\lambda\in \Delta_{\mc{X}}}\max_{z\in \mc{Z}\setminus z^{\ast}} \frac{\|z^{\ast} - z\|_{H(\lambda, \theta^{\ast})^{-1}}^2}{\langle \theta^{\ast}, z^{\ast} - z\rangle^2}.\]
This is reminiscent of a similar quantity  that is within a $\log(\cdot)$ factor of being optimal for pure exploration linear bandits \citep{fiez2019sequential}. 
We provide a close connection between our upper bound and information theoretic lower bounds in the supplementary, although they do not match exactly.
We also prove a novel lower bound in moderate confidence regimes, which we elaborate more in Section~\ref{sec:lower_kappa}.

We now compare to the result of \citet{kazerouni2019best}. Using a variant of the UGapE algorithm for linear bandits \citep{xu2018fully}, they demonstrate a sample complexity 
$\tilde{O}(\frac{d|\mc{X}|}{\kappa^2\Delta^2_{\min}})$ in the setting where $\mc{X} =\mc{Z}$. 
This sample complexity, unlike ours, scales linearly with the number of arms, and only captures a dependency on the smallest gap. 
We note that one can improve on their sample complexity by using a naive passive algorithm that uses a fixed G-optimal design, along with the trivial bound $H(\lambda, \theta^{\ast})\geq \kappa_0 A(\lambda)$, resulting in $\tilde{O}(d/(\kappa_0\Delta_{\min}^2))$ \citep{soare2014best}.\footnote{This is equivalent to computing the allocation from Algorithm~\ref{alg:rage-burnin}, and sampling until all arms are eliminated.}
In contrast, the bound of Theorem~\ref{thm:rage-sc} only depends on the number of arms logarithmically, captures a local dependence on $\theta^{\ast}$, and has a better gap dependence. 

\textbf{Extra samples.} 
Algorithm~\ref{alg:rage-glm} potentially samples in each round to ensure the confidence interval is valid (i.e., the first argument of the $\max$ in line~\ref{line:expDesign}).
In our supplementary, we propose RAGE-GLM-2 that removes these samples needed in each round (but not the burn-in samples) by employing the confidence interval of \citet{faury2020improved}. This algorithm has a better asymptotic behavior as $\delta\rightarrow 0$, but could perform worse with large $d$ or $S_\ast$ due to an additional factor of $\sqrt{d}$ and a factor of $S_{\ast}^3$.

\subsection{Experiments}\label{sec:rage_experiments}


This section evaluates the empirical performance of \textbf{RAGE-GLM}, alongside two additional algorithms:
\begin{itemize}[leftmargin=*, noitemsep, topsep=0pt]
    \item \textbf{RAGE-GLM-R}: This is a heuristic version of RAGE-GLM that makes two changes. First, it does not do a burn-in in each round and samples from $\lambda_k =\min_{z,z'} \|z-z'\|_{H(\lambda, \hat\theta_{k-1})^{-1}}^2$ directly. Second, to compute the estimate $\hat\theta_k$, it uses all samples up to round $k$.
    \item \textbf{Passive Baseline}: This baseline runs the burn-in procedure and then computes the static design $\lambda = \min_{z,z'\in \mc{Z}} \|z-z'\|_{H(\lambda, \hat\theta_{0})^{-1}}^2$. It then proceeds in rounds, drawing $2^k$ samples in round $k$, terminating when we are able to verify that each arm is sub-optimal using the fixed design confidence interval (see Appendix for details). As in the heuristic, we recycle samples over rounds.   
\end{itemize}

\textbf{Remark.} We also implemented the algorithm of \cite{kazerouni2019best}. However a) the algorithm was extremely slow to run since an MLE and a convex optimization had to be run each round, b) the confidence bounds do not exploit the true variance. As a result, the algorithm did not terminate on any of our examples.

\begin{figure}
    \centering
    \includegraphics[width=.8\linewidth]{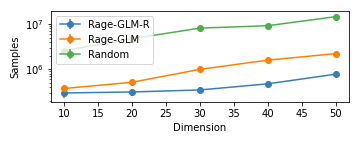}
    \vspace{-1.4em}
    \caption{Standard Baseline Example}
    \label{fig:soare}
    \vspace{-.8em}
\end{figure}

\begin{table}
        \centering
        \small
        \setlength\tabcolsep{4.1pt} 
        \begin{tabular}{|c|c|c|c|c|}
            \hline
            &open&pointy&sporty&comfort\\ 
            \hline
            $|\cZ|$&3327&2932&3219&3374\\
            \hline
            RAGE-GLM-R&\textbf{9.17e+07}&\textbf{2.38e+05}&\textbf{2.29e+05}&\textbf{3.34e+06}\\ 
            \hline
            RAGE-GLM&\textbf{9.17e+07}&\textbf{2.38e+05}&\textbf{2.29e+05}&4.55e+06\\ 
            \hline
            Passive&2.69e+08&\textbf{2.38e+05}&\textbf{2.29e+05}&8.54e+06\\
            \hline
        \end{tabular}
        \vspace{-.8em}
        \caption{Zappos pairwise comparison data, bold indicates a win.}
        \label{tab:zappos}
\end{table}

Our first experiment (Fig.~\ref{fig:soare}) is a common baseline in the linear bandits literature. We consider $d+1$ arms in $d$ dimensions with arm $i \in [1,n]$ being the $i$-th basis vector, and arm $i+1$ as $\cos(.1)e_1+\sin(.1)e_2$. We use $d\in \{10,20,30,40,50\}$ and 10 repetitions for each value of $d$. In all instances, all algorithms found the best arm correctly. RAGE-GLM was competitive against the heuristic RAGE-GLM-R and took roughly a factor of 10 less samples compared to Random.

Our next experiment is based on the Zappos pairwise comparison dataset~\cite{finegrained, semjitter}. This dataset consists of pairwise comparisons on 50k images of shoes and 960 dimensional vision-based feature vectors for each shoe. Given a pair of shoes, participants were asked to compare them on the the attributes of ``open'', ``pointy'',  ``sporty'' and ``comfort'' obtaining several thousand queries. For each one of these categories, we fit a logistic model to the set of pairwise comparisons after PCA-ing the features down to 25 dimensions (for computational tractability) and used the underlying weights as $\theta^{\ast}$. We then set $\cZ$ to be the set of shoes that were considered in that category and $\cX$ to be 5000 random pairs. Table~\ref{tab:zappos} shows the result. For the ``open'' and ``comfort'' category, RAGE-GLM took about a factor of 3 less samples compared to Random. The large sample complexity for ``open'' is due to an extremely small minimum gap of $O(10^{-4})$. For ``pointy'' and ``sporty'' the empirical gaps were large and all three algorithms terminated after the burn-in phase. Finally, $\kappa$ for all instances was on the order of $0.1$. 
See supplementary for a deeper discussion, alongside pictures of winning and informative shoes.

\section{\texorpdfstring{$1/\kappa_0$}{} Is Necessary}
\label{sec:lower_kappa}
In this section, we turn to lower bounds on the sample complexity of pure exploration problems. 
\vspace{-0.5em}
\begin{theorem}\label{thm:permuation_lower_bound}
Fix $\delta_1 < 1/16$, $d \geq 4$, and $\epsilon \in (0, 1/2]$ such that $d\varepsilon^2 \geq 12.2$. Let $\mc{Z}$ denote the action set and $\Theta$ denote a family of possible parameter vectors.
There exists instances satisfying the following properties simultaneously
\begin{enumerate}
    \item $|\mc{Z}| = |\Theta| = e^{\epsilon^2d/4}$ and $\|z\| =1 $ for all $z \in \mc{Z} $. 
    \item $S =\|\theta_{\ast}\|= O(\epsilon^2d)$
    \item Any algorithm that succeeds with probability at least $1-\delta_1$ satisfies
    $$\exists \th\in\Theta \text{ such that } \E_\th[T_{\delta_1}] > \Omega\left(e^{\epsilon^2d/4}\right) = c\left(\frac{1}{\kappa_0}\right)^{\frac{1-\epsilon}{1+3\epsilon}} $$
where $T_{\delta_1}$ is the random variable of the number of samples drawn by an algorithm and $c$ is an absolute constant. 
\end{enumerate}
\end{theorem}

The implications of this bound are two-fold. Firstly, it shows a family of instances where the dependence on $1/\kappa_0$ in the sample complexity of Algorithm~\ref{alg:rage-glm} is necessary.  
Secondly, this bound demonstrates a particular phenomenon of the logistic bandit problem: there are settings where $\kappa_0^{-1} \approx e^d$ samples are needed despite $\theta^\ast \in \R^d$. By contrast, for linear bandits, $O(d/\Delta_{\min}^2)$ samples are always sufficient \cite{soare2014best}. For the instances in the theorem, $\Delta_{\min} \geq \Omega(1-e^{-d})$. 
In the appendix, we state a second lower bound that captures the asymptotic sample complexity as $\delta \rightarrow0$, but show that this bound would suggest that only $O(\log(1/\delta))$ samples are necessary, which is vacuous for $\delta > e^{-e^d}$. 
Instead, the above \textit{moderate confidence} bound reflects the true sample complexity of the problem for values of $\delta$ seen in practice, e.g. $\delta\approx .05$. 
This dichotomy highlights that there are important challenges to logistic bandit problems that are not captured by their asymptotic sample complexity. 
In particular, this demonstrates that there exist instances of pure exploration logistic bandits that are  \textit{exponentially harder} than their linear counterparts. The proof is in the supplementary, inspired by a construction from \cite{dong2019performance}.

\vspace{-.5em}
\section{\texorpdfstring{$K$}{}-Armed Contextual Bandits}
\label{sec:suplogistic}
\vspace{-.5em}

We now switch gears and consider the contextual bandit setting where at each time step $t$ the environment presents the learner with an arm set $\cX_t = \{x_{t,1}, \ldots, x_{t,K} \}\subset \RR^d$
independently of the learner's history.
~\citep{auer2002using}.
The learner then chooses an arm index $a_t\in[K]$ and then receives a reward $y_t\sim\Bernoulli(\mu(x_{t,a_t}^\T \th^*))$, where parameter $\th^*$ is unknown to the learner.
Let $x_{t,a^*} = \argmax_{x\in \cX_t} \mu(x_{t,a}^\T \th^*)$ be the best arm at time step $t$.
The goal is to minimize the cumulative (pseudo-)regret over the time horizon $T$:
\begin{align}
    \mathsf{Reg}_T = \sum_{t=1}^T \mu(x_{t,a^*}^\T \th^*) - \mu(x_{t,a_t}^\T \th^*).
\end{align}
While the regret $\tilde O(d\sqrt{T} + d^2\kappa^{-1})$ is achievable by~\citet{faury2020improved}\footnote{$\tilde O$ hides poly-logarithmic factors in $T$}, one can aim to achieve an accelerated regret bound when $K=o(e^d)$.
Specifically, \citet{li2017provably} achieve the best-known bound of $\tilde O(\fr{1}{\kappa}\sqrt{dT\ln(K)})$.
However, the factor $1/\kappa$ is exponential w.r.t. $\|\th^\ast \|$, which makes the regret impractically large.
Leveraging our new confidence bound, we propose a new algorithm SupLogistic that removes $1/\kappa$ from the leading term: $\tilde O(\sqrt{dT\ln(K)})$. 

We assume that $\|x_{t,a}\| \le 1, \forall t\in[T], a\in[K]$, and that $\|\th^*\| \le S_*$ where $S_*$ is known to the learner. 
We follow \citet{li2017provably} and assume that there exists $\sig_0^2$ such that $\lammin(\EE[\fr1K \sum_{a\in[K]} x_{t,a} x_{t,a}^\T]) \ge \sig_0^2$, which is used to characterize the length of the burn-in period in our theorem.

We describe SupLogistic in Alg.~\ref{alg:SupLogistic}, which follows the standard mechanism for maintaining independent samples \citep{auer2002using}.
As the confidence bound is not available until enough samples are accrued, we perform $\tau$ time steps of burn-in sampling and then spread the samples across the buckets $\Psi_1,\ldots,\Psi_S, \Phi$ equally. 
Note that our burn-in sampling is different from \citet{li2017provably}. We show in the supplementary that their approach is problematic.
After the burn-in, in each time step $t$, we loop through the buckets $s\in[S]$.
In each loop, we compute $\th_{t-1}^{(s)}$, the MLE given in Eq~\eqref{eq:mle}, using the samples in the bucket $\Psi_s(t-1)$.
We compute $\th_\Phi$ in the same way using $\Phi$.
Let $X_t = x_{t,a_t}$.
For any $\th$, define 
\begin{align}
  H^{(s)}_{t}(\th) := \sum_{u\in\Psi_s(t)} \dmu(X_u^\T \th) X_u X_u^\T ~.
\end{align}
The algorithm computes the mean estimate and the confidence width of each arm $a\in[K]$ as follows:
\begin{align}\label{eq:m_w}
  \!\!\! m^{(s)}_{t,a} \!:=\! \la x_{t,a}, \theta_{t-1}^{(s)}\ra,  w^{(s)}_{t,a} \!:=\! \alpha \sqrt{2.2} \|x_{t,a}\|_{\!H_{t-1}^{(s)}\! (\theta_\Phi)^{- \! 1}}.
\end{align}
For each $s\in[S]$, we check if there is an underexplored arm (step \textbf{(a)}) and pull it.
Otherwise, we pull the arm with the highest empirical mean.
Finally, we filter arms whose empirical means are sufficiently far from the highest empirical mean and go to the next iteration.

The bucketing is important to maintain the validity of the concentration inequality in the analysis, which requires that the data satisfies the fixed design assumption in Theorem~\ref{thm:concentration}. 
Our comment on \citet{li2017provably} in our supplementary elaborates more on this issue. 

%


The main challenge of the design of SupLogistic over SupCB-GLM~\cite{li2017provably} is how we use our tight confidence bound, which requires the confidence width to depend on $\th^*$ (see Theorem~\ref{thm:concentration}).
Our solution is to use a separate bucket $\Phi$ dedicated for computing the width.
If we do not use $\Phi$ and use the empirical version of Theorem~\ref{thm:concentration}, we would break the fixed design assumption as we collect samples as a function of the rewards from the same bucket.
\textfloatsep=.5em

\begin{algorithm}[t]
\caption{SupLogistic}
\label{alg:SupLogistic}
\begin{algorithmic}[1]
    \Input{time horizon $T$, burn-in length $\tau$, and exploration rate $\alpha$}

    \Initialize{$S=\lfloor \log_2{T} \rfloor$}
    \Initialize{Buckets $\Psi_1=\cdots=\Psi_{S} = \Psi_{S+1} =\varnothing$}
    \For{$t\in[\tau]$}
        \State{Choose $a_t \in [K]$ uniformly at random.}
        \State{Add $a_t$ to the set $\Psi_{((t-1) \!\mod S+1) + 1}$.}
    \EndFor
    \Initialize{$\Psi_0 = \emptyset$, $\Phi = \Psi_{S+1}$}
    \For{$t=\tau+1,\tau+2,\cdots,T$}
        \Initialize{$A_1=[K]$, $s=1$, $a_t = \varnothing$.}
        \While{$a_t=\varnothing$}
            \State Compute $m_{t,a}$ and $w_{t,a}$ \Comment{use Eq~\eqref{eq:m_w}}
            \If {$w_{t,a}^{(s)}>2^{-s} \text{ for some } a \in A_s$} 
                \State $a_t=a$                                    \Comment{\textbf{(a)}}
                \State $\Psi_s \larrow \Psi_s \cup \{t\}$
            \ElsIf{$w_{t,a}^{(s)} \le 1/\sqrt{T}, \ \forall\ a \in A_s$}
                \State $a_t=\argmax_{a \in A_s} m_{t,a}^{(s)}$    \Comment{\textbf{(b)}}
                \State $\Psi_0 \larrow \Psi_0 \cup \{t\}$
            \ElsIf{$w_{t,a}^{(s)} \le 2^{-s}, \ \forall \ a \in A_s$} 
                \State {$A_{s+1} = \{a\in A_s: $ \Comment{\textbf{(c)}}
                \\ \hspace{19ex} $m_{t,a}^{(s)} \ge \max_{j \in A_s}\!\! m_{t,j}^{(s)} - 2 \cdot 2^{-s}\}$ }
                \State $s \leftarrow s+1$
            \EndIf
        \EndWhile 
    \EndFor
	\end{algorithmic}
\end{algorithm}

%

We present our regret bound in the following theorem whose proof can be found in our supplementary.
\begin{theorem}\label{thm:suplogistic}
  Let $T\!\ge\! d$, $\tau = \sqrt{dT}$, $\alpha=3.5 \sqrt{\ln(\fr{2(2+\tau)\cd 2STK}{\delta})}$, and $Z =\fr{d}{\kappa^2} + \fr{\ln^2(K/\dt)}{d\kappa^2}$.
  Then,  w.p. at least $1-\dt$,
  \begin{align*}
      \!\mathsf{Reg}_T 
      &\le 10 \alpha\sqrt{dT\ln(T/d)\log_2(T)}
    \\&\quad+ O\del{\fr{\alpha^2}{\kappa}d\ln\lt(\fr{\alpha^2 T}{\kappa}\rt) + Z\ln^4(Z)}.
  \end{align*}
\end{theorem}
Our bound improves upon SupCB-GLM \citep{li2017provably} by removing the factor $1/\kappa = \Theta(\exp(S_*))$ in the leading term (i.e., $\sqrt{T}$ term).
Ours further improves the dependence on the non-leading term from $\tilde O(\frac{d^3}{\kappa^4})$ to $\tilde O(\frac{d}{\kappa^2})$.
Such an improvement parallels that of~\citet{faury2020improved} with $\tilde O(d\sqrt{T})$ upon UCB-GLM~\citep{li2017provably} with $\tilde O(\fr1\kappa d\sqrt{T})$.

Compared to Logistic-UCB-2 of \citet{faury2020improved}, our regret bound can be better or worse depending on the problem, which we summarize in three cases.
%
%
First, ours has a factor of $\sqrt{d\ln(K)}$ in the $\sqrt{T}$ term, which is a $\sqrt{d}$ factor better than theirs when $K = o(e^d)$.
Secondly, our bound's lower order term scales like $1/\kappa^2$, which is worse than $1/\kappa$ of Logistic-UCB-2.
Thirdly, while Logistic-UCB-2 manages to avoid an exponential dependence on $S_*$ in the leading term, the regret still has a factor $S_*^{1.5}$ and requires the knowledge of $S_*$.\footnote{Which may be removed if their algorithm uses a burn-in phase.}
In contrast, our bound does not depend on $S_*$ in the leading term nor require the knowledge of $S_*$.
\begin{remark}
  A parallel work by \citet{abeille20instancewise} achieves a leading term of $d\sqrt{\dot\mu(x_*^\T\theta^*)T}$ in the regret bound where $x_*$ is the best arm that is fixed throughout. 
  This is possible since their setting assumes a fixed arm set. 
  In contrast, our setting assumes that the arm set is changing, so the best arm can change in every time step.
  For this reason, we do not expect to achieve a factor like $\sqrt{\dot\mu(x_*^\T\theta^*)}$ in the leading term without further assumptions.
\end{remark}
\vspace{-.7em}

\vspace{-.5em}
\section{Future work}
\vspace{-.5em}
\label{sec:future}


Our confidence bound utilizes self-concordance and local analysis to significantly improve upon the existing state of the art results for the logistic MLE.
We remove a direct dependence on $\kappa^{-1}$ in the confidence width and relax significantly the requirement on the minimum sample size for the bound to be valid. To better leverage our knowledge burn-in condition, we hope to develop online procedures that adapt to $\theta^{\ast}$ instead of paying a worst case dependence in $\kappa$ to satisfy the burn-in condition.
Furthermore, understanding the optimal burn-in condition is an important open problem with practical implications.



Pure exploration for linear logistic models is largely under-explored, although its applications are abundant.
Exploiting the local nature of the logistic loss and closely working with non-uniform variances that naturally arise from the model is crucial in sample-efficient design of experiments. Our work is an important first step on understanding the true sample complexity of this problem and determining the precise dependence of the sample complexity on $\kappa_0^{-1}$ is an exciting direction. 

A major road block to developing practical contextual bandit algorithms is the fact that SupLogistic (and its ancestors like~\citet{auer2002using}) have to maintain independent buckets, and cannot share samples across buckets. It would be interesting to develop new algorithms that do not waste samples, without increasing the regret bound.
\citet{foster20beyond} have proposed such an algorithm but its dependence on the number of arms is sub-optimal.


\bibliography{references}
\bibliographystyle{icml2021}

\twocolumn
\newcommand{\sqrtf}[2]{\sqrt{\frac{#1}{#2}}}



\def \1{\mathbbm{1}}





\setlength{\abovedisplayskip}{4pt}
\setlength{\belowdisplayskip}{4pt}
\setlength{\abovedisplayshortskip}{4pt}
\setlength{\belowdisplayshortskip}{4pt}
\setlist[itemize]{topsep=.5pt,itemsep=0pt,parsep=2pt}
\setlist[enumerate]{topsep=.5pt,itemsep=0pt,parsep=2pt}

%

%

\onecolumn

\appendix
\addcontentsline{toc}{section}{Appendix} 
\part{} 
\parttoc 

\renewcommand\thesection{\Alph{section}}




\section{Proof of Theorem 1}


We state the full version of our concentration result.
\begin{theorem}\label{thm:concentration-supp}
  Let $\dt \le e^{-1}$. 
  Let $\hat\th_t$ be the solution of Eq.~{\normalfont(1)} in the main text where, for every $s\in[t]$, $y_s$ is conditionally independent from $x_1,\ldots,x_{s-1},x_{s+1},\ldots,x_t$ given $x_s$ (i.e., fixed design).
  Let $t_{\mathsf{eff}}$ be the number of distinct vectors in $\{x_s\}_{s=1}^t$.
  Fix $x\in\RR^d$ such that $\|x\|\le1$.
  Define ${\gamma(d)} = d + \ln(6(2+\teff)/\dt)$.
  If ${\xi^2_t} := \max_{s\in[t]} \|x_s\|^2_{H_{t}(\th^*)^{-1}} \le \fr{1}{\gamma(d)}$,
  \begin{align*}
      \PP \Big( |x^\T(\hat\th_t - \th^*)|  \le 3.5 \cd\|x \|_{H_{t}(
              \th^*)^{-1}} \sqrt{\ln(2(2+\teff)/\dt)},
            \qquad\qquad\qquad\qquad\qquad\\
              \forall x'\in\RR^d, \fr{1}{\sqrt{2.2}}\|x'\|_{(H_t(\th^*))^{-1}}
              \le \|x'\|_{(H_t(\hat\th_t))^{-1}} 
              \le \sqrt{2.2} \|x'\|_{(H_t(\th^*))^{-1}}
           \Big) \ge 1-\dt~,
  \end{align*}
  which implies the following empirical variance bound:
  \begin{align*}
      \PP\del{ |x^\T(\hat\th_t - \th^*) | \le 5.2 \cd\|x \|_{H_{t}(
              \hat\th_t)^{-1}} \sqrt{\ln(2(2+\teff)/\dt)} } \ge 1-\dt~.
  \end{align*} 
\end{theorem}

To improve the concentration inequality from~\citet{li2017provably}, we follow their analysis closely but exploit the variance term whenever possible.

We define the following:
\begin{itemize}
    \item Let ${H} := H_t(\th^*) = \sum_{s=1}^t \dot\mu(x_s^\T\th^*) x_s x_s^\T$.
    \item Let ${y_s} \in \cbr{0,1}$ be the binary reward when arm $x_s$ is pulled at time $s$. 
        Let ${\eta_s} := y_s - \mu(x_s^\T\th^*)$ and ${\sig_s^2} := \dot\mu(x_s^\T\th^*)$.
    \item Define ${z_t} := \sum_{s=1}^t \eta_s x_s$. 
    \item Let $\alpha(x,\th_1,\th_2) = \frac{\mu(x^\T\th_1) - \mu(x^\T\th_2)}{x^\T(\th_1 - \th_2)}$.
    We use the shorthand ${\alpha_s(\hat\th_t, \th^*)} := \alpha(x_s, \hat\th_t, \th^*)$.
    \item Let ${G}:= \sum_{s=1}^t \alpha_s(\hat\th_t,\th^*) x_s x_s^\T$.
    Note that by the optimality condition, 
    \begin{align}  \label{eq:optcond}
         z_t = \sum_s (\mu(x_s^\T\hat\th_t) - \mu(x_s^\T\th^*)) x_s = \sum_s \alpha_s(\hat\th_t, \th^*) x_s x_s^\T (\hat\th_t - \th^*) = G (\hat\th_t - \th^*)~.
    \end{align}
    \item Define ${g_t} := \sum_{s=1}^t \mu(x_s^\T \th)x_s$. The following identity is well-known (e.g.,  \citep[Proposition 1]{filippi2010parametric}):
    \begin{align}\label{eq:theta_G_equality}
      \|\hat\th_t - \th^*\|_G = \|g_t(\hat\th_t) - g_t(\th^*)\|_{G^{-1}}
    \end{align}
    
    \item Let ${E} := G-H$.
\end{itemize}

First, we assume the following event:
\begin{align}\label{eq:Q}
    {\cE_0} := \cbr{
    \forall s\in[t], \lt| \fr{\alpha_s(\hat\th_t, \th^*) - \dot\mu(x_s^\T\th^*)}{\dot\mu(x_s^\T\th^*)} \rt| \le \blue{Q} \text{ for some } Q>0
    }~,
\end{align}
which we will show is true later under suitable stochastic events.

\textbf{The main decomposition:}
We use the following decomposition based on~\eqref{eq:optcond} and tackle those two terms separately.
\begin{align*}
  |x^\T(\hat\th_t - \th^*) |
  = |x^\T G^{-1} z_t|
  = |x^\T (H+E)^{-1} z_t|
  &= |x^\T H^{-1}z_t - x^\T H^{-1}E (H+E)^{-1}z_t|
  \\&\le |x^\T H^{-1}z_t| + |x^\T H^{-1}E (H+E)^{-1}z_t|~.
\end{align*}
We bound the two terms separately.

\textbf{Term 1:} 
$   |x^\T H^{-1}z_t| = |\sum_s \la x, H^{-1} x_s\ra \eta_s|$
\vspace{.5em}
\\
Note that $H^{-1}$ is deterministic (unlike $G^{-1}$) conditioning on $\cbr[0]{x_1,\ldots,x_t}$, so we can apply the standard argument for the concentration inequality.
With the following Bernstein's inequality in mind, we assume the event $\cE_1(x)$ defined below.
\begin{lemma}\label{lem:bernstein}
    Let $\dt\le e^{-1}$ and define
    \begin{align*}
        \cE_1(x) : = \cbr{|x^\T H^{-1}z_t|   \le 2 \|x \|_{H^{-1}} \sqrt{\ln(2/\dt)} + \|x\|_{H^{-1}} \xi_t \ln(2/\dt)}~.
    \end{align*}
    Then,  $    \PP\del{\cE_1(x)} \ge 1- \dt$. 
\end{lemma}
\begin{proof}
    The proof can be found in Section~\ref{sec:sm-conf-aux}.
\end{proof}

\textbf{Term 2:} 
$ |x^\T H^{-1}E (H+E)^{-1}z_t|$
\vspace{.5em}\\

We have
\begin{align*}
  |x^\T H^{-1}E (H+E)^{-1}z_t|
  = |x^\T H^{-1} E G^{-1}z_t|
  &\le \|x \|_{H^{-1}} \|H^{-1/2} E H^{-1/2}\| \|G^{-1} z_t \|_H
\\&= \|x \|_{H^{-1}} \|H^{-1/2} E H^{-1/2}\| \|\hth_t - \th^* \|_H
\end{align*}

Let us study the term $\|H^{-1/2}E H^{-1/2} \|$.
For a symmetric matrix $A$, the singular values are the absolute values of the eigenvalues.
Thus, we have 
\begin{align*}
  \|A\| = \max\cbr{ \max_{x:\|x\|\le1} x^\T A x, \max_{x:\|x\|\le1} x^\T(-A) x }~.
\end{align*}
With this, we need to study both $x^\T H^{-1/2}E H^{-1/2} x$ and $x^\T H^{-1/2}(-E) H^{-1/2} x$.
Under the event $\cE_0$,
\begin{align*}
  &\max\{x^\T H^{-1/2}E H^{-1/2} x,~~ x^\T H^{-1/2} (-E) H^{-1/2} x\}
  \\&\le x^\T H^{-1/2} \del{  \sum_s  |\alpha_s(\hat\th_t, \th^*) - \dot\mu(x_s^\T\th^*) |  x_s x_s^\T  } H^{-1/2} x
  \\&\le x^\T H^{-1/2} \del{  \sum_s  Q\dot\mu(x_s^\T\th^*)  x_s x_s^\T  } H^{-1/2} x
  \\&= Q \|x\|^2 \le Q
      \tag{$\because \|x\|\le1$}
  \\\implies \|H^{-1/2}E H^{-1/2} \| &\le Q~.& 
\end{align*}

For $\|\hth_t - \th^* \|_H$, we first use the lemma below to bound it by $(1+D)\|z_t \|_{H^{-1}}$.
The key is to use the self-concordance control lemma~\citep[Lemma 9]{faury2020improved}, we can relate $G$ and $H$ as a function of $D$.
If $A$ and $B$ are matrices, then we use $A \succeq B$ to mean that $A - B$ is positive semi-definite.
\begin{lemma}\label{lem:selfconcordance-H}
  Let $D = \max_{s\in[t]} |x_s^\T(\hat\th_t - \th^*)|$.
  \begin{align*}
    G \succeq \fr{1}{1+D} \cd H
  \end{align*}
  where $A\succeq B$ means that $A - B$ is positive semi-definite.
\end{lemma}
\begin{proof}
  We first note that, by the self-concordance control lemma~\citep[Lemma 9]{faury2020improved},
  \begin{align*}
    \alpha_s(\th_1, \th_2) 
    \ge \fr{\dot\mu(x_s^\T \th_2)}{1+|x_s^\T(\th_1 - \th_2)|}~.
  \end{align*}
  Then, 
  \begin{align*}
    G
    = \sum_s^{t} \alpha(x_s,\th_1, \th_2) x_s x_s^\T
    \succeq  \fr{1}{1+ \max_{s\in[t]} |x_s^\T(\th_1 - \th_2)|}\sum_s^{t} \dot\mu(x_s^\T\th_2) x_s x_s^\T~.
  \end{align*}
  Notice that the sum on the RHS is $H_t(\th_2)$.
\end{proof}

We then bound $\|z_t \|_{H^{-1}}$ by the following concentration result via the covering argument.
\begin{lemma} \label{lem:bernstein-d}
  Recall $\xi_t = \max_{s\in[t]} \|x_s\|_{H^{-1}}$.
  Let $\dt\le e^{-1}$.
  Define  
  \begin{align*}
    {\cE_2} := \cbr{\enVert{\sum_s^t \eta_s x_s}_{H^{-1}} \le 2\sqrt{d + \ln(6/\dt)} + \xi_t(d + \ln(6/\dt)) =: {\sqrt{\beta_t}}}~.    
  \end{align*}
  Then, $\PP(\cE_2) \ge 1-\dt$.
\end{lemma}
\begin{proof}
  See Section~\ref{sec:sm-conf-aux}.
\end{proof}

Let ${D} := \max_{s\le t}|x_s^\T(\hat\th_t - \th^*)|$.
Therefore,
\begin{align*}
    |x^\T H^{-1}E (H+E)^{-1}z_t|
  &\le \|x \|_{H^{-1}} Q (1+D)\sqrt{\beta_t} \tag{Lemma~\ref{lem:selfconcordance-H}}
\end{align*}

To summarize, under $\cE_0 \cap \cE_1(x) \cap \cE_2$, we have
\begin{align}\label{eq:concentration-0}
  \begin{aligned}
    &|x^\T(\hat\th_t - \th^*) |
    \\&\le  2 \|x \|_{H^{-1}} \sqrt{\ln(2/\dt)} + \|x\|_{H^{-1}} \xi_t \ln(2/\dt) + \|x \|_{H^{-1}} (1+D) Q \del{2\sqrt{d+ \ln(6/\dt)} + \xi_t (d+\ln(6/\dt))}~.
  \end{aligned}
\end{align}
We now aim to control $Q(1+D)$ to be small so that the entire RHS is $O(\|x\|_{H^{-1}} \sqrt{\ln(1/\dt)})$.

\subsection{Controlling \texorpdfstring{(1+D)Q}{}}

Assume $\cE_2$.
We first assume that for every $s\in[t]$, $\cE_1(x_s)$ is true and then take the maximum over $s$ on the inequality implied by $\cE_1(x_s)$ to obtain 
\begin{align}\label{eq:concentration-1}
  \begin{aligned}
    D 
    &= \max_s |x_s^\T(\hat\th_t - \th^*) |
  \\&\le  2 \xi_t \sqrt{\ln(2/\dt)} + \xi_t^2 \ln(2/\dt) + \xi_t (1+D) Q \sqrt{\beta_t}~.
\end{aligned}  
\end{align}

To control $(1+D)Q$, one can show that the self concordance control lemma~\citep[Lemma 9]{faury2020improved} implies the following, which we use to motivate our choice of $Q$ and satisfy $\cE_0$:
\begin{align*}
  \lt|\fr{\alpha_s(\hat\th_t, \th^*) - \dot\mu(x_s^\T\th^*)}{\dot\mu(x_s^\T\th^*)}\rt| 
    &\le \max\cbr{\fr{D}{1+D}, \fr{e^D - 1 - D}{D}} =: {Q}~.
\end{align*}
Then, one can show that
\begin{align} \label{eq:cond-D}
  D \le \fr35 \implies (1+D)Q = \max\cbr{D, \fr{e^D - 1 - D}{D}(1+D)} \le D~.
\end{align}
So, if we can ensure $D \le \fr35$, then we have $(1+D)Q \le D$, which can be applied to~\eqref{eq:concentration-1} to arrive at
\begin{align}\label{eq:concentration-2}
  \begin{aligned}
    D &\le  \fr{2 \xi_t \sqrt{\ln(2/\dt)} + \xi_t^2 \ln(2/\dt)}{1-\xi_t \sqrt{\beta}}~,
  \end{aligned}  
\end{align}
assuming that $1 - \xi_t \sqrt{\beta_t} > 0$.
Thus, it remains to
\begin{itemize}
  \item [(A)] find a sufficient condition for $D \le \fr35$ and $1-\xi_t \sqrt{\beta_t} > 0$,
  \item [(B)] bound the RHS of~\eqref{eq:concentration-2} to obtain the bound on $D$, and
  \item [(C)] use the bound on~\eqref{eq:concentration-0} to get the final bound. 
\end{itemize}

For $(A)$, we prove the following lemma.
\begin{lemma} \label{lem:cond-xi_t}
  Under $\cE_2$, we have
\begin{align}\label{eq:xi_t}
  \xi_t \le \fr{0.17}{\sqrt{d + \ln(6/\delta)}}\implies \xi_t\sqrt{\beta_t} \le \fr38 \implies D \le \fr 3 5 ~.
\end{align}
\end{lemma}
\begin{proof}
  Using Lemma~\ref{lem:selfconcordance-H},
  \begin{align*}
    D^2 &\le \xi_t^2 \|\hat\th_t - \th^* \|^2_H 
    \le \xi_t^2 (1+D)  \| \hat\th_t - \th^* \|^2_G
    \\&=  \xi_t^2 (1+D) \cd \| g_t(\hat\th_t) - g_t(\th^*) \|^2_{G^{-1}}
    \\&\le  \xi_t^2 \cd (1+D)^2\| g_t(\hat\th_t) - g_t(\th^*) \|^2_{H^{-1}}
    \\&\le  (1+D)^2\xi_t^2 \beta_t   \mytag{\text{by~\eqref{eq:optcond} and $\cE_2$}}
    \\     \implies D &\le  (1+D)\xi_t\sqrt{\beta_t}
    \\     \implies D &\le \fr{\sqrt{\beta_t}\xi_t}{1-\sqrt{\beta_t}\xi_t}~.
  \end{align*}
  where the last line requires an assumption that $1-\sqrt{\beta_t} \xi_t >0$.
  For this, we require that $\sqrt{\beta_t}\xi_t \le \fr38$.
  Then,
  \begin{align*}
    D \le \fr{8}{5} \sqrt{\beta_t} \xi_t~.
  \end{align*}
  In order to control the RHS above by $3/5$, we need to satisfy $\xi_t \le \fr{5}{8\sqrt{\beta_t}} \cd \fr 3 5$.
  Note, however, $\beta_t$ depends on $\xi_t$, so we need to solve for $\xi_t$.
  Let $C = d + \ln(6/\dt)$.
  Then, we need to solve
  \begin{align*}
    \xi_t \le \fr{3}{8(2\sqrt{C} + \xi_t C)}~,
  \end{align*}
  which is quadratic in $\xi_t$.
  Solving it for $\xi_t$, we have $\xi \le \fr{\sqrt{11/2} - 2}{2} \cd \frac{1}{\sqrt{C}} = \fr{0.172...}{\sqrt{C}}$.
  Thus, it suffices to require $\xi_t \le \fr{0.17}{\sqrt{C}}$.  
\end{proof}
Hereafter, we assume that $\xi_t \le \fr{0.17}{\sqrt{d + \ln(6/\delta)}}$.

For $(B)$, by Lemma~\ref{lem:cond-xi_t}, we can deduce from~\eqref{eq:concentration-2} that
\begin{align*}
  D \le \frac{8}{5}\cd\del{2 \xi_t \sqrt{\ln(2/\dt)} + \xi_t^2 \ln(2/\dt)}.
\end{align*}
For $(C)$, we now turn back to the initial concentration inequality~\eqref{eq:concentration-0}.
We first bound the last term of~\eqref{eq:concentration-0}:
\begin{align*}
 \|x \|_{H^{-1}} (1+D) Q \sqrt{\beta_t}
    &\le\|x \|_{H^{-1}} D \sqrt{\beta_t}  \tag{$\because$ Lemma~\ref{lem:cond-xi_t} \& \eqref{eq:cond-D}}
  \\&\le  \|x \|_{H^{-1}} \frac{8}{5}\cd\del{2 \xi_t \sqrt{\ln(2/\dt)} + \xi_t^2 \ln(2/\dt)} \cd \sqrt{\beta_t} 
  \\&\le  \|x \|_{H^{-1}} \frac{8}{5}\cd\del{2 \cd\fr38 \sqrt{\ln(2/\dt)} + \xi_t^2 \sqrt{\beta_t} \ln(2/\dt)} \tag{$\because$ Lemma~\ref{lem:cond-xi_t}}
  \\&\le  \|x \|_{H^{-1}} \frac{8}{5}\cd\del{2 \cd\fr38 \sqrt{\ln(2/\dt)} + \xi_t \fr38 \ln(2/\dt)} 
  \\&\le  \|x \|_{H^{-1}} \frac{8}{5}\cd\del{2 \cd\fr38 \sqrt{\ln(2/\dt)} + \frac{0.17}{\sqrt{d + \ln(6/\dt)}} \fr38 \sqrt{\ln(2/\dt)}\sqrt{\ln(2/\dt)}} 
  \\&\le  \|x \|_{H^{-1}} \frac{8}{5}\cd\del{2 \cd\fr38 \sqrt{\ln(2/\dt)} + 0.17\cd \fr38 \sqrt{\ln(2/\dt)}} 
  \\&\le 1.3 \cd \|x \|_{H^{-1}}\sqrt{\ln(2/\dt)}~.
\end{align*}
Similarly, one can show that $\|x\|_{H^{-1}} \xi_t \ln(2/\dt) \le 0.17 \cd \|x\|_{H^{-1}} \sqrt{\ln(2/\dt)}$.
Altogether, under our condition on $\xi_t$, $\cE_2$, $\cE_1(x)$, and $\cap_{s=1}^t  \cE_1(x_s)$, Eq.~\eqref{eq:concentration-0} implies that
\begin{align*}
  |x^\T(\hth_t - \th^*)| \le 3.5 \cd \|x\|_{H^{-1}} \ln(2/\dt)~.
\end{align*}
We replace $\dt$ with $\dt/(\teff+2)$ to obtain the concentration inequality in our theorem statement.
Here, we have $\teff$ instead of $t$ because $\cE(x)$ and $\cE(x')$ are identical events when $x = x'$.

Furthermore, the following lemma shows that the empirical variance is within a constant factor of the true variance.
One can easily check that the condition in the lemma is satisfied under the events we have assumed and the condition on $\xi_t$, which implies the theorem statements.
\begin{lemma}\label{lem:eqvar}
  Suppose $D=\max_{s\in[t] } |x_s^\T (\hat\th_t - \th^*)|\le 1$.
  Then, for all $x$, 
  \begin{align*}
    \fr{1}{\sqrt{2D+1}}\|x\|_{(H_t(\th^*))^{-1}}
    \le \|x\|_{(H_t(\hat\th_t))^{-1}} 
    \le \sqrt{2D+1} \|x\|_{(H_t(\th^*))^{-1}}~.
  \end{align*}
\end{lemma}
\begin{proof}
  See Section~\ref{sec:sm-conf-aux}.
\end{proof}

\subsection{Proof of Auxiliary Results}
\label{sec:sm-conf-aux}

\paragraph{Proof of Lemma~\ref{lem:bernstein}.}
  It suffices to bound $\PP(\cE_1(x) \mid x_1,\ldots,x_t) \ge 1-\dt$ because this implies that $\PP(\cE_1(x)) = \int_{x_1,\ldots,x_t} \PP(\cE_1(x) \mid x_1,\ldots,x_t) \diff F(x_1,\ldots,x_t) \ge 1-\dt$.
  For brevity, we omit the conditioning on $x_1,\ldots,x_t$ from probability statements for the rest of the proof.
  
  One can easily extend \citet[Lemma 7]{faury2020improved} to show the following: if a random variable $Z$ is centered and bounded ($|Z|\le R$ for some $R>0$) with variance $\sig^2$, then we have $\EE[\exp(\phi Z - \phi^2 \sig
  ^2)] \le 1$ for any deterministic $|\phi|\le 1/R$. 
  Note that \citet[Lemma 7]{faury2020improved} is a special case of $R=1$.
  
  Note that scaling $Z$ by a constant $c$ would increase $R$ to $cR$ and $\sig^2$ to $c^2\sig^2$.
  This observation leads to: $\forall \phi \text{ with } |\phi| \le \fr{1}{\max_{s\in[t]} |\la x, H^{-1} x_s\ra |}$,
  \begin{align*}
      \EE\sbr{\exp\cbr{ \phi \sum_{s=1}^t \la x, H^{-1} x_s\ra \eta_s - \phi^2 \underbrace{\sum_{s=1}^t \la x, H^{-1}x_s\ra^2 \sig_s^2 }_{ = \|x\|^2_{H^{-1}}}  }   } \le 1 ~.
  \end{align*}
  where the equality is by the definition $H = \sum_s \sig_s^2 x_s x_s^\T$ and $\sig_s^2 := \dot\mu(x_s^\T\th^*)$.
  Thus, by Markov's ineq., we have, w.p. at least $1-\dt$,
  \begin{align*}
      x^\T H^{-1}z_t \le \phi \|x \|^2_{H^{-1}} + \fr1\phi \ln(1/\dt)~.
  \end{align*}
  Then, one can tune $\phi=\sqrt{\fr{\ln(1/\dt)}{\|x\|^2_{H^{-1}}}} \wedge \fr1{\max_{s\in[t]} |\la x, H^{-1} x_s\ra| }$ to show that
  \begin{align*}
      x^\T H^{-1}z_t 
      &\le 2 \|x \|_{H^{-1}} \sqrt{\ln(1/\dt)} + \max_{s\in[t]} |\la x, H^{-1} x_s\ra|  \ln(1/\dt)
      \\&\le 2 \|x \|_{H^{-1}} \sqrt{\ln(1/\dt)} + \|x\|_{H^{-1}} \xi_t \ln(1/\dt)~.
  \end{align*}
  Using $|\la x, H^{-1}z_t\ra| \le \max\{\la x, H^{-1}z_t\ra, \la -x, H^{-1}z_t\ra\}$, one can make the same argument for $-\la x, H^{-1}z_t\ra$, substitute $\dt$ with $\dt/2$, and then apply the union bound.

\paragraph{Proof of Lemma~\ref{lem:bernstein-d}.}

    The proof closely follows~\citet{li2017provably} but we employ the Bernstein inequality.
    Let ${\cB(1)}$ be the Euclidean ball of radius $1$ and ${\hat\cB(1)}$ be a $1/2$-cover of $\cB(1)$.
    It is well-known that one can find a cover $\hat\cB(1)$ of cardinality $6^d$; see \citet[Lemma 4.1]{pollard1990empirical}. 
    In this proof, we use the shortcut $H := H_t(\th^*)$.

    Note that $\|z_t\|_{H^{-1}} = \|H^{-1/2}z_t\|_2 = \sup_{a\in\cB(1)} \la a, H^{-1/2}z_t\ra $.
    Fix $x\in\RR^d$.
    Let $\hat x$ be the closes point to $x$ in the cover $\hat\cB(1)$.
    Then,
    \begin{align*}
        \la x, H^{-1/2}z_t \ra 
        &= \la \hat x, H^{-1/2}z_t\ra + \la x-\hat x, H^{-1/2}z_t\ra
        \\&= \la \hat x, H^{-1/2}z_t\ra + \|x-\hat x\| \la \fr{x-\hat x}{\| x-\hat x\|}, H^{-1/2}z_t\ra
        \\&\le \la \hat x, H^{-1/2}z_t\ra + \fr12 \cd \sup_{a\in \cB(1)}\la a, H^{-1/2}z_t\ra
        \\&= \la \hat x, H^{-1/2}z_t\ra + \fr12 \cd \|z_t\|_{H^{-1}}~.
    \end{align*}
    Taking sup over $x\in\cB(1)$ on both sides, we have 
    \begin{align*}
        \|z_t\|_{H^{-1}}   \le  \la \hat x, H^{-1/2}z_t\ra + \fr12 \cd \|z_t\|_{H^{-1}}
        \implies 
        \|z_t\|_{H^{-1}} \le 2 \la \hat x, H^{-1/2}z_t\ra ~.
    \end{align*}
    This implies that, for $q>0$,
    \begin{align}\label{eq:unionbound-0}
        \PP(\|z_t\|_{H^{-1}} > q ) 
        \le\PP(2\sup_{\hat x\in\hat\cB(1)} \la \hat x, H^{-1/2}z_t\ra > q)
        ~\le \sum_{\hat x \in \hat\cB(1)} \PP(\la \hat x, H^{-1/2}z_t \ra>q/2)~.
    \end{align}
    It remains to bound $\PP(\la \hat x, H^{-1/2}z_t \ra>q/2)$ for any $\hat x$ and then apply the union bound.
    Let $\phi>0$ and
    \begin{align*}
        M_t 
        &= \exp\del{\phi \sum_s \la \hat x,H^{-1/2}x_s\ra \eta_s - \phi^2\del{\sum_s\la\hat x, H^{-1/2} x_s \ra^2\sig_s^2}}~.
        \\&= \exp\del{\phi \sum_s \la \hat x,H^{-1/2}x_s\ra \eta_s - \phi^2 \|\hat x\|^2 }~.
    \end{align*}
    Let $M_0=1$ as a convention. 
    One can show that $M_t$ is supermartingale for all $|\phi|\le \fr{1}{\max_{\hat x \in\hat\cB(1), s\in[t]}\la \hat x, H^{-1/2}x_s\ra}$.
    In fact, for simplicity, we require a tighter condition on $\phi$ which is $|\phi| \le  \fr{1}{\max_{ s} \|x_s\|_{H^{-1}}}  = \fr{1}{\xi_t}$.
    Thus, w.p. at least $1-\dt$, we have
    \begin{align*}
        \sum_s \la \hat x, H^{-1/2} x_s \ra \eta_s 
        &\le \phi \|\hat x\|^2 + \fr1\phi \ln(1/\dt) ~.
    \end{align*}
    Applying $\|\hat x\|\le 1$ and the usual tuning of $\phi = \sqrt{\ln(1/\dt)} \wedge \fr{1}{\xi_t} $ leads to the RHS being $ 2 \sqrt{\ln(1/\dt)} + \xi_t \ln(1/\dt)$. 
    Replacing $\dt$ above with $\fr{\dt}{6^d}$ and taking the union bound at~\eqref{eq:unionbound-0} conclude the proof.

\paragraph{Proof of Lemma~\ref{lem:eqvar}.}
  Let $\alpha(z_1,z_2) = \fr{\mu(z_1) - \mu(z_2)}{z_1 - z_2}$.
  Let $s\in[t]$.
  Because $\alpha(z_1,z_2) = \alpha(z_2,z_1)$, \citet[Lemma 9]{faury2020improved} implies
  \begin{align*}
    \dot\mu(x_s^\T \th^*) \fr{1-\exp(-D)}{D} \le \alpha(x_s^\T \th^*, x_s^\T \hat\th_t) \le \dot \mu(x_s^\T \th^*) \cd \fr{\exp(D) - 1}{D}
  \end{align*}
  and 
  \begin{align*}
    \dot\mu(x_s^\T \hat\th_t) \fr{1-\exp(-D)}{D} \le \alpha(x_s^\T \th^*, x_s^\T \hat\th_t) \le \dot \mu(x_s^\T \hat\th_t) \cd \fr{\exp(D) - 1}{D}~.    
  \end{align*}
  Then, 
  \begin{align*}
    \dot\mu(x_s^\T \hat\th_t)
    \ge \fr{D}{\exp(D)-1} \cd \alpha(x_s^\T \th^*, x_s^\T \hat\th_t) 
    \ge \fr{D}{\exp(D)-1} \fr{1-\exp(-D)}{D} \cd \dot\mu(x_s^\top \th^*) 
    \sr{(a)}{\ge} \fr{1}{2D + 1} \dot\mu(x_s^\top \th^*) 
  \end{align*}
  where $(a)$ is due to the following fact: using $z\le1\implies e^{z} \le z^2+z+1$, we have $\fr{D}{\exp(D)-1} \fr{1-\exp(-D)}{D} = \fr{e^D - 1}{e^D(e^D - 1)} = \fr{1}{e^D} \ge \fr{1}{D^2+D+1} \ge \fr{1}{2D+1}$.
  This implies that $H_t(\hat\th_t)\succeq \fr{1}{2D+1} H_t(\th^*)$.
  This concludes the proof of the second inequality.
  One can prove the other inequality similarly.

\section{\texorpdfstring{$\kappa^{-1}$}{}-free conditioning for Theorem~\ref{thm:concentration-supp}}

In this section, we consider a case where the burn-in condition (i.e., the requirement on $\xi_t$) in our Theorem~\ref{thm:concentration-supp} can be satisfied without spending $\kappa^{-1} = \min_{x:~ \norm x \le 1} \dot\mu(x^\T \th^*)=\Theta(\exp(S))$ samples where $S = \|\th^*\|$.
More specifically, we show that it is possible to use a sample size that is
\textit{polynomial} rather than \textit{exponential}in $S^*$ to satisfy our burn-in condition. 
The implication of this is that our improved burn-in condition $\xi^2_t = \max_{s\in[t]} \|x_s\|^2_{H_{t}(\th^*)^{-1}} \le O(\frac{1}{d+\ln(t/\dt)})$ is fundamentally different from that of \citet[Theorem 1]{li2017provably} which requires $\frac{1}{\lammin(V)} \le O(\frac{1}{\kappa^{-4} (d^2 + \ln(1/\dt))})$. Indeed their condition can only be satisfied after $\Omega(\exp(S))$ burn-in samples at all times.
The construction is based on the Gaussian measurements that are common in practice and often considered in the compressed sensing literature~\cite{plan12robust}.

\textbf{Gaussian Assumption:} We consider $t$ arms sampled from the following Gaussian distribution: $ x_s \sim \cN(0,\fr1d I), 1\leq s\leq t$. We define $r := S^2/d$ and further assume $d = \Omega(S^2)$ 
so that $r \le 1$.

\textbf{Note:} though this will violate the assumption that $\|x_s\|\le1$, needed for the theorem, one can show that for large enough $d$, the norm of $x_s$ concentrates around $1$.
Using this, one can find a constant $c\leq 1$ so that with high probability a sample $x\sim \cN(0,\fr c d I)$ satisfies $\|x\|\leq 1$, and then apply our argument below.

Continuing, under our Gaussian assumption, we have $x_s^\T\th^* \sim \cN(0,S^2/d)$.
This implies that, w.p. at least $1-\dt$, we have $\forall s\in[t], |x_s^\T\th^*| \le \sqrt{\frac{2S^2}{d}\ln(2t/\dt)} =: W$.
Let $V = \sum_{s=1}^t x_s x_s^\T$.
Then with high probability,
\begin{align*}
    H(\th^*)  \succeq \dot\mu(-W) V = \dot\mu(W) V
\end{align*}
Furthermore utilizing \citet[Proposition 1]{li2017provably} and our Gaussian assumption on the samples $\{x_s\}_{s=1}^t$ implies that, given $B>0$, there exists an absolute constant $C_1$ such that, w.p. at least $1-\dt$, 
\begin{align}\label{eq:matrix_conc}
    t \ge C_1 \cd d^2(d + \ln(1/\dt)) + 2d B \implies \lammin(V) \ge B
\end{align}
Thus, under the condition on $t$ above, on a high probability event we have that 
\begin{align*}
    \xi_t^2 \le \max_{x:\|x\|\le1} \|x\|^2_{H(\th^*)^{-1}} 
    \le \fr{1}{\dot\mu(W)} \max_{x:\|x\|\le1} \|x\|^2_{V^{-1}} 
    \le \fr{1}{\dot\mu(W)} \frac{1}{\lammin(V)}
    \le \fr{1}{\dot\mu(W)B}
\end{align*}
It remains to control the RHS above to be no larger than $\fr{1}{d+\ln(6(2+t)/\dt)}$, which means that we will satisfy the burnin condition of Theorem~\ref{thm:concentration-supp}.
%
Since $6(2+t) \le 18t$, it suffices to show that
\begin{align*}
    \dot\mu\del{\sqrt{\fr{S^2}{d}\ln(t/\dt)}}B &\ge  d + \ln(18 t/\dt)
\end{align*}
Using the fact that $\dot\mu(z) \ge \fr14 e^{-z} $, it thus suffices to show that
\begin{align*}
    \fr14 \exp\del{-\sqrt{2r \log(2t/\dt)}} \cd B
    \ge d + \ln(18 t/\dt) 
\end{align*}
We will make the simple choice of
\begin{align*}
    B &:=   \fr14 \exp\del{\sqrt{2r \log(2t/\dt)}} \del{d + \ln(18 t/\dt) }
    \\&\le  \fr14 \exp\del{1 + \ln((2t/\dt)^{1/2})} \del{d + \ln(18 t/\dt) } \tag{$r\le1$ and AM-GM ineq.}
    \\&=    \fr e 4 (\fr{2t}{\dt})^{1/2} \del{d + \ln(18 t/\dt) } 
\end{align*}
With this choice of $B$ it suffices to compute the lower bound on the right hand side of~\eqref{eq:matrix_conc}.
With algebra, one can show that there exists
\begin{align*}
    t_0 = \tilde O( d^2\ln(1/\dt) + \fr{d^4}{\dt} )
\end{align*}
where $\tilde O$ hides polylogarithmic factors.
such that $t \ge t_0$ implies the condition of~\eqref{eq:matrix_conc}.

To summarize, we just showed that, there exists an absolute constant $C$ such that, w.p. at least $1-2\dt$, 
\begin{align*}
    t \ge C \cd(d^2 \ln(1/\dt) + \frac{d^4}{\dt}) \implies \xi_t \le \frac{1}{\gamma(d)}
\end{align*}
when $r=S^2/d \leq 1$.
Simply setting $r=1$, we have $d = S^2$, so our the statement above implies that the sample size needs to be only \textit{polynomial} in $S$ for our choice of measurements.
This is in stark contrast to the result of \citet{li2017provably} that requires the sample size to be \textit{exponential} in $S$ for \textit{any} set of measurements.

% ...

\section{Proofs for GLM-Rage}

\begin{algorithm}[h]
\caption{BurnIn}
    \label{alg:rage-burnin-supp}
	\begin{algorithmic}[1]
        \Input $\mc{X}, \kappa_0, \delta, \gamma(d) = d+\log(6(2+|\mc{X}|)/\delta)$
		\Initialize $\lambda_0 = \arg\min_{\lambda\in \Delta_{\mc{X}}}\max_{x\in \mathcal{X}} \|x\|_{A(\lambda)^{-1}}^{2}$ 
		\Initialize $n_0 = 3(1+\epsilon)\kappa_0^{-1}d\gamma(d)\log(2|\cX|(2+|\cX|)/\delta)$
		\State $x_1, \cdots, x_{n_0}\leftarrow \text{round}(n_0, \lambda_0, \epsilon)$
		\State Observe associated rewards $y_1, \cdots, y_{n_0}$
		\State \Return MLE ${\theta}_0$ \Comment{Use Eq~\eqref{eq:mle}}
	\end{algorithmic}
\end{algorithm}

\begin{algorithm}[h]
\caption{RAGE-GLM }
        \label{alg:rage-glm-supp}
	    \begin{algorithmic}[1]
            \Input {$\epsilon$, $\delta$, $\mc{X}$, $\mc{Z}$, $\kappa_0$, effective rounding procedure $\text{round}(n, \epsilon, \lambda)$}
            
            \Initialize{$k=1, \mc{Z}_1 = \mc{Z}, r(\epsilon) = d^2/\epsilon$}
            \State {$\theta_0\leftarrow$ \textbf{BurnIn}($\mc{X}, \kappa_0$)}  \Comment{Burn-in phase}
            \While{$|\mc{Z}_k| > 1$ } \Comment{Elimination phase}
            \State\label{line:expDesign-2}Define\[f(\lambda) := \max\Big[ 
                    \gamma(d)\max_{x\in \mc{X}} \|x\|_{H(\lambda, \hat{\theta}_{k-1})^{-1}}^2 ,
                    2^{2k}\cdot(\cone)^2 
                    \max_{z,z'\in\mc{Z}_k} \|z-z'\|_{H(\lambda, \hat{\theta}_{k-1})^{-1}}^2 
                    \Big]\log(2\max\{|\cX|, |\mc{Z}|\}k^2(2+|\cX|)/\delta)\] 
            \State $\lambda_k = \argmin_{\lambda\in \Delta_{\mc{X}}} f(\lambda)$
            
            \State{$n_k = \max\{3(1+\epsilon)f(\lambda_k), r(\epsilon)\}$}
            \State{$x_1, \cdots, x_{n_k}\leftarrow \text{round}(n,\epsilon,\lambda)$} 
            \State{Observe rewards $y_1, \cdots, y_{n_k}\in \{0,1\}$}
            \State{Compute the unregularized MLE $\hat\theta_k$}
            \State{$\hat{z}_k = \arg\max_{z\in\mc{Z}_k} \hat\theta_k^\top z$}
            \State{$\mc{Z}_{k+1}\leftarrow \mc{Z}_k \setminus \left\{z\in \mc{Z}_k:\exists z'\in \mc{Z}_k,\  \langle z'-z,\hat\theta_k\rangle \geq \cone\|z'-z\|_{H_k( \hat\theta_{k-1})^{-1}}\sqrt{3\log(2|\mc{Z}|(2+|\cX|)k^2/\delta)}\right\}$} 
            \State{$k\gets k+1$}
			\EndWhile
			\State \Return $\hat{z}_k$
		\end{algorithmic}
	\end{algorithm}

\textbf{Burn-In Results}
\begin{lemma}\label{lem:burnin}
    For $\delta \leq 1/16$ and $d\geq 4$, with probability greater than $1-\delta$, for all $\lambda\in \Delta_{\mc{X}}$, 
    $\frac{1}{3}H(\lambda,\theta^{\ast}) \leq H(\lambda,\hat\theta_{0})\leq 3H(\lambda,\theta^{\ast})$
\end{lemma}

\begin{proof}
Firstly note that, 
\[H(\lambda_0, \theta^{\ast})\geq \sum_{x\in \mathcal{X}}\lambda_{0,x} \kappa_0 xx^{\top} \geq \kappa_0 A(\lambda_0)\]
So for any $x\in \mathcal{X}$, 
\[\|x\|^2_{H(\lambda_0, \theta^{\ast})^{-1}}\leq \kappa_0^{-1}\|x\|^2_{A(\lambda_0)^{-1}}\]


Define $H_0(\theta^{\ast}) = \sum_{s=1}^{n_0} \dot\mu(x_s^{\top}\theta^{\ast})x_sx_s^{\top}$. Thus at the end of the burn-in phase, 
\begin{align*}
        \max_{x\in\mathcal{X}} \|x\|_{H_0^{-1}(\theta^{\ast})}^{2} \tag{Lemma~\ref{lem:rounding} rounding}
        &\leq \frac{(1+\epsilon)}{n_0}\max_{x\in\mathcal{X}} \|x\|_{H^{-1}(\lambda_0, \theta^{\ast})}^{2}\\
        &\leq \frac{3(1+\epsilon)}{n_0}\kappa_0^{-1}\max_{x\in\mathcal{X}} \|x\|_{A^{-1}(\lambda_0)}^{2}\\
        &\leq \frac{3(1+\epsilon)\kappa_0^{-1}d}{n_0} \tag{Kiefer-Wolfowitz}\\
        &\leq \frac{1}{\gamma(d)\log(2|\cX|(2+|\mc{X}|)/\delta)}
\end{align*}
where we have employed the Kiefer-Wolfowitz theorem \citep[Theorem 21.1]{lattimore2020bandit}, which states that $\min_{\lambda\in \Delta_{\mc{X}}}\max_{x\in\mc{X}} \|x\|^2_{A(\lambda)^{-1}} = d$. 
In particular this implies using Theorem~\ref{thm:concentration-supp},
\begin{align*}
        |x^{\top}(\theta^{\ast} - \hat\theta_0)| 
        &\leq \cone\sqrt{\|x\|^2_{(H_0(\theta^{\ast}))^{-1}}\log(2|\cX|(2+|\mc{X}|)/\delta)}\\
        &\leq \cone\sqrtf{\log(2|\cX|(2+|\mc{X}|)/\delta)}{\gamma(d)\log(2|\cX|(2+|\mc{X}|)/\delta)}\\
        &\leq 1
\end{align*}

With this, we apply Lemma~\ref{lem:self-concordance} to conclude the proof. 
%

\end{proof}

Define the events
\[\mc{R}_{k} = \{\frac{1}{3}H(\lambda,\theta^{\ast}) \leq H(\lambda,\hat\theta_{k})\leq 3H(\lambda,\theta^{\ast}), \forall \lambda\in \Delta_{\mc{X}}\}, k\geq 0\]
and 
\[\mc{E}_{2, k} = \{\forall z \in \mc{Z}_k, |\langle z^{\ast} - z, \hat{\theta}_k - \theta^{\ast}\rangle| \leq 2^{-k}\}, k\geq 1.\]

In addition, define $\mc{E}_1 = \cap_{k=0}^{\infty} \mc{R}_k$ and $\mc{E}_{2} = \cap_{k=1}^{\infty} \mc{E}_{2,k}$.

\begin{lemma}[Closeness of $\theta_t$]\label{lem:closeness}
We have that $\P(\mc{R}_{k}|\mc{R}_{k-1}, \cdots, \mc{R}_0)\geq 1-2\delta$, i.e. for all $k \geq 1$, $\frac{1}{3}H(\lambda_k,\theta^{\ast}) \leq H(\lambda_k,\hattheta_{k-1})\leq 3H(\lambda_k,\theta^{\ast})$
\end{lemma}


\begin{proof}

We proceed by induction. The base case of $t=0$, is handled by Lemma~\ref{lem:burnin} above. Assume that the event $\mc{R}_{k-1}$ holds. On this event, for $k>1$, we first verify that $\max_{x\in \mc{X}} \|x\|_{H_t(\theta^{\ast})^{-1}}^2\leq 1/\gamma(d)$ 
\begin{align*}
\max_{x\in\mathcal{X}} \|x\|_{H_k^{-1}(\theta^{\ast})}^{2} \tag{Lemma~\ref{lem:rounding} rounding}
        &\leq \frac{(1+\epsilon)}{n_k}\max_{x\in\mathcal{X}} \|x\|_{H^{-1}(\lambda_k, \theta^{\ast})}^{2}\\
        &\leq \frac{3(1+\epsilon)}{n_k} \max_{x\in\mathcal{X}} \|x\|_{H^{-1}(\lambda_k, \hat{\theta}_{k-1})}^{2}\tag{On event $\mc{R}_{k-1}$}\\
        &= \frac{1}{\gamma(d)\log(2|\cX|k^2(2+|\mc{X}|)/\delta)} 
\end{align*}

Thus with probability greater than $1-\delta/(k^2|\mc{X}|)$ conditioned on $\mc{R}_{k-1}$
\begin{align*}
  |x^{\top}(\theta^{\ast} - \hat\theta_k)| 
  &\leq \cone\sqrt{\|x\|^2_{H_k(\theta^{\ast})^{-1}}\log(2|\cX|k^2(2+|\mc{X}|)/\delta)}\tag{By Theorem~\ref{thm:concentration-supp}}\\
  &\leq \cone\sqrt{\frac{(1+\epsilon)\|x\|^2_{H(\lambda_k, \theta^{\ast})^{-1}}\log(2|\cX|k^2(2+|\mc{X}|)/\delta)}{n_k}}\tag{By Lemma~\ref{lem:rounding}}\\
  &\leq \cone\sqrt{\frac{3(1+\epsilon)\|x\|^2_{H(\lambda_k, \hat\theta_{k-1})^{-1}}\log(2|\cX|k^2(2+|\mc{X}|)/\delta)}{n_k}}\tag{By the induction hypothesis $\mc{R}_{k-1}$}\\
  &\leq \cone\sqrt{\frac{3(1+\epsilon)\log(2|\cX|k^2(2+|\mc{X}|)/\delta)}{3(1+\epsilon)\gamma(d)\log(2|\cX|k^2(2+|\mc{X}|)/\delta)}}\\
  &\leq 1
\end{align*}
Then, union bounding over $\mc{X}$ gives that conditioned on $\mc{R}_{k-1}$,  we have the event $\cup_{x\in \mc{X}} \{|x^{\top}(\hat\theta_{k-1} - \theta^{\ast})| \leq 1\}$ is true with probability greater than $1-\delta/k^2$. In particular, now applying Lemma~\ref{lem:self-concordance} proves the claim.\\



\end{proof}

\begin{lemma}[Concentration]\label{lem:concentration}
In round $k$, if we take $n_k$ samples as specified in the algorithm, then 
        \[|(z_\ast - z)^{\top}( \hat{\theta}_k-\theta^\ast)| \leq 2^{-k}\]
        for all $z\in \mc{Z}_k$ with probability greater than $1-\frac{\delta}{k^2}$ given $\{\mc{R}_{s}, \mc{E}_{2,s}\}_{s=1}^{k-1}\cap \cR_0$, or in other words $\P(\mc{E}_{2,k}|\{\mc{R}_{s}, \mc{E}_{2,s}\}_{s=1}^{k-1}\cap \cR_0)\geq 1-\frac{\delta}{k^2}$
\end{lemma}


\begin{proof}


In the previous lemma we showed that conditioned on $\{\mc{R}_{s}, \mc{E}_{2,s}\}_{s=1}^{k-1}$, $\max_{x\in \mc{X}} \|x\|_{H_k(\theta^{\ast})^{-1}}^2\leq 1/\gamma(d)$ .
Given $\mc{Z}_k$ (a random set), we can apply Theorem~\ref{thm:concentration-supp}, to calculate for any $z\in \mc{Z}_k$ 
 \begin{align*}
         (z^{\ast}-z)^{\top}( \hat{\theta}_k-\theta^\ast) 
         &\leq \cone\sqrt{\|z^{\ast}-z\|^2_{H_k(\theta^{\ast})^{-1}}\log(2|\mc{Z}|k^2(2+|\mc{X}|)/\delta)} \tag{Theorem~\ref{thm:concentration-supp}}\\
         &\leq \cone\sqrt{\frac{(1+\epsilon)\|z^{\ast}-z\|^2_{H(\lambda_k, \theta^{\ast})^{-1}}\log(2|\mc{Z}|k^2(2+|\mc{X}|)/\delta)}{n_k}}\tag{Lemma~\ref{lem:rounding}}\\
         &\leq \cone\sqrt{\frac{3(1+\epsilon)\|z^{\ast}-z\|^2_{H(\lambda_k, \hat\theta_{k-1})^{-1}}\log(2|\mc{Z}|k^2(2+|\mc{X}|)/\delta)}{n_k}}\tag{Lemma~\ref{lem:closeness}}\\
         &\leq \cone\sqrt{\frac{3(1+\epsilon)\log(12 k^2|\mc{Z}|^2/\delta)}{2^{2k}\cdot\cone^2\cdot 3(1+\epsilon)\log(2|\mc{Z}|k^2(2+|\mc{X}|)/\delta)}}\\
         &\leq 2^{-k}
 \end{align*} 
 
 Now 
 \begin{align*}
    \P(\cE_{2, k}|\{\mc{R}_{s}, \mc{E}_{2,s}\}_{s=1}^{k-1}\cap \cR_0)
    &\leq \sum_{\cV\subset \mc{Z}} \P(\cE_{2, k}, \mc{Z}_{k}=\cV| \{\mc{R}_{s}, \mc{E}_{2,s}\}_{s=1}^{k-1}\cap \cR_0)\\
    &\leq \sum_{\cV\subset \mc{Z}}\P(\cE_{2, t}|\mc{Z}_{k}=\cV,\{\mc{R}_{s}, \mc{E}_{2,s}\}_{s=1}^{k-1}\cap \cR_0)\P(\mc{Z}_{k}=\cV|\{\mc{R}_{s}, \mc{E}_{s}\}_{s=1}^{k-1}\cap \cR_0)\\
    &\leq \frac{\delta}{k^2}\sum_{\cV\subset\mc{Z}} \P(\mc{Z}_{k}=\cV|\{\mc{R}_{s}, \mc{E}_{2,s}\}_{s=1}^{k-1}\cap \cR_0)\\
    &\leq \frac{\delta}{k^2}
 \end{align*}

Finally we record a consequence of the previous computation for later use, namely,
\begin{align}\label{eq:useful-fact}
         (z^{\ast}-z)^{\top}( \hat{\theta}_k-\theta^\ast) 
        \leq \cone\|z^{\ast}-z\|_{H_t( \hat\theta_{k-1})^{-1}}\sqrt{3\log(2|\mc{Z}|k^2(2+|\mc{X}|)/\delta)}
        \leq 2^{-k}
\end{align}



 \end{proof}




\begin{lemma}[Correctness.]\label{lem:correctness}
    On $\cE_1(x)$ and $\cE_2$, we have that $z^\ast\in \mc{Z}_k$, and $\max_{z\in \mc{Z}_{k+1}} \langle z^\ast - z, \theta^\ast\rangle \leq 2\cdot2^{-k}$ for all $t$. 
    Furthermore, we have $\P(\mc{E}_1,\mc{E}_2)\geq 1-\delta$.
\end{lemma}

\begin{proof} 

Firstly note for any set of events $\{A_k\}_{k=1}^{\infty}$,
\begin{align*}
  \PP(\cup_{k=1}^\infty A_k) = \PP( \cup_{k=1}^\infty \del{A_k \setminus (\cup_{j<k} A_j)})  \le  \sum_{k=1}^\infty \PP(A_k \setminus (\cup_{j<k} A_j)) \le \sum_{k=1}^\infty \PP(A_k \mid (\cap_{j<k} \overline{A_j}))~.
\end{align*}
Then, with $A_k = \wbar{\cR_k} \cup \wbar{\cE_{2,k}}$, we have, using Lemma~\ref{lem:burnin}, Lemma~\ref{lem:concentration}, and Lemma~\ref{lem:correctness},
\begin{align*}
\P(\wbar{\mc{E}_2}\cup \wbar{\mc{E}_1})
  &\leq \PP(\cup_{k=1}^\infty (\wbar{\cR_k} \cup \wbar{\cE_{2,k}})\cup \overline{\mc{R}_0}) \\
  &\le \sum_{k=1}^\infty \PP(\wbar{\cR_k} \cup \wbar{\cE_{2,k}} \mid \cR_{k-1}, \cE_{2,k-1}, \ldots, \cR_1, \cE_{2,1})+\P(\wbar{\mc{R}_0})
  \\&\le \sum_{k=1}^\infty \PP(\wbar{\cE_{2,k}} \mid \cR_{k-1}, \cE_{2,k-1}, \ldots, \cR_1, \cE_{2,1}) + \sum_{k=1}^{\infty} \PP(\wbar{\cR_{k}} \mid \cR_{k-1}, \cE_{2,k-1}, \ldots, \cR_1, \cE_{2,1}) +\delta \\
  &\leq \sum_{k=1}^{\infty} \frac{2\delta}{k^2} +\delta\\
  &\leq 3\delta~.
\end{align*}

For the following we will assume that event $\mc{E}_1\cap \mc{E}_2$ holds. Now we argue that $z^{\ast}$ will never be eliminated. Indeed for any $z\in \mc{Z}_k$, note that

\begin{align*}
        \langle z - z^{\ast}, \hat{\theta}_k\rangle 
        &=\langle z - z^{\ast}, \hat{\theta}_k-\theta^\ast\rangle + \langle z - z^{\ast}, \theta^\ast\rangle \\
        &\leq \cone\|z^{\ast}-z\|_{H_k( \hat\theta_{k-1})^{-1}}\sqrt{3\log(2|\mc{Z}|k^2(2+|\mc{X}|)/\delta)} + \langle z - z^{\ast}, \theta^\ast\rangle\\
        &\leq \cone\|z^{\ast}-z\|_{H_k( \hat\theta_{k-1})^{-1}}\sqrt{3\log(2|\mc{Z}|k^2(2+|\mc{X}|)/\delta)} ~,
\end{align*}
 implying that $z^{\ast}$ is not kicked out. 
Finally, if $\langle z^{\ast} - z, \theta^\ast\rangle\geq 2\times 2^{-k}$, then
 
  \begin{align*}
         \langle z^{\ast} - z, \hat\theta_k \rangle
         &= \langle z^{\ast} - z, \hat\theta_k-\theta^\ast + \theta^\ast\rangle\\
         &= \langle z^{\ast} - z, \theta^\ast\rangle - \|z^{\ast} - z\|_{H(\lambda_k, \hat\theta_{k-1})^{-1}}\sqrt{3\log(2|\mc{Z}|k^2(2+|\mc{X}|)/\delta)}\tag{From\eqref{eq:useful-fact}}\\
         &\geq 2\times 2^{-k}-2^{-k}\\
         &\geq 2^{-k}\\
         &\geq \cone\|z^{\ast}-z\|_{H_k( \hat\theta_{k-1})^{-1}}\sqrt{3\log(2|\mc{Z}|k^2(2+|\mc{X}|)/\delta)}
 \end{align*}
 which is precisely the condition for $z$ to be removed.
 Finally, we have $\la z^* - z,\th^*\ra = \la z^* - z, \th^* - \hat\th_k\ra + \la z^*-z,\hat\th_k\ra \le 2^{-k} + 2^{-k}$, which concludes the proof.
 \end{proof}

\begin{theorem}[Sample Complexity]
Define  $\mc{S}_k = \{z\in \mc{Z}: (z^{\ast}-z)^{\top}\theta_{\ast} \leq 2\cdot 2^{-(k-1)}\}$,  and take $\epsilon \leq 1/2$.  
up constant factors, Algorithm~\ref{alg:rage-glm} returns $z^{\ast}$ with probability greater than $1-3\delta$ in a number of samples no more than 
\begin{align*}
    O\bigg( &(1+\epsilon) \sum_{k=1}^{\lceil \log_2(2/\Delta_{\min})\rceil} \min_{\lambda\in \Delta_{\mc{X}}} \max\left[2^{2k}\max_{z,z'\in \mc{S}_k} \|z-z'\|_{H(\lambda,\theta^{\ast})^{-1}}^2, \gamma(d)\max_{x}\|x\|_{H(\lam,\theta^{\ast})^{-1}}^2\right]\log(\max(|\mc{X}|, |\mc{Z}|^2)k^2/\delta)
\\ &\qquad +  d(1+\epsilon)\kappa_0^{-1}\log(|\mc{X}|/\delta)+ r(\epsilon)\log_2(\frac{1}{\Delta_{\min}})\bigg)~.
\end{align*}
 
\end{theorem}
\begin{proof}
For the remainder of the proof we will assume that $\mc{E}_1\cap \mc{E}_2$ holds. 

By Lemma~\ref{lem:correctness} on $\mc{E}_2$, we have that $\mc{Z}_{k}\subseteq\mc{S}_{k}$, in particular this implies that when $2\times 2^{-k} \leq \Delta_{\min}$, we have $|\mc{Z}_k| = 1$, so this implies that the algorithm will terminate in a number of rounds not exceeding $\lceil \log_2(2/\Delta_{\min})\rceil$

By Lemma~\ref{lem:closeness} on $\mc{E}_1$, we have that $H(\lambda_k, \hat\theta_k)\geq \frac{1}{4} H(\lambda_k, \theta^\ast)$. Thus, in each round,

\begin{align*}
\min_{\lambda\in \Delta_{\mc{X}}}\max\bigg[ 
&2^{2k} \cone^2 
\max_{z,z'\in \mc{Z}_t} \|z-z'\|_{H(\lambda, \hat{\theta}_{k-1})^{-1}}^2,
\gamma(d)\max_{x\in \cX} \|x\|^2_{H(\lambda, \hat{\theta}_{k-1})^{-1}}
\bigg]\\
&\leq O\left(\min_{\lambda\in \Delta_{\mc{X}}}\max\bigg[ 
2^{2k} 
\max_{z,z'\in\cS_k} \|z-z'\|_{H(\lambda, \theta^{\ast})^{-1}}^2,
\gamma(d)\max_{x\in \cX} \|x\|^2_{H(\lambda, \theta^{\ast})^{-1}}
\bigg]\right)
\end{align*}



Let $c$ be an absolute constant.
Thus up to doubly logarithmic factors our final sample complexity is given by 

\begin{align*}
  &n_0 + \sum_{k=1}^{\lceil \log_2(2/\Delta_{\min})\rceil} n_k \\
  &\leq 
  \frac{d(1+\epsilon)\gamma(d)\log(|\mc{X}|/\delta)}{\kappa_0} 
  \\&\hspace{1em}+ c_1(1+\epsilon) \sum_{k=1}^{\lceil \log_2(2/\Delta_{\min})\rceil} \min_{\lambda\in \Delta_{\mc{X}}} \max\left[2^{2k}  \max_{z,z'\in \mc{S}_k} \|z-z'\|_{H(\lam,\theta^{\ast})^{-1}}^2,
  \gamma(d)\max_{x}\|x\|_{H(\lam,\theta^{\ast})^{-1}}^2\right] \log(\max(|\mc{X}|, |\mc{Z}|)\fr{k^2}{\delta}) 
  \\&\hspace{1em}+ c_2\log_2(\Delta_{\min}^{-1})r(\epsilon)~.
\end{align*}


\end{proof}
\begin{lemma}
Define, $\mc{S}_k = \{z\in \mc{Z}:(z^{\ast}-z)^{\top}\theta_{\ast} \leq 2\cdot 2^{-(k-1)}\}$. 

\begin{align*}
    \sum_{k=1}^{\log_2(2/\delta_{\min)}} 2^{2k}\min_{\lambda\in \Delta_{\mc{X}}} \max_{z,z'\in \mc{S}_k} \|z-z'\|_{H(\lambda, \theta^{\ast})}^2
    &\leq \log\left(\frac{1}{\Delta_{\min}}\right)\min_{\lambda\in \Delta_{\mc{X}}}\max_{z\in \mc{Z}\setminus z^{\ast}} \frac{\|z^{\ast} - z\|_{H(\lambda,\theta^{\ast})^{-1}}^2}{\langle \theta^{\ast}, z^{\ast} - z\rangle^2}\\
    &= \frac{1}{4}\log\left(\frac{2}{\Delta_{\min}}\right) \left(\max_{\lambda\in \Delta_{\mc{X}}}\min_{\theta\in \mc{C}}\|\theta^{\ast} - \theta\|^2_{H(\lambda, \theta^{\ast})}\right)^{-1}
\end{align*}
where $\mc{C} = \{\theta\in \mathbb{R}^d:\exists z\in \mc{Z}\setminus z^{\ast}, \theta^{\top}(z^{\ast} - z)\leq 0\}$
\end{lemma}
\begin{proof}
Note that, 
\begin{align*}
    \log\left(\frac{2}{\Delta_{\min}}\right)\min_{\lambda\in \Delta_{\mc{X}}}\max_{z\in \mc{Z}\setminus z^{\ast}} \frac{\|z^{\ast} - z\|_{H(\lambda,\theta^{\ast})^{-1}}^2}{\langle \theta^{\ast}, z^{\ast} - z\rangle^2}
    &= \log\left(\frac{2}{\Delta_{\min}}\right)\min_{\lambda\in \Delta_{\mc{X}}}\max_{k\leq \log_2(2/\Delta_{\min})}\max_{z\in \mc{S}_k\setminus z^{\ast}} \frac{\|z^{\ast} - z\|_{H(\lambda,\theta^{\ast})^{-1}}^2}{\langle \theta^{\ast}, z^{\ast} - z\rangle^2}\\
    &= \log\left(\frac{2}{\Delta_{\min}}\right)\min_{\lambda\in \Delta_{\mc{X}}}\max_{k\leq \log_2(2/\Delta_{\min})} 2^{-2k+4}\max_{z\in \mc{S}_k\setminus z^{\ast}} \|z^{\ast} - z\|_{H(\lambda,\theta^{\ast})^{-1}}^2\\
    &\overset{a}{\geq} 16\sum_{k=1}^{\log(2/\Delta_{\min})} 2^{2k}\min_{\lambda\in \Delta_{\mc{X}}}\max_{z\in S_k\setminus z^{\ast}} \|z^{\ast} - z\|_{H(\lambda,\theta^{\ast})^{-1}}^2\\
    &\overset{b}{\geq} 4\sum_{k=1}^{\log(2/\Delta_{\min})} 2^{2k}\min_{\lambda\in \Delta_{\mc{X}}}\max_{z,z'\in S_k} \|z' - z\|_{H(\lambda,\theta^{\ast})^{-1}}^2
\end{align*}
where $a$ is replacing a max with an average and $b$ is using $\max_{z,z'\in S_k} \|z-z'\|^2_{H(\lambda, \theta^{\ast})^{-1}}= \max_{z,z'\in S_k} \|z-z^{\ast}\|^2_{H(\lambda,\theta^{\ast})^{-1}} + \|z'-z^{\ast}\|^2_{H(\lambda,\theta^{\ast})^{-1}} - 2\|z'-z^{\ast}\|_{H(\lambda,\theta^{\ast})^{-1}}\|z-z^{\ast}\|_{H(\lambda,\theta^{\ast})^{-1}}\leq 4\max_{z\in S_k} \|z-z^{\ast}\|^2_{H(\lambda,\theta^{\ast})^{-1}}$.

We now tackle the second equality in the theorem statement. Define $\mc{C}_z = \{\theta\in \mathbb{R}^d:\theta^{\top}(z^{\ast} - z)\leq 0\}$. Note that,
\begin{align*}
    \max_{\lambda\in \Delta_{\mc{X}}}\min_{\theta\in \mc{C}}\|\theta^{\ast} - \theta\|^2_{H(\lambda, \theta^{\ast})}
    &=
    \max_{\lambda\in \Delta_{\mc{X}}}\min_{z\in \mc{Z}\setminus z^{\ast}}\min_{\theta\in \mc{C}_z}\|\theta^{\ast} - \theta\|^2_{H(\lambda, \theta^{\ast})}
\end{align*}
For a fixed $\lambda$,  standard computation with Lagrange multipliers (as in Theorem~\ref{thm:lowerbound}) shows that the projection,
\[\theta_z := \arg\min_{\theta\in \mc{C}_z} \|\theta^{\ast} - \theta\|^2_{H(\lambda, \theta^{\ast})} = \theta^{\ast} -\frac{(z^{\ast}-z)^{\top}\theta^{\ast} H(\lambda,\theta^{\ast})^{-1}(z^{\ast} - z)}{\|z^{\ast}-z\|^2_{H(\lambda, \theta^{\ast})^{-1}}}
\]

Thus, 
\begin{align*}\|\theta^{\ast} - \theta^{\ast}\|_{H(\lambda, \theta^{\ast})}^2 = \frac{(z^{\ast}-z)^{\top}\theta^{\ast}}{\|z^{\ast}-z\|^2_{H(\lambda, \theta^{\ast})^{-1}}}
\end{align*}
and the result follows.
\end{proof}

\section{RAGE-GLM-2}
\label{sec:faury}

\def\MLE{\mathsf{MLE}}

\subsection{Review of confidence bounds of \cite{faury2020improved}}
Assume that we have observed a sequence of samples $(x_s, y_s)_{s=1}^T$, where, $\{x_s\}_{s=1}^T\in \mathcal{X}$ and the $x_s$'s are potentially chosen \textit{adaptively}, that is $x_s, 1\leq s\leq T$ is allowed to depend on the filtration $\mc{F}_{s-1} = \{(x_r, y_r)\}_{r=1}^{s-1}$.

For a regularization parameter $\eta > 0$, define
\[H_T(\eta, \theta) := \sum_{s=1}^{T} \dot\mu(x_s^{\top}\theta)x_sx_s^{\top} + \eta I\]

We begin by defining our estimator. Let 
\begin{align}\label{eq:mle}
\hat{\theta}_{\eta, T}^{\MLE} \!=\!  \argmax_{\theta\in \mathbb{R}^d} \sum_{s=1}^{T}  y_{s}\log\mu(x_s^{\top}\theta)
     \!+\!(1\!-\!y_s)\log(1\!-\!\mu(x_s^{\top}\theta)) -\frac{\lambda}{2}\|\theta\|_2^2.
\end{align}

Define, 
\begin{align}\label{eq:mlefaury}
    \hat{\theta}_T = \argmin_{\|\theta\|_2\leq S_*} \|g_t(\theta) - g_t(\hat{\theta}^{\MLE}_{\eta, T})\|_{H_T(\eta, \theta)^{-1}}
\end{align}
where $g_T(\theta) = \sum_{s=1}^{T} \mu(x_s^{\top}\theta)x_s + \eta \theta$. Finally, define 
\[\gamma_T(\delta) = \sqrt{\eta}(S_*+1/2) + \frac{2}{\sqrt{\eta}}\log(1/\delta) + \frac{2d}{\sqrt{\eta}}\log(2(1 + \frac{T}{d\eta})^{1/2})\]
We recall the following lemma from \cite{faury2020improved}.

\begin{lemma}[Lemma 11 of \cite{faury2020improved}]\label{lem:Faury}
        On an event $\mc{E}$ which is true with probability greater than $1-\delta$, for all $t\geq 1$
        \[\theta^{\ast} \in \{\theta\in \mathbb{R}^d:\|\theta\|\leq S_*, \|\theta-\hat\theta_T\|_{H_T(\eta, \theta)}\leq (2+4S_*)\gamma_T(\delta)\} \]
\end{lemma}

In the following we will take $\eta = (d+\log(1/\delta))/(S_*+1/2)$. Plugging this in to $\gamma_T(\delta)$
\begin{align*}
    \sqrt{\eta}(S_*+1/2) &+ \frac{2}{\sqrt{\eta}}\log(1/\delta) + \frac{2d}{\sqrt{\eta}}\log\left(2(1 + \frac{T}{d\eta}\right)^{1/2})\\
    &= \sqrt{\frac{d+\log(1/\delta)}{S_*+1/2}}(S_*+1/2) + \frac{2}{\sqrt{\frac{d+\log(1/\delta)}{S_*+1/2}}}\log(1/\delta) + \frac{2d}{\sqrt{\frac{d+\log(1/\delta)}{S_*+1/2}}}\log\left(2\left(1 + \frac{T}{d\frac{d+\log(1/\delta)}{S_*+1/2}}\right)^{1/2}\right)\\
    &\leq \sqrt{d+\log(1/\delta)}\sqrt{S_*+1/2} + 2\sqrt{S_*+1/2}\sqrt{\log(1/\delta)} + \frac{2d\sqrt{S_*+1/2}}{\sqrt{d}}\log(2(1 + \frac{T(2S_*+1)}{2d})^{1/2})\\
    &=\sqrt{S_*+1/2}\left(\sqrt{d+\log(1/\delta)} + 2\sqrt{\log(1/\delta)}+2\sqrt{d}\log\left(2\left(1 + \frac{T(2S_*+1)}{2d}\right)^{1/2}\right)\right)\\
    &\leq\sqrt{S_*+1/2}\left(\sqrt{d}\left(1+2\log(2) + \frac{1}{2}\log\left(1 + \frac{T(2S_*+1)}{2d}\right)\right) + 3\sqrt{\log{1/\delta}}\right)\\
    &\leq 3\sqrt{S_*+1/2}\left(\sqrt{d}\log\left(\frac{T(2S_*+1)}{2d}\right) + \sqrt{\log{1/\delta}}\right)
\end{align*}
where the last line uses, $(1+2\log(2) + 1/2\log(1+x))\leq 3\log(x), x\geq 2$. So as long as $T\geq 4d$, we have the following bound. 

\[\gamma_T(\delta) \leq 3\sqrt{2S_*+1}\bigg[\sqrt{d}\log\bigg(\frac{T(2S_*+1)}{2d}\bigg)+\sqrt{\log(1/\delta)}\bigg]=: \Gamma_T(\delta)~.\]
The guarantee that $T\geq 4d$ will be satisfied by the rounding procedures in the algorithm - indeed, taking $\epsilon\leq 1/2$ guarantees that the minimum number of samples we take in each round $r(\epsilon) = (d(d+1)+2)/\epsilon \geq 4d$.

\subsection{Proof of Sample Complexity}

\begin{algorithm}[t]
\caption{RAGE-GLM-2 }
        \label{alg:rage-glm-2}
	    \begin{algorithmic}[1]
            \Input {$\epsilon$, $\delta$, $\mc{X}$, $\mc{Z}$, $\kappa_0$, $S_*$, effective rounding procedure $\rdp(n, \epsilon, \lambda)$, $\eta = (d+\log(1/\delta))/(S_*+1/2)$}
            \Initialize{$t=1, \mc{Z}_1 = \mc{Z}, r(\epsilon) = d^2/\epsilon$, $c = c(S_*,\epsilon) = 48\sqrt{(1+\epsilon)(2S_*+1)^3}$}
            \State{$\theta_0\leftarrow$ \textbf{BurnIn}($\mc{X}, \kappa_0$)}  \Comment{Burn-in phase}
            \While{$|\mc{Z}_t| > 1$ } \Comment{Elimination phase}
            \State{$f(\lambda) := \min_{z,z'\in \mc{Z}_t} \|z-z'\|_{H(\lambda, \theta_0)^{-1}}^2$}
            \State $\lambda_t = \argmin_{\lambda\in \Delta_{\mc{X}}} f(\lambda)$
            \State $r_t = \lt\lceil  2^{2t}c^2 f(\lambda_t)\del{\sqrt{d}\log(c^2 2^{2t}(2S_*+1)f(\lambda_t)/d)+ \sqrt{\log(t^2|\mc{Z}|^2/\delta)}}^2\rt\rceil $
            \State{$n_t = \max\{r_t, r(\epsilon)\}$}
            \State{$x_1, \cdots, x_{n_t}\leftarrow \rdp(n,\epsilon,\lambda)$} 
            \State{Observe rewards $y_1, \cdots, y_{n_t}\in \{0,1\}$}
            \State{Compute $\hat\theta_t$ on the samples $\{(x_s, y_s)_{s=1}^{n_t}\}$}  \Comment{(Use Eq~\eqref{eq:mlefaury})}
            \State{$\hat{z}_t = \argmax_{z\in\mc{Z}_t} \hat\theta_t^\top z$}
            \State{$\mc{Z}_{t+1}\leftarrow \mc{Z}_t \setminus \left\{z\in \mc{Z}_t:\hat\theta_t^\top (\hat z_t-z) \geq 2^{-t}\right\}$}
            \State{$t\gets t+1$}
			\EndWhile
			\State \Return $\hat{z}_t$
		\end{algorithmic}
	\end{algorithm}

We now provide a sample complexity for Algorithm~\ref{alg:rage-glm-2}. In this section, we take $\theta_t$ as defined in \ref{eq:mlefaury} using the samples $\{(x_s,y_s)\}_{s=1}^{n_t}$ in each round $t$. 

In the regularized setting, rounding implies that, 
\begin{align}\label{eqn:reg-rounding}
        H_t(\eta, \theta^{\ast}) 
        &:= H_t(\theta^{\ast}) + \eta I\\
        &\geq \frac{n}{1+\epsilon} \sum_{x\in \mc{X}}\lambda_x\dot\mu(x^{\top}\theta^{\ast}) xx^{\top} + \eta I \\
        &\geq \frac{n}{1+\epsilon} H(\lambda, \theta^{\ast}) + \eta I\\
        &\geq \frac{n}{1+\epsilon} H(\lambda, \theta^{\ast}) 
\end{align}



Define
\[\mc{E}_1 := \cbr{\frac{1}{3}H(\lambda_0, \theta^{\ast})\leq H(\lambda_0, \hat\theta_0)\leq {3}H(\lambda_0, \theta^{\ast})}\]
By Lemma~\ref{lem:burnin}, $\P(\mc{E}_1)\geq 1-\delta$.

Define 
\begin{align*}
\mc{E}_2 = \cap_{t=1}^{\infty } \{\forall z \in \mc{Z}_t, |\langle z^{\ast} - z, \hat{\theta}_t - \theta^{\ast}\rangle| \leq 2^{-t}\}
\end{align*}

\begin{lemma}\label{correctness2}
$\P(\mc{E}_2\cap \mc{E}_1) \geq 1-3\delta$ and on $\mc{E}_1\cap\mc{E}_2$, $z^{\ast}\in \mc{Z}_t$ for all $t$.
\end{lemma}
\begin{proof}
\noindent\textbf{Claim 1: $\P(\mc{E}_2|\mc{E}_1) \geq 1-\delta$.} Assuming $\mc{E}_1$, For $z\in \mathcal{Z}_t$, with probability greater than $1-\frac{\delta}{t^2|\mc{Z}|}$
\begin{align*}
        |(z^{\ast}-z)^{\top}(\hat{\theta}_t - \theta^\ast)|
        &\leq \|z^{\ast}-z\|_{H_t(\eta, \theta^{\ast})^{-1}}\|\theta^{\ast} - \hat{\theta}_t\|_{H_t(\eta, \theta^{\ast})}\\
        &\leq (2+4S_*)\|z^{\ast}-z\|_{H_t(\eta, \theta^{\ast})^{-1}}\Gamma_{n_t}(\delta)\tag{Lemma~\ref{lem:Faury} }\\
        &\leq 2(1+2S_*)\sqrtf{1+\epsilon}{n}\|z^{\ast}-z\|_{H(\lambda_t, \theta^{\ast})^{-1}}\Gamma_{n_t}(\delta)\tag{Rounding Lemma~\ref{lem:rounding}}\\
        &\leq 8(1+2S_*)\sqrtf{1+\epsilon}{n}\|z^{\ast}-z\|_{H(\lambda_t, \theta_0)^{-1}}\Gamma_{n_t}(\delta)\tag{$\mc{E}_1$}\\
        &\leq 8(1+2S_*)\sqrtf{(1+\epsilon)f(\lambda_t)}{n}\Gamma_{n_t}(\delta)
\end{align*}
We wish for this quantity to be bounded above by $2^{-t}$. Plugging in $\Gamma_{n_t}(\delta)$, it suffices to take
\[\sqrt{n_t} \geq 24\cdot 2^{t}(2S_*+1)^{3/2}\sqrt{(1+\epsilon)f_t}\bigg[\sqrt{d}\log\bigg(\frac{n_t(2S_*+1)}{2d}\bigg)+ \sqrt{\log(t^2|\mc{Z}|/\delta)}\bigg]\]
where for ease of notation we have denoted $f_t = f(\lambda_t)$. 
Using Lemma~\ref{lem:invert} below, shows that it suffices to take, 
\[n_t = \lt\lceil  c^2 2^{2t}f_t\del{\sqrt{d}\log(c^2 2^{2t}(2S_*+1) \hat\rho_t/d)+ \sqrt{\log(t^2|\mc{Z}|/\delta)}}^2\rt\rceil \]

where $c = 2\cdot24\sqrt{1+\epsilon}(2S_*+1)^{3/2}$, which is precisely the number of samples we take in the algorithm. Union bounding over $z\in\mc{Z}_t\subset \mc{Z}$ and $t\geq 1$ now gives the result.

\textbf{Claim 2: $\P(\mc{E}_1\cap \mc{E}_2)\geq 1-3\delta$.} Note that, 
\begin{align*}
\P(\mc{E}_1^c\cup \mc{E}_2^c) 
&\leq \P(\mc{E}_2^c) + \P(\mc{E}_1^c)\\
&= \P(\mc{E}_2^c|\mc{E}_1^c)\P(\mc{E}_1^c) + \P(\mc{E}_2^c|\mc{E}_1)\P(\mc{E}_1)+\P(\mc{E}_1^c)\\
&\leq  \P(\mc{E}_2^c|\mc{E}_1)+2\P(\mc{E}_1^c)\\
&\leq \delta + 2\delta\\
&\leq 3\delta
\end{align*}

\textbf{Claim 3: On $\mc{E}_1\cap\mc{E}_2$, $z^{\ast}\in \mc{Z}_t$ for all $t\geq 1$.} Identical argument to Lemma~\ref{lem:correctness}

\end{proof}
\textbf{Remark.} We point out that this analysis is not particularly tight, and many constants and the dependence upon $S_*$ can be improved upon in practice. In particular, we can trade off a smaller constant for a larger burn-in phase.








\begin{theorem}[Sample Complexity]\label{thm:rage-glm-2}
Algorithm~\ref{alg:rage-glm-2}, returns $z^{\ast}$ with probability greater than $1-2\delta$ in a number of samples no more than 
\begin{align*}
  O\Bigg( 
    & (1+\epsilon)(2S_*+1)^{3}\sum_{r=1}^{\log_2(1/\Delta_{\min})} 2^{2t}\rho_t\del{d\log^2\left(\fr{(2S_*+1)\rho_t}{\Delta_{\min}}\right) + \log(t^2|\mc{Z}|^2/\delta)} 
  \\&\hspace{14em}+r(\epsilon)\log_2(1/\Delta_{\min})+ \kappa_0^{-1}(1+\epsilon)d\gamma(d)\log(|\mc{X}|/\delta)\Bigg)
\end{align*}
 where $\mc{S}_t = \{z\in \mc{Z}: (z^{\ast}-z)^{\top}\theta_{\ast} \leq 2\cdot 2^{-t}\}$ and $\rho_t = \min_{\lambda\in \Delta_{\mc{X}}}\max_{z,z'\in \mc{S}_t} \|z-z'\|_{H(\lambda, \theta^{\ast})^{-1}}^2$ and we assume $\epsilon \leq 1/2$.
\end{theorem}
\begin{proof}

Firstly note that $\P(\mc{E}_1^c\cup \mc{E}_2^c) \leq 2\delta$. For the remainder of the proof we will assume that $\mc{E}_1\cap \mc{E}_2$ holds. 

By Lemma~\ref{correctness2}, we have that $\mc{Z}_t\subset\mc{S}_t$, likewise on $\mc{E}_1$ we have that $H(\lambda_t, {\theta}_0)\geq \frac{1}{4} H(\lambda_t, \theta^\ast)$. Thus, in each round, 
\[\max_{z,z'\in\mc{Z}_t} \|z-z'\|_{H(\lambda_t, {\theta}_0)^{-1}}^2 \leq 
4\max_{z,z'\in \mc{S}_t} \|z-z'\|_{H(\lambda, \theta^{\ast})^{-1}}^2 \]

Denoting $\rho_t = \min_{\lambda\in \Delta_{\mc{X}}}\max_{z,z'\in \mc{S}_t} \|z-z'\|_{H(\lambda, \theta^{\ast})^{-1}}^2$, we see that $f_t \leq \rho_t$. 
This implies that $n_t\leq 4c^2 2^{2t}\rho_t[\sqrt{d}\log(2c^2 2^{2t}(2S_*+1) \rho_t/d)+ \sqrt{\log(t^2|\mc{Z}|^2/\delta)}]^2$.

Thus an upper bound on our final sample complexity is given by 
\begin{align*}
        &\sum_{t=1}^{\log_2(1/\Delta_{\min})} n_t +r(\epsilon)\log(1/\Delta_{\min}) + n_0\\
         &\leq \sum_{t=1}^{\log_2(1/\Delta_{\min})} 4c^2 2^{2t}\rho_t[\sqrt{d}\log(c^2 2^{2t}(2S_*+1) \rho_t/d)+ \sqrt{\log(t^2|\mc{Z}|^2/\delta)}]^2 + r(\epsilon)\log(1/\Delta_{\min})+n_0\\
         &\leq 8c^2\sum_{t=1}^{\log_2(1/\Delta_{\min})} 2^{2t}\rho_t[d\log^2(c^2 2^{2t}(2S_*+1) \rho_t/d)+ \log(t^2|\mc{Z}|^2/\delta)] + r(\epsilon)\log(1/\Delta_{\min})+n_0\\
        &\leq 8c^2\sum_{t=1}^{\log_2(1/\Delta_{\min})} 2^{2t}\rho_t[d\log^2(c^2 \frac{4}{\Delta^2_{\min}}(2S_*+1) \rho_t/d)+ \log(t^2|\mc{Z}|^2/\delta)] + r(\epsilon)\log(1/\Delta_{\min})+n_0\\
        & = O\bigg((1+\epsilon)(2S_*+1)^{3}\sum_{r=1}^{\log_2(1/\Delta_{\min})} 2^{2t}\rho_t[d\log^2((2S_*+1)\rho_t/\Delta_{\min})
          \\& \hspace{4em}+\log(t^2|\mc{Z}|^2/\Delta_{\min})] +r(\epsilon)\log(1/\delta) +  \kappa_0^{-1}(1+\epsilon)d\gamma(d)\log(|\mc{X}|/\delta)\bigg) 
\end{align*}



\end{proof}

\subsection{Miscellaneous results}

We let $\rdp(\lambda, n)$ denote an efficient rounding procedure as explained in Chapter 12 of \cite{pukelsheim}, or summarized in Section B of the Appendix of \cite{fiez2019sequential}.

\begin{lemma}[Rounding ]\label{lem:rounding}
    Assume that $\lambda\in \Delta_{\mathcal{X}}$, and that we have sampled $x_1, \cdots, x_n\sim {\normalfont \rdp}(\lambda, n, \epsilon)$ with $n\geq r(\epsilon) = (d(d+1)+2)/\epsilon$, and $\epsilon\leq 1$. Then, for any $\theta$, $\sum_{s=1}^n \dot\mu(x_s^{\top} \theta) x_sx_s^{\top} \succeq \frac{n}{1+\epsilon} \sum_{x\in \mc{X}} \lambda_x \dot\mu(x^{\top}\theta)xx^{\top}$. This in particular implies 
    \begin{itemize}
    \item For any $z$, 
    \[\|z\|^2_{(\sum_{s=1}^n \dot\mu(x_s^{\top} \theta) x_sx_s^{\top})^{-1}} \leq \frac{(1+\epsilon)}{n}\|z\|^2_{(\sum_{x\in \mc{X}} \lam_x \dot\mu(x^{\top} \theta) xx^{\top})^{-1}}\]
    \item $\lambda_{\min}(\sum_{s=1}^n \dot\mu(x_s^{\top} \theta) x_s x_s^{\top})\geq \frac{n}{1+\epsilon} \lambda_{\min}(\sum_{x\in \mc{X}} \lambda_x \dot\mu(x^{\top}\theta)xx^{\top})$
    \end{itemize}
\end{lemma}
\begin{proof}
Let $s = (n_x)_{x\in \mc{X}}\in \mathbb{N}^{\mc{X}}$ denote the allocation returned by the rounding procedure and let $\gamma = s/n\in \Delta_{\mc{X}}$ denote the associated fractional allocation. Now consider, 
\[\epsilon_{\gamma/\lambda} = \min_{x\in \supp(\lambda)} \frac{\gamma_x}{\lambda_x} = \max \{\kappa\geq 0:\gamma_x\geq \kappa\lambda_x \text{  for all  } x\in \mc{X}\}\]
By definition of $\epsilon_{\gamma/ \lambda}$, 
\[\sum_{x\in\mc{X}} \gamma_x\dot\mu(x^\top\theta)xx^{\top}\geq \epsilon_{\gamma/\lambda}\sum_{x\in\mc{X}} \lambda_x\dot\mu(x^\top\theta_\ast) xx^{\top}\]
By Theorem 12.7 of \cite{pukelsheim}, $\epsilon_{\gamma/\lambda} \geq 1-p/n$ where $p = |\supp \lambda|$. When $\dim\operatorname{span}\mc{X} = d$, Caratheodory's Theorem \cite{vershynin2018high}, implies $p\leq d(d+1)/2 + 1$.  Hence,
\begin{align*}
    \sum_{s=1}^n \dot\mu(x_s^{\top}\theta) x_sx_s^{\top}
    &= n\sum_{x\in\mc{X}} \gamma_x\dot\mu(x^\top\theta)xx^{\top}\\
    &\geq n(1-\frac{p}{n})\sum_{x\in\mc{X}} \lam_x\dot\mu(x^\top\theta)xx^{\top}\\
    &\geq \frac{n}{1+\epsilon} \sum_{x\in\mc{X}} \lam_x \dot\mu(x^\top\theta)xx^{\top}
\end{align*}
as long as $n\geq (d(d+1) + 2)/\epsilon$. The result now follows. 
\end{proof}

As long as $n_t\geq r(\epsilon)$, we have a guarantee that $H_t(\theta)\geq \frac{n_t}{1+\epsilon}H(\lambda_t)$ for any $\theta$. This implies, $H_t(\theta)^{-1}\leq \frac{1+\epsilon}{n_t}H(\lambda_t)^{-1}$. This is a modification of the argument in \citet{fiez2019sequential}.

\begin{lemma}\label{lem:self-concordance}
  Let $\th\in\RR^d$. 
  Suppose $D = \max_{x \in \cX } |x^\T (\th - \th^*)|\le 1$.
  Then, for all $x$, 
  \begin{align*}
    \fr{1}{2D+1}H(\lam,\th^*) \le H(\lam,\th) \le (2D+1) H(\lam,\th^*)
  \end{align*}
\end{lemma}
\begin{proof}
  The proof is identical to Lemma~\ref{lem:eqvar}.
\end{proof}

\begin{lemma}\label{lem:invert}
Assume $a>0, b > 2$, 
then for any $t\geq \max[(2a)^2(\ln((2a)^2/c) + \ln{b}+d)^2,2c]$ we have that $\sqrt{t}\geq a[\log(b+t/c) +d]$
\end{lemma}
\begin{proof}

Note that if $t>2c$, $a\log(b+t/c)\leq a\log(b) + a\log(t/c)$, so it suffices to show $\sqrt{t} \geq a(\log(b) + \log(t/c)) + ad$ or equivalently, $\frac{1}{a}\sqrt{t} -\log(b)-d\geq \log(t/c)$, or doing the substitution $u = t/c$,  $\frac{\sqrt{c}}{a}\sqrt{u} -\log(b)-d\geq \log(u)$
for $t\geq (2a)^2(\ln{(2a)^2/c} + \ln{b}+d)^2$. However, this follows directly from Proposition 6 of \cite{antos2010active}. 
\end{proof}

\section{Lower Bounds}

\subsection{Instance Dependent Lower Bound}
\vspace{-.5em}

The lemma shows that the main term of the sample complexity of Theorem~2 (from the main paper) given as $\rho^*$ in Section 3.1 is bounded by a natural experimental design arising from the true problem parameter. 

Finally, we provide the following information theoretic lower bound for any PAC-$\delta$ algorithm. Define, 
\[\beta(a,b) = \int_0^1 (1-t)\dot{\mu}(a+t(b-a)) dt\]
and analogous to $H(\lambda, \theta)$ we define two additional matrix valued functions,
\begin{gather}
        G(\lambda, \theta_1, \theta_2) = \sum_{x\in \mathcal{X}} \lambda_x \alpha(x,\theta_1, \theta_2)xx^{\top}\\
        K(\lambda, \theta_1, \theta_2) = \sum_{x\in \mathcal{X}} \lambda_x \beta(x,\theta_1, \theta_2)xx^{\top}
\end{gather}
\begin{theorem}\label{thm:lowerbound}
Any PAC-$\delta$ algorithm for the pure exploration logistic bandits problem has a stopping time $\tau$ satisfying, 
\begin{align*}
\E[\tau] 
&\!\geq\! \min_{\lambda\in \Delta_Z}\!\!\!\!\max_{\substack{z\in \mc{Z}\setminus z^{\ast}\\ \theta\in \mathbb{R}^d\\ \theta^{\top}(z^{\ast}-z)\leq 0}} \!\!\!\!\!\frac{1}{\sum_{x\in \mc{X}} \lambda_x \mathsf{KL}(\nu_{x,\theta^{\ast}}|\nu_{x,\theta})}\log\left(\frac{1}{2.4\delta}\right)\\
&\!=\! c(\lambda)^{-1}\log{\frac{1}{2.4\delta}}, c(\lambda)
= \max_{\lambda\in \Delta{\mc{X}}} \min_{\mc{Z}\setminus z^\ast} \|\theta - \theta_z\|^2_{K(\lambda, 
\theta^{\ast}, \theta_z)}
\end{align*}
where firstly, $\nu_{x,\theta} = \Bernoulli(x^{\top}\theta)$, and secondly $\theta_z:= \min_{\theta\in \mathbb{R}^d:\theta^{\top}(z^{\ast}-z)\leq 0} \|\theta - \theta_z\|_{K(\lambda, \theta^{\ast}, \theta_z)}^2$
and is given explicitly as the solution to the fixed-point equation
\begin{align*}
  \theta_z = \theta^{\ast} - \frac{(z^{\ast}-z)^{\top}\theta^{\ast}G(\lambda, \theta_z, 
    \theta^{\ast})^{-1}(z^{\ast}-z)}{\|z^{\ast}-z\|_{G(\lambda, \theta_z, \theta^{\ast})^{-1}}^2}
\end{align*}
\end{theorem}
In general, it is not clear how to compare our upper bound from Theorem~\ref{alg:rage-glm} to 
this lower bound due to the non-explicit nature of $G(\lambda, \theta_z, \theta^{\ast})$. The quantity, $\max_{\lambda\in \Delta_{\mc{X}}}\min_{z\in \mc{Z}\setminus \{z^\ast\}} \|\theta^{\ast} - \theta_z\|^2_{H(\lambda, \theta^{\ast})}$ can be interpreted as a lower bound arising from a quadratic approximation of the KL-divergence in the first line of the lower bound by the Fisher information matrix. 


In this section, we provide an information theoretic lower bound for any PAC-$\delta$ algorithm. Define, 
\[\beta(a,b) = \int_0^1 (1-t)\dot{\mu}(a+t(b-a)) dt\]
and analogous to $H(\lambda, \theta)$ we define two additional matrix valued functions,
\begin{gather}
        G(\lambda, \theta_1, \theta_2) = \sum_{x\in \mathcal{X}} \lambda_x \alpha(x,\theta_1, \theta_2)xx^{\top}\\
        K(\lambda, \theta_1, \theta_2) = \sum_{x\in \mathcal{X}} \lambda_x \beta(x,\theta_1, \theta_2)xx^{\top}
\end{gather}

\begin{theorem}\label{thm:lowerbound}
Any PAC-$\delta$ algorithm for the pure exploration logistic bandits problem has a stopping time $\tau$ satisfying, 
\begin{align*}
\E[\tau] & \geq c(\lambda)^{-1}\log{\frac{1}{2.4\delta}}, c(\lambda)
= \max_{\lambda\in \Delta{\mc{X}}} \min_{z\neq z^{\ast}\in \mathcal{Z}} \|\theta - \theta_z\|^2_{K(\lambda, 
\theta^{\ast}, \theta_z)}
\end{align*}
where $\theta_z:= \min_{\theta\in \mathbb{R}^d:\theta^{\top}(z^{\ast}-z)\leq 0} \|\theta - \theta_z\|_{K(\lambda, \theta^{\ast}, \theta_z)}^2$
and is given explicitly as the solution to the fixed-point equation
\[\theta_z = \theta^{\ast} - \frac{(z^{\ast}-z)^{\top}\theta^{\ast}G(\lambda, \theta_z, 
\theta^{\ast})^{-1}(z^{\ast}-z)}{\|z^{\ast}-z\|_{G(\lambda, \theta_z, \theta^{\ast})^{-1}}^2}\]
\end{theorem}
\begin{proof}
Let $\cC = \{\theta\in \Theta:\exists z\in Z,  \theta^{\top}(z^{\ast} -z)\leq 0\}$. The transportation theorem of \cite{kaufmann2016complexity} implies that any algorithm that is $\delta$-PAC, takes at least $T$ samples with 
\begin{align*}
    \E[T] &\geq \log\left(\frac{1}{2.4\delta}\right)\min_{\lambda\in \Delta_Z}\max_{\theta \in \cC}\frac{1}{\sum_{x\in \mc{X}} \lambda_x \KL(\nu_{x,\theta^{\ast}}|\nu_{x,\theta})}\\
    &\geq \log\left(\frac{1}{2.4\delta}\right)\min_{\lambda\in \Delta_Z}\max_{z\in \mc{Z}\setminus z^{\ast}}\max_{\theta\in \mathbb{R}^d, \theta^{\top}(z^{\ast}-z)\leq 0}\frac{1}{\sum_{x\in \mc{X}} \lambda_x \KL(\nu_{x,\theta^{\ast}}|\nu_{x,\theta})}
\end{align*}
where $\nu_{x,\theta}$ is the distribution of arm $x$ under the parameter vector $\theta$, i.e. $\nu_{x,\theta} = \Bernoulli(x^{\top}\theta)$

For a fixed $z'\in \mc{Z}$, s.t. $z'\neq z^{\ast}$ consider 
\[\min_{\theta\in\mathbb{R}^d, \theta^\top(z^\ast - z')\leq 0} \sum_{x\in \mc{X}} \lambda_x \KL(\nu_{x,\theta^{\ast}}|\nu_{x,\theta}).\]
We have that 
 
\begin{align*}
 \KL(\nu_{x,\theta^{\ast}}|\nu_{x,\theta})
 &=  \mu(x^\top\theta^{\ast})
 \log\left(\frac{\frac{e^{x^{\top}\theta^{\ast}}}{1+e^{x^{\top}\theta^{\ast}}}}{\frac{e^{x^{\top}\theta}}{1+e^{x^{\top}\theta}}}\right) + (1-\mu(x^\top\theta^{\ast}))\log\left(\frac{\frac{1}{1+e^{x^{\top}\theta^{\ast}}}}{\frac{1}{1+e^{x^{\top}\theta}}}\right)\\
 &= \mu(x^\top\theta^{\ast})x^\top(\theta^{\ast} - \theta) + \log\left(\frac{\frac{1}{1+e^{x^{\top}\theta^{\ast}}}}{\frac{1}{1+e^{x^{\top}\theta}}}\right)\\
 &= \mu(x^\top\theta^{\ast})x^\top(\theta^{\ast} - \theta) + \log\left(\frac{1-\mu(z^{\top}\theta^{\ast})}{1-\mu(x^{\top}\theta)}\right)\\
 &= \mu(x^\top\theta^{\ast})x^\top(\theta^{\ast} - \theta) + \log(1-\mu(x^{\top}\theta^{\ast}))-\log(1-\mu(x^{\top}\theta))
\end{align*}

Differentiating with respect to $\theta$ gives, 
\[\nabla_{\theta} \KL(\nu_{x,\theta^{\ast}}|\nu_{x,\theta}) = -\mu(x^{\top}\theta^\ast)x +\frac{\dot\mu(x^\top\theta)x}{1-\mu(x^{\top}\theta)}= (\mu(x^{\top}\theta)-\mu(x^{\top}\theta^\ast))x\]
 using the fact that $\dot\mu(a) = \mu(a)(1-\mu(a))$ so this implies
 \[\nabla_{\theta}\sum_{x\in \mc{X}} \lambda_x \KL(\nu_{x,\theta^{\ast}}|\nu_{x,\theta}) = \sum_{x\in \mc{X}}\lambda_x(\mu(x^{\top}\theta)-\mu(x^{\top}\theta^\ast))x\]
 
 Assuming $\Theta=\mathbb{R}^d$ and letting $\psi$ denote the Lagrange Multiplier corresponding to the constraint $\theta^\top(z^\ast-z)\leq 0$ gives that the minimal $\theta$ satisfies, 
 \[\sum_{x\in \mc{X}}\lambda_x(\mu(x^{\top}\theta)-\mu(x^{\top}\theta^\ast))x = \psi \cd (z^\ast - z) \]
 
 Now by definition, $\mu(x^\top\theta) - \mu(x^\top\theta^\ast) = \alpha(x,\theta, \theta^\ast)x^{\top}(\theta-\theta^\ast)$ so, this reduces to,
 \[\left(\sum_{x\in \mc{X}} \lambda_x \alpha(x,\theta, \theta^\ast)xx^{\top}\right)(\theta-\theta^\ast) = \psi(z^\ast -z) \Rightarrow \theta = \theta^{\ast} +\psi G(\lambda,\theta,\theta^\ast)^{-1}(z^{\ast} - z)\]

Let $\theta_z$ be the solution to this fixed point equation. Since we are saturating the constraint, it should be true that $\theta_z(z^\ast -z) = 0$. With this, we take an inner product with $z^* - z$ on both sides to obtain
\[\psi = -\frac{(z^\ast - z)^{\top}\theta^{\ast}}{\|z^{\ast}-z\|^2_{G(\lambda,\theta, \theta^{\ast})^{-1}}} ~.\]

So finally, we see that 
\[\theta_z = \theta^{\ast} - \frac{{\theta^\ast}^\T(z^{\ast} - z)G(\lambda,\theta_z,\theta^\ast)^{-1}(z^{\ast} - z)}{\|z^{\ast}-z\|^2_{G(\lambda,\theta_z, \theta^{\ast})^{-1}}}\]
 and 
\[\theta_z = \arg\min_{\theta\in\mathbb{R}^d, \theta^\top(z^\ast - z')\leq 0} \sum_{x\in \mc{X}} \lambda_x \KL(\nu_{x,\theta^{\ast}}|\nu_{x,\theta})\]

Now to finish the proof, note
\begin{align*}
    \KL(\nu_{x,\theta^{\ast}}|\nu_{x,\theta})
    &= \mu(x^\top\theta^{\ast})x^\top(\theta^{\ast} - \theta) + \log(1-\mu(x^{\top}\theta^{\ast}))-\log(1-\mu(x^{\top}\theta))\\
    &= (\theta^\ast - \theta) \left[\frac{\mu(x^\top\theta^{\ast})}{(x^\top(\theta^{\ast} - \theta))^2} + \frac{\log(1-\mu(z^{\top}\theta^{\ast}))}{(x^\top(\theta^{\ast} - \theta))^2}-\frac{\log(1-\mu(x^{\top}\theta))}{(x^\top(\theta^{\ast} - \theta))^2} \right]xx^\T (\theta^\ast - \theta)
\\  &\stackrel{(a)}{=} \|\theta^* - \theta \|^2_{\beta(z^\T \theta, z^\T\theta^*) z z^\T}
\end{align*}
where the last expression follows from the computation, 
\begin{align*}
     \beta(a,b) &= \int_0^1 (1-t)\dot{\mu}(a+t(b-a)) dt\\ 
     &=  \frac{\ln\left({\mathrm{e}}^{-a}+1\right)}{{\left(b-a\right)}^2}-\frac{\ln\left({\mathrm{e}}^{-b}+1\right)}{{\left(b-a\right)}^2}-\frac{1}{\left({\mathrm{e}}^{b}+1\right)\,\left(b-a\right)}
  \end{align*}

\end{proof}

In general, it is not clear how to compare our upper bound from Theorem 2 (in the main paper) to this lower bound due to the non-explicit nature of $G(\lambda, \theta_z, \theta^{\ast})$. In the case of Gaussian linear bandits, previous work has shown an elimination scheme similar to Algorithm~\ref{alg:rage-glm} is indeed near optimal. 

\subsection{$1/\kappa_0$ Lower Bounds}

In this section, we prove that there exist bounds where a dependence on $1/\kappa_0$ is necessary. Throughout, we take $\EE_\th$ and $\PP_\th$ to denote expectation and probability under an instance where parameter vector 
$\theta^\ast$ is equal to $\theta$.

\begin{theorem}\label{thm:permuation_lower_bound}
Fix $\delta_1 < 1/16$, $d \geq 4$, and $\epsilon \in (0, 1/2]$ such that $d\varepsilon^2 \geq 12.2$. Let $\mc{Z}$ denote the action set and $\Theta$ denote a family of possible parameter vectors.
There exists instances satisfying the following properties simultaneously
\begin{enumerate}
    \item $|\mc{Z}| = |\Theta| = e^{\epsilon^2d/4}$ and $\|z\| =1 $ for all $z \in \mc{Z} $. 
    \item $S =\|\theta_{\ast}\|= O(\epsilon^2d)$
    \item Any algorithm that succeeds with probability at least $1-\delta_1$ satisfies
    $$\exists \th\in\Theta \text{ such that } \E_\th[T_{\delta_1}] > \Omega\left(e^{\epsilon^2d/4}\right) = c\left(\frac{1}{\kappa_0}\right)^{\frac{1-\epsilon}{1+3\epsilon}} $$
where $T_{\delta_1}$ is the random variable of the number of samples drawn by an algorithm and $c$ is an absolute constant. 
\end{enumerate}
\end{theorem}

To prove this, we first state a more general theorem about lower bounds for logistic bandits. 
\begin{theorem}\label{thm:lower_general}
Fix $\delta_1 < 1/16$, $n \in \N: n> 20$, and a set of arms $\mc{Z} = \{z_1, \cdots, z_n\} \in \R^d$. Consider a family of parameter vectors $\Theta = \{\theta_1, \cdots, \theta_n\} \in \R^d$ such that for every $i\in[n]$,
\begin{enumerate}
    \item $\mu(\theta_i^Tz_i) \in [1-\delta_2, 1]$ 
    \item $\mu(\theta_i^Tz_j) \in [0, \delta_2]$ for all $j \neq i$
\end{enumerate}
If $\delta_2 \leq \fr{1}{2n}$,
then any $\delta_1$-PAC algorithm satisfies
  $$\exists \th\in\Theta \text{ such that } \E_\th [T_{\delta_1}] > \frac{n}{16} $$
where $T_{\delta_1}$ is the random variable of the number of samples drawn by an algorithm.

\end{theorem}


\begin{proof}[\textbf{Proof of Theorem}~\ref{thm:permuation_lower_bound}]
We begin with a construction based on the technique from \cite{dong2019performance} but optimized for our setting. Then we choose $S$ to satisfy the conditions of Theorem~\ref{thm:lower_general} while controlling $\kappa_0$. 

\textbf{Step 1: Constructing of $A$ and $\Theta$}

There exists at least $n=\lfloor \exp(\epsilon^2d/4)\rfloor > 20$ vectors on the sphere in $\R^{d-1}$, $a_1, \cdots, a_n$ such that $|a_i^T a_j| < 1/2$ and $\|a_i\| = 1$ for all $i$. 
Define arms $z_1, \cdots z_n$ such that 
$$z_i = (\cos(u), \sin(u)a_i) \in \R^d $$
where $u  = \tan^{-1}\left(\sqrt{\frac{2}{1+\epsilon}}\right)$.
Similarly, define a family of $\theta$'s such that 
$$\frac{\theta_i}{S} = (-\cos(u), \sin(u)a_i) \in \R^d $$
where $S$ is the norm of all of the $\theta_i$'s to be specified later. Since $\|a_i\|=1$ for all $i$, we have that $\|z_i\| = 1$ and $\|\theta_i\| = S$ for all $i$. 
Then we have that 
$$z_i^T\frac{\theta_i}{S} = -\cos(u)^2 + \sin(u)^2.$$
Plugging in our choice of $u$ and recalling that $\sin(\tan^{-1}(x)) = \frac{x}{\sqrt{1+x^2}}$ and $\cos(\tan^{-1}(x)) = \frac{1}{\sqrt{1+x^2}}$. Therefore,
$$z_i^T\frac{\theta_i}{S} = \frac{-1}{1 + \frac{2}{1+\epsilon}}+ \frac{\frac{2}{1+\epsilon}}{1 + \frac{2}{1+\epsilon}} = \frac{1-\epsilon}{3+\epsilon} > 0$$

Furthermore, we have that 
$$z_j^T\frac{\theta_i}{S} = -1 + \sin(u)^2(1 + a_i^Ta_j).$$
We have that $|a_i^Ta_j| \leq \epsilon$. Therefore, 
$$z_j^T\frac{\theta_i}{S} \leq -1 + \frac{\frac{2}{1+\epsilon}}{1 + \frac{2}{1+\epsilon}}(1+\epsilon) = \frac{\epsilon - 1}{3 + \epsilon} < 0$$
and 
$$z_j^T\frac{\theta_i}{S} \geq -1 + \frac{\frac{2}{1+\epsilon}}{1 + \frac{2}{1+\epsilon}}(1-\epsilon) = \frac{-1 - 3\epsilon}{3+ \epsilon}.$$
Taken together, we have that 
$$\max_{i,j}|z_i^T\theta_j| \in \left[S\frac{1-\epsilon}{3+\epsilon}, S\frac{1 + 3\epsilon}{3+\epsilon} \right]. $$

\textbf{Step 2: Choosing $S$ to satisfy Theorem \ref{thm:lower_general}}

To invoke the result of Theorem \ref{thm:lower_general}, we require that $\mu(\theta_i^Tz_i) \geq 1-\delta_2$ and $\mu(\theta_i^Tz_j) \leq \delta_2$ for $j \neq i$ for $\delta_2$ defined therein. For the above construction, we require an $S$ that satisfies: 1) $\mu\left(S\frac{1-\epsilon}{3 + \epsilon}\right) \geq 1-\delta_2$ and 2) $\mu(S\frac{\epsilon - 1}{3 + \epsilon}) \leq \delta_2$. Clearly this can be achieved by taking $S\rightarrow \infty$. 
Using the fact that $\mu( -x) = 1-\mu(x)$ for the logistic function as well as its monotonicity, 
we see that both are satisfied for $S$ such that $\mu(S\frac{\epsilon - 1}{3 + \epsilon}) \leq \delta_2$. 

Note that $\mu(x)$ is invertible with inverse $\mu^{-1}(x) = \log\left(\frac{x}{1-x}\right)$. 
Hence 
$$S\geq \frac{3 + \epsilon}{\epsilon - 1}\mu^{-1}(\delta_2) = \frac{3 + \epsilon}{1-\epsilon}\log\left(\frac{1-\delta_2}{\delta_2}\right).$$
implies that for any $\delta_1$-PAC algorithm there exists a $\theta \in \Theta$ such that
$$\E_{\theta}[T_{\delta_1}] > \frac{n}{16} = \frac{1}{16}\lfloor e^{d\epsilon^2/4}\rfloor. $$

\textbf{Step 3: Choosing $S$ to control $1/\kappa_0$}

Any $S$ that satisfies the constraint in step 2 satisfies the conditions of Theorem \ref{thm:lower_general} and implies a sample complexity lower bounds. As $\kappa_0^{-1} = O\left(e^S\right)$, to have the tightest correspondence between $\kappa$ and $n$, we want $S$ as small as possible. Therefore, we take 
$$S = \frac{3 + \epsilon}{1-\epsilon}\log\left(\frac{1-\delta_2}{\delta_2}\right). $$
By construction, we have
$$\max_{i,j}|z_i^T\theta_j| \leq S\frac{1 + 3\epsilon}{3+\epsilon}~. $$
This implies that 
\begin{align*}
    \min_{i,j}\mu\left(z_i^T\theta_j\right) \geq \mu\left(-S\frac{1 + 3\epsilon}{3+\epsilon} \right) & = \mu\left(-\frac{1 + 3\epsilon}{3+\epsilon} \times \frac{3 + \epsilon}{1-\epsilon}\log\left(\frac{1-\delta_2}{\delta_2}\right) \right) \\
    &= \mu\left(-\frac{1 + 3\epsilon}{1-\epsilon} \log\left(\frac{1-\delta_2}{\delta_2}\right) \right) \\
    & = \frac{1}{1 + \left(\frac{1-\delta_2}{\delta_2}\right)^{\alpha_\epsilon}} =: \delta_3
\end{align*}
where in the final line we have defined $\alpha_\epsilon := \frac{1 + 3\epsilon}{1-\epsilon}$. 
Then, using $1/(1-\delta_3) \leq 2$ for all $n \geq 1$,
$$\frac{1}{\kappa_0} = \frac{1}{\delta_3(1-\delta_3)} \leq \frac{2}{\delta_3} = 2\left(1 + \left(\frac{1-\delta_2}{\delta_2}\right)^{\alpha_\epsilon}\right)~.$$


\textbf{Step 4: Putting it all together}

By Theorem~\ref{thm:lower_general}, the above holds for $\delta_2 \leq \frac{1}{2n}$. Choose $\delta_2 = \frac{1}{2n}$.
Hence, 
\begin{align*}
    \frac{1}{\kappa_0} \leq 2\left(1 + \left(\frac{1-\delta_2}{\delta_2}\right)^{\alpha_\epsilon}\right) \iff n \geq \frac{1}{2}\left(\frac{1}{2} \cdot \frac{1}{\kappa_0} - 1\right)^{-\alpha_\epsilon} + 1 = c_1 \left(\frac{1}{\kappa_0}\right)^{-\frac{1-\epsilon}{1+3\epsilon}}
\end{align*}
for an absolute constant $c_1$. 
Furthermore, there exists a constant $c_2$ such that 
$S = c_2\log(n)$.
Noting that $n = \Omega(e^{\epsilon^2d})$ completes the proof. 



\end{proof}

\begin{proof}[\textbf{Proof of Theorem}~\ref{thm:lower_general}]

Suppose not.
Then, for every $\th\in\Theta$, we have $\EE_\th [T_{\dt_1}] \le n/16$.


\textbf{Step 1: Defining event $\cE_i$ that leads to errors}

Let $R_t$ be the reward received at time $t$.
Let $\tau$ to be the first time $t$ the algorithm receives $R_t=1$: 
\begin{align}\label{eq:def-tau}
  \tau = \begin{cases}
        \infty & \text{if $\{t\in [1,T_{\dt_1}]: R_t = 1\}$ is empty}\\
      \min_{t\in [1,T_{\dt_1}]: R_t = 1}~ t & \text{otherwise}
  \end{cases}    
  ~. 
\end{align}
If $\tau=\infty$, then the algorithm only sees reward 0 until termination.
Note that $\tau$ is a stopping time.
Let $i(\th)$ be the best arm under the parameter $\th$.
Let $T_j$ be the number of pulls of arm $j$ up to (and including) time $T_{\dt_1}$.
We define the following event:
\begin{align*}
    \cE_{\nu,\th} = \left\{\mc{A} \text{ returns } i(\nu)\right\} 
        \cap \left\{T_{\delta_1} \leq \frac{n}{4} \right\} 
        \cap \left\{T_{i(\th)} = 0 \right\} 
        \cap \left\{\tau = \infty\right\}~,
\end{align*} 
which is bad and should not happen frequently when $\th$ is true but is likely to happen under $\nu$, roughly speaking.
Define 
$$
  \widehat{\KL}(\nu, \th) :=  \sum_{t=1}^{T_{\delta_1}}\sum_{z\in \mc{Z}} \1[z = a_{t}]\left[R_{t}\log\left(\frac{\mu(\nu^Tz)}{\mu(\th^Tz)} \right) + (1 - R_{t})\log\left(\frac{1 - \mu(\nu^Tz)}{1-\mu(\th^Tz)} \right)\right].
$$
On the event $\cE_{\nu,\th}$ using the assumptions of the theorem, we have that 
$$  
  \widehat{\KL}(\nu, \th) \leq \frac{n}{4}\log\left(\frac{1}{1 - \delta_2} \right)
$$
Then,
\begin{align}\label{eq:kappalb-core}
\begin{aligned}
  \P_\th(\mc{A} \text{ returns } i(\nu)) 
  &\geq \P_{\th}(\cE_{\nu,\th})
\\&= \E_{\nu}\left[\one\{\cE_{\nu,\th}\}\exp(-\widehat{\KL}(\nu, \th)) \right] 
\\& \geq \left(1-\delta_2\right)^{n/4}\E_{\nu}\left[\one\{\cE_{\nu,\th}\}\right]
\\&= \left(1-\delta_2\right)^{n/4}\P_{\nu}\left(\cE_{\nu,\th}\right)~.
\end{aligned}
\end{align}
Let us fix an arbitrary circular ordering of the members of $\Theta$.\footnote{ For example, a circular orderings for $\{a,b,c\}$ is $a \prec b \prec c \prec a$. Another example is $c \prec b \prec a \prec c$.}
Let $\sq(\th)$ be member of $\Theta$ that comes immediately after $\th$ in the order, and $\sp(\th)$ be member of $\Theta$ that comes immediately before $\th$.
Then,
\begin{align*}
    \fr1n \sum_{\th\in\Theta} \P_\th(\mc{A} \text{ returns } i(\sq(\th))) 
    &\ge \left(1-\delta_2\right)^{n/4} \fr1n \sum_{\th\in\Theta} \P_{\sq(\th)}\left(\cE_{\sq(\th),\th}\right)
        \tag{By \eqref{eq:kappalb-core}}
    \\&\stackrel{(a)}{=} \left(1-\delta_2\right)^{n/4} \fr1n \sum_{\nu\in\Theta} \P_{\nu}\left(\cE_{\nu,\sp(\nu)}\right)
    \\&\stackrel{(b)}{=} \left(1-\delta_2\right)^{n/4} \PP\del{ \cE_{\th,\sp(\th)} }~.
\end{align*}
where $(a)$ is simply a reindexing and $(b)$ is by treating the parameter $\th$ to be drawn from $\mathsf{Uniform}(\Theta)$.
Hereafter, $\PP$ and $\EE$ without subscripts are w.r.t. the measure on all the random variables including the prior on $\th$.
Note that the terminology `prior' is just for brevity in this proof and does not imply that our problem setup is Bayesian.


\textbf{Step 2: Bounding $\PP(\cE_{\th,\sp(\th)})$}

Hereafter, we shorten the notation $\cE_{\th,\sp(\th)}$ as $\cE_\th$.
We aim to find a lower bound on $\PP(\cE_\th)$.
For this, we upper bound $\PP(\wbar{\cE_\th})$:
\begin{align}\label{eq:kappalb-step2}
\begin{aligned}
  &\PP(\wbar{\cE_\th}) 
  \\&\le \PP\del{ \wbar{\mc{A} \text{ returns } i(\th)} }
    +\PP\del{ T_{\delta_1} > \frac{n}{4} }
    + \PP\del{ T_{\delta_1} \le \frac{n}{4}, \tau \neq \infty}
  + \PP\del{T_{\delta_1} \le \frac{n}{4}, \tau = \infty, T_{i(\sp(\th))} \ge 1) }
\end{aligned}
\end{align}
Note that $\PP\del{ \wbar{\mc{A} \text{ returns } i(\th)}} = \fr1n \sum_\th \PP_\th \del{ \wbar{\mc{A} \text{ returns } i(\th)}} \le \dt_1$ by the assumption of the theorem.
Also, by Markov's inequality, 
$$
  \P_\th\left(T_{\delta_1} > \frac{n}{4}\right) \leq \frac{\E_\th[T_{\delta_1}]}{\frac{n}{4}} \leq \frac{\frac{n}{16}}{\frac{n}{4}} = \frac{1}{4}
  \implies \PP\del{T_{\delta_1} > \frac{n}{4}} \le \fr14
$$
where the second inequality is by our assumption made for the sake of contradiction.
For the last term in~\eqref{eq:kappalb-step2},
\begin{align*}
    \PP\del{T_{\delta_1} \le \frac{n}{4}, \tau = \infty, T_{i(\sp(\th))} \ge 1 }
    &= \fr1n \sum_{\th} \PP_\th \del{T_{\delta_1} \le \frac{n}{4}, \tau = \infty, T_{i(\sp(\th))} \ge 1 }
    \\&= \fr1n \sum_{\th} \PP_\nu \del{T_{\delta_1} \le \frac{n}{4}, \tau = \infty, T_{i(\sp(\th))} \ge 1 }     \text{ for any $\nu\in\Theta$~,}
\end{align*}
where the last equality uses the fact that under $\tau=\infty$ the algorithm's behavior is independent of the unknown $\th$ because the algorithm's behavior is determined by its observed reward (and its internal randomization if any).
Therefore,
\begin{align}\label{eq:thelastterm}
\begin{aligned}
    \PP\del{T_{\delta_1} \le \frac{n}{4}, \tau = \infty, T_{i(\sp(\th))} \ge 1 }
    &= \fr1n \sum_{\th} \fr1n \sum_{\nu}\PP_\nu \del{T_{\delta_1} \le \frac{n}{4}, \tau = \infty, T_{i(\sp(\th))} \ge 1 } 
    \\&\stackrel{(a)}{=} \fr1n \sum_{j=1}^n \fr1n \sum_{\nu}\PP_\nu \del{T_{\delta_1} \le \frac{n}{4}, \tau = \infty, T_{j} \ge 1 } 
    \\&= \fr1n \sum_{j=1}^n \PP \del{T_{\delta_1} \le \frac{n}{4}, \tau = \infty, T_{j} \ge 1 }  
    \\&= \fr1n \sum_{j=1}^n \EE\sbr{ \one\cbr{T_{\delta_1} \le \frac{n}{4}, \tau = \infty, T_{j} \ge 1}} 
    \\&\le \fr1n \sum_{j=1}^n \EE\sbr{ \one\cbr{T_{\delta_1} \le \frac{n}{4}, \tau = \infty, T_{j} \ge 1} \cd T_j}    
    \\&= \EE\sbr{ \one\cbr{T_{\delta_1} \le \frac{n}{4}, \tau = \infty} \cd \fr1n \sum_{j=1}^n T_j}    
    \\&\le \fr1{4}
\end{aligned}
\end{align}
where $(a)$ is by reindexing.

It remains to bound the third term in~\eqref{eq:kappalb-step2}.
This case becomes a bit tricky for the following reasons.
We would like to use the event $T_{\dt_1}\le n/4$ in the same way as the display~\eqref{eq:thelastterm}, but, now that $\tau\neq \infty$, the random variable $T_{\dt_1}$ depends on the instance $\th$.
Thus, we cannot employ the same independence argument used above.
To get around, we construct a surrogate algorithm $\cA'$ that follows a $\delta_1$-PAC algorithm $\cA$ but still selects arms after $\cA$ terminates.
Specifically, for $t\le T_{\dt_1}$ the algorithm $\cA'$ follows the selection $a_t$ of $\cA$ exactly, and for $t > T_{\dt_1}$ the algorithm $\cA'$ employs an arbitrary policy.
For example, one can choose to select arms uniformly at random for $t > T_{\dt_1}$.
However, any policy works for our proof as long as the policy does not change as a function of $\th$.

Let $\chi_j$ be the time step $t\le \fr n4$ where $i(\th)$ is pulled for the first time:
\begin{align*}
    \chi_j :=
    \begin{cases}
        \infty  & \text{if $\cbr{t \le n/4: a_t = i(\th)}$ is empty }
      \\\min \cbr{t \le n/4: a_t = i(\th)} & \text{otherwise}
    \end{cases}
    ~.
\end{align*}
We also define $\tau'$ as the same as~\eqref{eq:def-tau} except that we replace $T_{\dt_1}$ therein with $n/4$.
Then, by introducing the notation $\PP^{\cA}$ to indicate the dependency on the algorithm $\cA$,
\begin{align}\label{eq:thethirdterm}
\begin{aligned}
      &\PP^{\cA} \del{T_{\dt_1}\le \fr n4, \tau \neq \infty}
    \\&\le \PP^{\cA'}\del{\tau' \neq \infty}
    \\&= \PP^{\cA'}\del{\tau' \neq \infty,~ \chi_{i(\th)} \le \tau'}
        + \PP^{\cA'}\del{\tau' \neq \infty,~ \chi_{i(\th)} > \tau'}
    \\&\stackrel{(a)}{\le} \PP^{\cA'}\del{a_{\chi_{i(\th)}} = i(\th),~ \chi_{i(\th)} \neq \infty,~ R_{1:(\chi_{i(\th)}-1)} = 0}
        + \PP^{\cA'}\del{\tau'\neq\infty,~ a_{\tau'} \neq i(\th),~ R_{\tau'} = 1}
\end{aligned}
\end{align}
where $(a)$ introduces the notation $R_{1:t} = 0$ for $R_i=0, \forall i\in[t]$.

Hereafter, we omit the dependence on $\cA'$ for brevity.
Denote by $T_j^{(n/4)}$ the number of pulls of arm $j$ up to (and including) time $n/4$.
The first term above is equal to
\begin{align*}
      &\fr1n \sum_\th \PP_\th\del{a_{\chi_{i(\th)}} = i(\th),~ \chi_{i(\th)} \neq \infty,~ R_{1:(\chi_{i(\th)}-1)} = 0}
    \\&= \fr1n \sum_\th \PP\del{a_{\chi_{i(\th)}} = i(\th),~ \chi_{i(\th)} \neq \infty,~ R_{1:(\chi_{i(\th)}-1)} = 0}
        \tag{independence}
    \\&= \fr1n \sum_{j=1}^n \PP\del{a_{\chi_{j}} = j,~ \chi_{j} \neq \infty,~ R_{1:(\chi_{j}-1)} = 0}
        \tag{reindexing}
    \\&\le \fr1n \sum_{j=1}^n \PP\del{T_j^{(n/4)} \ge 1}
    \\&\le \fr1n \sum_{j=1}^n \EE\sbr{T_j^{(n/4)}} \tag{Markov's inequality}
    \\&= \fr1n \cd \fr{n}{4} = \fr14
\end{align*}
where the reasoning is mostly the same as the display~\eqref{eq:thelastterm}.

The second term of the display~\eqref{eq:thethirdterm} is equal to 
\begin{align*}
    &\fr1n \sum_\th \PP_\th \del{\tau'\neq\infty,~ a_{\tau'} \neq i(\th),~ R_{\tau'} = 1}
    \\&=\fr1n \sum_\th \PP_\th \del{\exists t \le n/4: a_t \neq i(\th), R_t = 1, a_{1:t-1} \neq i(\th), R_{1:t-1}=0} 
    \\&\le \fr1n \sum_\th \sum_{t=1}^{n/4} \PP_\th \del{R_t = 1\mid a_t \neq i(\th), a_{1:t-1} \neq i(\th), R_{1:t-1}=0} \cd \PP_\th\del{ a_t \neq i(\th), a_{1:t-1} \neq i(\th), R_{1:t-1}=0 }
    \\&\le \fr1n \sum_\th \sum_{t=1}^{n/4}  \dt_2 \cd 1
    \\&\le \dt_2 \cd \fr n 4~.
\end{align*}



\textbf{Step 3: Putting it all together}

Using the results of the previous two steps, we have that
\begin{align*}
    \fr1n \sum_\th \P_\th(\mc{A} \text{ returns } i(\mathsf{q}(\th))) 
    &\ge  \left(1-\delta_2\right)^{n/4} \PP\del{ \cE_{\th} } 
    \\&\ge \left(1-\delta_2\right)^{n/4}\left(1 - \del{\dt_1 + \fr14 + \fr14 + \delta_2\cd \fr n 4}\right)
    \\&\ge \left(1-\fr1{2n}\right)^{n/4} \fr18 \tag{$\dt_1<\fr 1 {16} < \fr14,~  \dt_2 \le \fr{1}{2n}$}
    \\&\ge \frac{1}{16} \tag{$n \ge 20$}
    \\&> \dt_1 \tag{assumption}
\end{align*}
However, we have by design that $i(\mathsf{q}(\theta)) \neq i(\theta)$. 
As we have assumed that $\cA$ is $\delta_1$-PAC, we have the following contradiction, which concludes the proof:
\begin{align*}
    \fr1n \sum_\th \P_\th(\mc{A} \text{ returns } i(\mathsf{q}(\th))) \leq \fr1n \sum_\th \delta_1 = \delta_1~.
\end{align*}

\end{proof}

\subsubsection{Comparison to Theorem~\ref{thm:lowerbound} on the same instances}

The family of $\theta$'s is important to the statement of Theorems~\ref{thm:permuation_lower_bound} and \ref{thm:lower_general}. It captures the complexity of \emph{exploration} for logistic bandits and rules out pathological algorithms that have knowledge of $\theta^\ast$. Otherwise, if we restrict to finding the best arm and fix a set $\mc{Z}$ and a single $\theta^\ast \in \Theta$ as defined in in Theorems ~\ref{thm:permuation_lower_bound} and \ref{thm:lower_general}, 
then an oracle that knows $i(\theta^\ast)$ can put all of its samples on it. Indeed, the constraint on $\mc{Z}$ and $\theta \in \Theta$ imposed by Theorem~\ref{thm:permuation_lower_bound}
implies that $\mu(i(\theta)^T\theta) \geq 1-\delta_2$ and $\mu(z_j^T\theta) \leq \delta_2$ for all $z_j \neq i(\theta)$ and all $\theta\in \Theta$. Hence, for $\theta^\ast \in \Theta$, the set of alternates to $\theta^\ast$ for the bound given in Theorem~\ref{thm:lowerbound}
is $\Theta \backslash \{\theta^\ast\}$. For every $\theta'\in \Theta \backslash\{\theta^\ast\}$, $KL(\mu(i(\theta^\ast)^T\theta^\ast) || \mu(i(\theta^\ast)^T\theta') \geq KL(1-\delta_2 || \delta_2)) = \Omega(1)$. Therefore, the allocation 
\begin{align*}
    \lambda_a = \begin{cases}
        1 & \text{ if } a = i(\theta^\ast)\\
        0 & \text{ otherwise}
    \end{cases}
\end{align*}
which puts all of the samples on arm $i(\theta^\ast)$ is feasible for the optimization in Theorem~\ref{thm:lowerbound}. However, this would imply a lower bound of $KL(1-\delta_2 || \delta_2)^{-1} \log(1/2.4\delta) = O(\log(1/\delta))$ independent of both the dimension and the number of arms. Naturally, this is pathological since the allocation depends on knowledge of $\i(\theta)$. In order to rule out such pathological instances, we consider a family of $\theta$'s.


\section{Comments on \texorpdfstring{\citet{li2017provably}}{}}
\label{sec:li17}

\begin{figure}
    {\centering
\includegraphics[width=.4\linewidth]{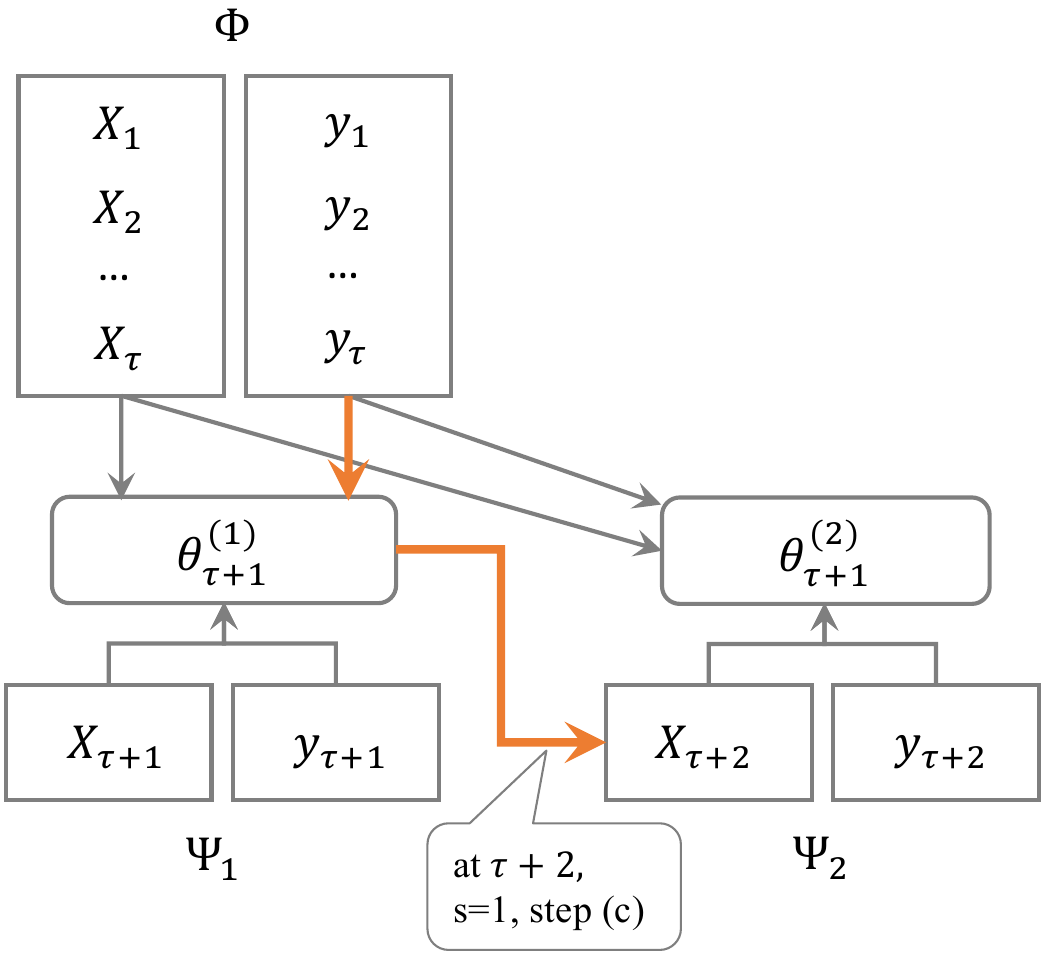}
\caption{A diagram showing the dependency of the variables in SupCB-GLM of \citet{li2017provably}. The troublesome dependency is colored orange with thick lines. Note that we did not show all the dependencies here to avoid clutter. For example, $X_{\tau+1}$ depends on $X_1,\ldots,X_{\tau}$.  }
\label{fig:li17}
    }
\end{figure}

\citet{li2017provably} collects the burn-in samples in a different way from our SupLogistic.
They collect burn-in samples (denoted by $\Phi$ here) for the first $\tau=\sqrt{dT}$ rounds, and the buckets $\Psi_1, \ldots, \Psi_S$ are empty at the beginning of time step $\tau+1$.
Then, at time $t>\tau$, when they compute the estimate $\th^{(s)}_t$, they use both the samples from $\Psi_s$ and $\Phi$.
However, we claim that this scheme invalidates the concentration inequality.
We explain below how this happens with the help of Figure~\ref{fig:li17}.
\begin{itemize}
    \item At time $\tau+1$, we choose $X_{\tau+1}$ in $s=1$ with step \textbf{(a)}.
    \item At time $\tau+2$, we pass $s=1$ and then choose $X_{\tau+2}$ in $s=2$ with step \textbf{(a)}. Note that the set of arms that has survived $s=1$ and passed onto $s=2$ are dependent on $\th^{(1)}_{\tau+1}$ that is a function of $y_1,\ldots, y_\tau$; see the orange thick line in Figure~\ref{fig:li17}. 
    \item At time $\tau+3$, we pass $s=1$, arrive at $s=2$ step \textbf{(c)}, and we perform the arm rejections using $\th^{(2)}_{\tau+2}$. 
    At this point, the fixed design assumption of our concentration inequality for $\th^{(2)}_{\tau+2}$ is violated as we describe below.
\end{itemize}
Let $\tau=1$ for convenience.
The estimator $\theta^{(2)}_{3}$ is computed based on $\{X_1, X_3, y_1, y_3\}$, but $X_3$ depends on $y_1$.
That is, $y_1\mid X_1$ is not conditionally independent from $X_3$.
Specifically,
\begin{align*}
    p(y_1,y_3\mid X_1,X_3) = \frac{p(y_1,y_3,X_1,X_3)}{p(X_1,X_3)} = \fr{p(y_3\mid X_3) p(y_1\mid X_1,X_3)p(X_1,X_3)}{p(X_1,X_3)} = p(y_3\mid X_3) p(y_1\mid X_1,X_3)
\end{align*}
where we use $p$ to denote both the PDF and PMF.
Note that $p(y_1\mid X_1,X_3) \neq p(y_1\mid X_1)$ in general; the distribution of $y_1\mid X_1, X_3$ is algorithm-dependent and thus hard to control.
To be clear, note that $p(y_1\mid X_1,X_3)$ must follow a Bernoulli distribution. It is just that we cannot guarantee that it has the mean $\mu(X_1^\T \th^*)$. 

Our algorithm SupLogistic circumvents this issue by collecting burn-in samples for each bucket $\Psi_1,\ldots,\Psi_S$, but there is another challenge in dealing with the confidence width that depends on $\th^{(s)}_{t}$ due to our novel variance-dependent concentration inequality.

\section{Proofs for SupLogistic}

While SupLogistic is inspired by the standard SupLinRel-type algorithms~\cite{auer2002using,li2017provably}, its design and analysis are challenging due to the fact that the confidence width scales with the unkonwn $\th^*$ unlike the standard linear bandit setting.
To get around this issue, we design the algorithm so that the mean estimate and the variance estimate are computed from different buckets -- $\Psi_s(t-1)$ and $\Phi$ respectively.
Such a design was necessary (as far as we stick to the SupLinRel type) because using the confidence widths based on $\th^{(s)}_t$ would introduce a dependency issue similar to what we described in Section~\ref{sec:li17} and invalidate our confidence bound.

While our regret bound improves upon the dependence on $\|\th^*\|$ in the leading term, we believe it should also be possible to incorporate recent developments of SupLinRel-type algorithms by \citet{li19nearly} to shave off the logarithmic factors from $\log^{3/2}(T)$ to $\log^{1/2}(T)$.
The focus our paper is, however, to show the impact of our novel confidence bounds.

We first define our notations for the proof.
\begin{itemize}
    \item We define a shorthand ${X_t} := x_{t,a_t}$ for the arm chosen at time step $t$.
    \item Denote by $\blue{\Psi_s(t)}$ the set of time steps at which the pulled arm $a_t$ was 
    included to the bucket $s$ up to (and including) time $t$. 
    In other words, $\Psi_s(t)$ is the variable $\Psi_s$ in the pseudocode of SupLogistic at the end of time step $t$.
    \item Define ${H^{(s)}_{t}(\th)} = \sum_{u\in \Psi_s(t)} \dmu(X_u^\T \th) X_u X_u^\T, 
    \forall s\in[S],$ and ${H^\Phi_{t}(\th)} := \sum_{u\in \Phi(\tau)} \dmu(X_u^\T \th) X_u 
    X_u^\T$. 
    We remark that this definition depends on three variables: bucket index, 
    time step, and the parameter $\th$ for computing the variance $\dmu(X_u^\T \th)$.
    Note that the bucket $\Phi$ is never updated after the time step $\tau$ and, so 
    we often use the notation $\Phi$ to mean $\Phi(\tau)$.
\end{itemize}

We present the proof by a bottom-up approach:
\begin{itemize}
  \item Lemma~\ref{lem:event} set up the basic lemma concerned with $T_0$, the smallest budget for which we can enjoy a regret guarantee, and the key stochastic events that hold with high probability.
  \item Lemma~\ref{lem:regret-inst} analyzes the instantaneous regret.
  \item Lemma~\ref{lem:regret-perbucket} analyzes the cumulative regret per bucket $\Psi_s(T)$.
  \item Theorem~\ref{thm:regret} analyzes the final cumulative regret.
\end{itemize}
Note that the proof of the final regret bound becomes a matter of invoking the Cauchy-Schwarz inequality and the elliptical potential lemma (Lemma~\ref{lem:epl} below), which is standard in linear bandit analysis.

We first begin with the condition under which the algorithm collects enough burn-in samples and guarantees concentration of measure after time step $\tau$.
\begin{lemma} \label{lem:event}
    Fix $\delta>0$. 
    Consider {\normalfont SupLogistic} with $T\ge d$, $\tau=\sqrt{dT}$, and $\alpha=3.5 \sqrt{\ln(2(2+\tau) \cd \fr{2STK}{\dt})}$. 
    Let $H^{(S+1)}_{t}(\th^*) := H^{\Phi}_{t}(\th^*)$.
    Define the following event: 
    \begin{align}    \label{eq:event}
        \begin{aligned}
            {\cE_{\mathsf{mean}}} &:= 
            \Bigg\{ 
              \forall t\in\cbr{\tau+1, \ldots, T}, a\in[K], s\in[S],~~ | x_{t,a}^\T\th^{(s)}_{t-1}  - x_{t,a}^\T\th^*| \le \alpha \|x_{t,a}\|_{(H^{(s)}_{t-1}(\th^*))^{-1}},
              \\ &\hspace{15em}
              \fr{1}{\sqrt{2.2}}\|x\|_{(H_{t-1}^{(s)}(\th^*))^{-1}}
              \le \|x\|_{(H_{t-1}^{(s)}(\th_\Phi))^{-1}} 
              \le \sqrt{2.2} \|x\|_{(H_{t-1}^{(s)}(\th^*))^{-1}}          
           \Bigg\}
            \\ 
            {\cE_{\mathsf{diversity}}} &:= \cbr{\forall s\in[S+1], \lammin(H^{(s)}_{\tau}(\th^*)) \ge d + \ln\del{6(2+\tau)\cd \fr{2STK}{\dt}}} ~.
        \end{aligned}
    \end{align}
    Then, there exists
    \begin{align}\label{eq:T0}
        T_0 = \Theta(Z\ln^4(Z)) \text{~~ where ~~} Z =\fr{d^3}{\kappa^2} + \fr{\ln^2(K/\dt)}{d\kappa^2}\cd   
    \end{align}
    such that $\forall T \ge T_0$, $\PP(\cE_{\mathsf{mean}}, \cE_{\msf{diversity}}) \ge 1-\dt$. 
\end{lemma}
\begin{proof}
    To avoid clutter, let us fix $s$ and drop the superscript from $H^{(s)}_\tau(\th^*)$ and use $H_\tau(\th^*)$.
    Note that each bucket has at least $\lfl\fr{\tau}{S+1}\rfl$ samples.
    Since $\lammin(H_\tau(\th^*)) \ge \kappa \lammin(V_\tau)$ where ${V_\tau} = \sum_{u=1}^\tau X_u X_u^\T$, to ensure $\cE_{\mathsf{diversity}}$, it suffices to show that $\lammin(V_\tau) \ge \kappa^{-1}\del{d + \ln\del{6(2+\tau)\cd \fr{2STK}{\dt}}} =: \blue{F}$. 
    Recall our stochastic assumption on the context vectors ${x_{t,a}}$, the definition of $\Sigma$, and our assumption $\lammin(\Sigma)\ge\sig_0^2$.
    By \citet[Proposition 1]{li2017provably}, there exists $C_1,C_2>0$ such that if
    \begin{align*}
        \lt\lfl\fr{\tau}{S+1}\rt\rfl
        &\ge \del{\fr{C_1\sqrt{d} + C_2 \sqrt{\ln(2STK/\dt)}}{\sig_0^2}}^2 + \fr{2}{\sig_0^2} \cd F~,
    \end{align*}
    then $\PP(\lammin(V) \ge F) \ge 1-\fr{\dt}{2STK}$.
    Since $\tau = \sqrt{dT}$, we have $T$ in both LHS and RHS.
    It remains to find the smallest $T$ that satisfies the inequality above.
    Omitting the dependence on $\sig_0^2$, one can show that it suffices to find a sufficient condition for $T$ such that 
    \begin{align}\label{eq:suffcond-div}
        T \ge \underbrace{\del{C_3 \cd \fr{d}{\kappa^2} + C_4\cd\fr{1}{d\kappa^2}\cd \ln^2(K/\dt)}}_{=:\blue{Z}}\ln^4(T)   
    \end{align}
    for some absolute constants $C_3$ and $C_4$.
    One can show that $T < Z \ln^4(T)$ implies $T < \Theta( Z\ln^4(Z))$ whose contraposition implies that there exists $T_0 = \Theta(Z\ln^4(Z))$ such that if $ T\ge T_0$, then \eqref{eq:suffcond-div} is true, which in turn implies that $\PP(\cE_{\mathsf{diversity}}) \ge 1- \fr{\dt}{2STK} \cd (S+1)$ via union bound over $s\in[S+1]$.
    
    When $\cE_{\mathsf{diversity}}$ is true, it is easy to see that the condition on $\nu_t$ in Theorem~\ref{thm:concentration-supp} is satisfied if we substitute $\dt \larrow \dt/(2STK)$.
    Thus, by the union bound, 
    \begin{align*}
        \PP(\cE_{\mathsf{mean}} \mid \cE_{\msf{diversity}}) \ge 1- \fr{\dt}{2STK}\cd STK~.
    \end{align*}
    Note that $\PP(A\cup B) = \PP((A\cap \bar B) \cup B) \le \PP(A\cap \bar B) + \PP(B) \le  \PP(A\mid \bar B) + \PP(B)$. 
    Setting $A = \bar\cE_{\msf{mean}}$ and $B = \bar\cE_{\msf{diversity}}$, we have
    \begin{align*}
        \PP(\bar\cE_{\msf{mean}} \cup \bar\cE_{\msf{diversity}} ) \le \fr{\dt}{2STK}\cd STK + \fr{\dt}{2STK}(S+1) \le \dt
    \end{align*}
    where the last inequality is by $K\ge2$.
\end{proof}

\begin{lemma}\label{lem:regret-inst}
  Take $\tau$ and $\alpha$ from Lemma~\ref{lem:event}.
  Recall that $M=1/4$.
  Suppose $\cE_{\msf{mean}}$.
  Consider the time step $t\ge \tau+1$. 
  Let $s_t$ be the while loop counter $s$ at which the arm $a_t$ is chosen.
  Let $a_t^* = \arg \max_{a\in[K]} \mu(x_{t,a}^\T \th^*)$ be the best arm at time $t$.
  Then, the best arm $a_t^*$ survives through $s_t$, i.e., $a_t^* \in A_s$ for all $s \le s_t$.
  Furthermore, we have
  \begin{align*}
      &\mu(x_{t,*}^\T \th^*) - \mu(x_{t,a_t}^\T \th^*) \\
      &\le 
      \begin{cases}
          \dot\mu(x_{t,a_t}^\T\th^*)8\cd2^{-s_t} + M\cd 64\cd2^{-2s_t} & \text{if $a_t$ is selected in step \textbf{(a)}}
          \\ \dot\mu(x_{t,a_t}^\T\th^*)2T^{-1/2} + M\cd 4\cd T^{-1} & \text{if $a_t$ is selected in step \textbf{(b)}}
      \end{cases}~.
  \end{align*}
\end{lemma}
\begin{proof}
    This proof is adapted from \citet[Lemma 6]{li2017provably} while keeping the dependence on the variance $\dot\mu(x_{t,a_t}^\T \th^*)$ to avoid introducing $\kappa^{-1}$ explicitly.
    Fix $t$.
    To avoid clutter, let us omit the subscript $t$ from $\cbr{x_{t,a},a_t^*}$ and use $\cbr{x_a, a^*}$ instead, respectively. 
    We also drop the subscript $t-1$ from $H^{(s)}_{t-1}(\th)$.
    Let us refer to the iteration index of the while loop as \textit{level}.
    We use the notation $m^{(s)}_a$ to denote $m_{t,a}$ at level $s$.
    
    We prove the first part of the lemma by induction. 
    For the base case, we trivially have $a^* \in A_1$.
    Suppose that $a^*$ has survived through the beginning of the $s$-th level (i.e., $a^* \in A_s$) where $s \le s_t-1$.
    We want to prove $a^* \in A_{s+1}$. 
    Since the algorithm proceeds to level $s+1$, we know from step \textbf{(a)} at $s$-th level that, $\forall a\in A_s$,
    \begin{align}\label{eq:induction}
        |m_{a}^{(s)} - x_{a}^\T \th^*| 
        \le \alpha \|x_a\|_{(H^{(s)}(\th^*))^{-1}} 
        \le \alpha \sqrt{2.2} \|x_a\|_{(H^{(s)}(\th_\Phi))^{-1}} \le 2^{-s}
    \end{align}
    where both inequalities are due to $\cE_{\mathsf{mean}}$. 
    Specifically, it holds for $a=a^*$ because $a^* \in A_{s}$ by our induction step. 
    Then, the optimality of $a^*$ implies that, $\forall a\in A_s$,
    \begin{align*}
        m_{a^*}^{(s)} 
        \sr{\eqref{eq:induction}}{\ge} x_{a^*}^\T\th^* - 2^{-s} 
        \ge x_{a}^\T\th^* - 2^{-s} 
        \sr{\eqref{eq:induction}}{\ge} m_{a}^{(s)} - 2\cdot2^{-s} ~.
    \end{align*}
    Thus we have $a^* \in A_{s+1}$ according to step \textbf{(c)}. 
    
    For the second part of the lemma, suppose $a_t$ is selected at level $s_t$ in step \textbf{(a)}. 
    If $s_t=1$, obviously the lemma holds because $\mu(z) \in (0,1), \forall z$. 
    If $s_t>1$, since we have proved $a^* \in A_{s_t}$, 
    Eq.~\eqref{eq:induction} implies that for $a \in \cbr{a_t, a^*}$,
    \begin{align*}
        |m_{a}^{(s_t-1)} - x_{a}^\T\th^*| \le 2^{-s_t+1}~.
    \end{align*}
    Step \textbf{(c)} at level $s_t-1$ implies
    \begin{align*}
        m_{a^*}^{(s_t-1)} -  m_{a_t}^{(s_t-1)} \le 2\cdot 2^{-s_t+1} ~.
    \end{align*}
    Combining the two inequalities above, we get 
    \begin{align*}
        x_{a_t}^\T\th^* \ge m_{a_t}^{(s_t-1)} - 2^{-s_t+1} 
        \ge m_{a^*}^{(s_t-1)} - 3\cdot 2^{-s_t+1} \ge x_{a^*}^\T\th^*- 4\cdot 2^{-s_t+1} ~.
    \end{align*}
    Recall that $M=1/4$ is an upper bound on $\ddot{\mu}(z)$.
    The inequality above implies that, using Taylor's theorem,
    \begin{align}\label{eq:reasoning}
      \begin{aligned}
        \mu(x_{a^*}^\T\th^*) - \mu(x_{a_t},\th^*) 
          &= \alpha(x_{a^*}^\T\th^*,x_{a_t}^\T\th^*) \cd (x_{a^*}-x_{a_t})^\T\th^*
          \\&\le \alpha(x_{a^*}^\T\th^*,x_{a_t}^\T\th^*) \cd 4 \cd 2^{-s_t + 1}
          \\&\le \del{\dot\mu(x_{a_t}^\T\th^*) + M\cd (x_{a^*}-x_{a_t})^\T\th^*} \cd 4 \cd 2^{-s_t + 1}
          \\&= \dot\mu(x_{a_t}^\T\th^*)4 \cd 2^{-s_t + 1} + M\cd (4 \cd 2^{-s_t + 1})^2
          \\&\le \dot\mu(x_{a_t}^\T\th^*)8\cd2^{-s_t} + M\cd 64\cd2^{-2s_t}~.
      \end{aligned}        
    \end{align}
    When $a_t$ is selected in step \textbf{(b)}, since $m_{a_t}^{(s_t)} \ge m_{a^*}^{(s_t)}$, we have
    \begin{align*}
        x_{a_t}^\T\th^* \ge m_{a_t}^{(s_t)} - 1/\sqrt{T} 
        \ge m_{a^*}^{(s_t)} - 1/\sqrt{T} \ge x_{a^*}^\T\th^*- 2/\sqrt{T} ~.
    \end{align*}
    We now apply a similar reasoning as~\eqref{eq:reasoning}, we have
    \begin{align*}
        \mu(x_{a^*}^\T\th^*) - \mu(x_{a_t}^\T\th^*) 
        \le \dot\mu(x_{a^*}^\T\th^*)2T^{-1/2} + M\cd 4\cd T^{-1}~.
    \end{align*}
\end{proof}

\begin{lemma}[Regret per bucket]\label{lem:regret-perbucket}
    Assume $\cE_{\mathsf{mean}}$ and take $\alpha$ from Lemma~\ref{lem:event}. Recall that $L=1/4$.
    Then, $\forall s\in[S]$, 
    \begin{align*}
        \sum_{t\in \Psi_s(T)\sm[\tau]} \mu(x_{t,*}^\T \th^*) - \mu(X_t^\T\th^*) 
        \le 18\sqrt{L} \cd \alpha \sqrt{|\Psi_s(T)| d \ln(LT/d)} + \fr{320 M\alpha^2}{\kappa}d \ln(LT/d)~.
    \end{align*}
\end{lemma}
\begin{proof}
    By Lemma~\ref{lem:regret-inst} and the fact that $\mu(z) \in (0,1), \forall z$, we have
    \begin{align*}
        \sum_{t\in \Psi_s(T)} \mu(x_{t,*}^\T \th^*) - \mu(X_t^\T\th^*)
        &\le \sum_{t\in \Psi_s(T)} 1\wedge \del{\dot\mu(X_t^\T\th^*) \cd 8 \cd 2^{-s} + 64 M\cd 2^{-2s}}
        \\&\le \del{\sum_{t\in \Psi_s(T)} 1\wedge  \dot\mu(X_t^\T\th^*) \cd 8 \cd 2^{-s}} + \del{ \sum_{t\in \Psi_s(T)} 1 \wedge  64 M\cd 2^{-2s}}
    \end{align*}
    where the last inequality is true by $1 \wedge (a+b) \le 1 \wedge a + 1 \wedge b$.
    For the first summation, we use $w_{t,a_t}^{(s)} > 2^{-s}$ due to step \textbf{(a)} of the algorithm:
    \begin{align*}
        \del{\sum_{t\in \Psi_s(T)} 1\wedge  \dot\mu(X_t^\T\th^*) \cd 8 \cd 2^{-s}}
        &\le \sum_{t\in \Psi_s(T)} 1\wedge \dot\mu(X_t^\T\th^*) \cd 8 \cd w_{t,a_t}^{(s)}
        \\&= \sum_{t\in \Psi_s(T)} 1\wedge \dot\mu(X_t^\T\th^*) \cd 8 \cd \alpha \sqrt{2.2} \|X_t\|_{(H^{(s)}_{t-1}(\th_\Phi) )^{-1}} \tag*{(Def'n of $w^{(s)}_{t,a_t}$)}
        \\&\le \sum_{t\in \Psi_s(T)} 1\wedge \dot\mu(X_t^\T\th^*) \cd 18 \cd \alpha \|X_t\|_{(H^{(s)}_{t-1}(\th^*))^{-1}} \tag{$\cE_{\mathsf{mean}}$}
        \\&\le \sum_{t\in \Psi_s(T)}  1\wedge 18\alpha\sqrt{L} \|\sqrt{\dot\mu(X_t^\T\th^*)} X_t\|_{(H^{(s)}_{t-1}(\th^*))^{-1}}
                \tag*{($\because \sqrt{\dmu(X_t^\T\th^*)} \le \sqrt{L}$)}
        \\&\sr{(a)}{\le} \sqrt{ |\Psi_s(T)| \sum_{t\in \Psi_s(T)} 1\wedge (18\alpha\sqrt{L})^2 \|\sqrt{\dot\mu(X_t^\T\th^*)} X_t\|^2_{(H^{(s)}_{t-1}(\th^*))^{-1}}}
        \\&\sr{(b)}{\le} 18\alpha\sqrt{L} \sqrt{|\Psi_s(T)| \cd d\ln\del{LT/d}}
    \end{align*}
    where $(a)$ by the Cauchy-Schwarz inequality and $(b)$ by Lemma~\ref{lem:epl}, with the fact that $(18\alpha\sqrt{L})^2 \ge \fr12$, and  $\lammin(H_{\tau}(\th^*)) \ge 1$ ($\because \cE_{\mathsf{diversity}}$).
    
    
    The second summation follows a similar derivation:
    \begin{align*}
        \sum_{t\in \Psi_s(T)} 1 \wedge 64M\cd 2^{-2s}
        &\le \sum_{t\in \Psi_s(T)} 1 \wedge 64M\cd \alpha^2\cd 2.2\|X_t \|^2_{(H_{t-1}(\th_\Phi))^{-1}}
        \\&\le \sum_{t\in \Psi_s(T)} 1 \wedge 64M\cd \alpha^2 \fr{\dot\mu(X_t^\T\th^*)}{\dot\mu(X_t^\T\th^*)} 5\|X_t \|^2_{(H_{t-1}(\th^*))^{-1}}
        \\&\le \sum_{t\in \Psi_s(T)} 1\wedge \fr{64M \alpha^2}{\kappa}  5\cd \lt\|\sqrt{\dot\mu(X_t^\T\th^*)} X_t \rt\|^2_{(H_{t-1}(\th^*))^{-1}}
        \\&\le \fr{320M\alpha^2}{\kappa}d\ln(LT/d)
    \end{align*}
    where $(a)$ is by Lemma~\ref{lem:epl} and $\fr{320M\alpha^2}{\kappa} \ge 320M^2\alpha^2/L \ge 1/2$, and  $\lammin(H_{\tau}(\th^*)) \ge 1$.
\end{proof}

Finally, we prove the regret bound of SupLogistic below.
Note that the statement here is slightly different from that of the main paper because we state the regret bound for large enough $T$ only.
This is only an aesthetic difference.
Indeed, assume that we have a regret bound $\mathsf{Reg}_T \le A\sqrt{T} + B$ for $T \ge C$.
Then, using $\Reg_T \le T$, we have $\Reg_T \le C$ for $T < C$.
This implies that, for all $T$, we have $\Reg_T \le A\sqrt{T} + B + C$.
\begin{theorem}[Regret of SupLogistic]\label{thm:regret}
    Consider {\normalfont SupLogistic} with $\tau$, $\alpha$, and $T_0$ from Lemma~\ref{lem:event}. 
    Then, if $T \ge T_0$, then
    \begin{align*}
        \mathsf{Reg}_T 
        &\le 10 \alpha\sqrt{dT\ln(T/d)\log_2(T)} + O\del{\fr{\alpha^2}{\kappa} d\cd 
            (\ln(T/d))\cd\ln T } ~.
    \end{align*}
\end{theorem}
\begin{proof}
    Recall that $\Psi_0$ contains the time step indices at which the choice $a_t$ happened in 
    step 
    \textbf{(b)}. 
    Recall that we set $\tau=\sqrt{dT}$.
    Let $\blue{\Delta_t} := \mu(x_{t,*}^\T \th^*) - \mu(x_{t,a_t}^\T \th^*)$.
    Then,
    \begin{align*}
        R_T 
        &= \sum_{t=1}^{\tau} \Delta_t
        + \sum_{t=\tau+1}^{T} \Delta_t \\
        &\le  \sqrt{dT} + \sum_{t \in \Psi_0(T)}  \Delta_t
        + \sum_{s=1}^{S}  \sum_{t \in \Psi_s(T)\sm[\tau]} \Delta_t~.
    \end{align*}
    For the first term, using Lemma~\ref{lem:regret-inst},
    \begin{align*}
        \sum_{t \in \Psi_0(T)}  \Delta_t 
        \le T \cd \del{\dmu(X_t^\T\th^*) \cd \fr{2}{\sqrt{T}} + \fr{4M}{T}}
        \le 2L\sqrt{T} + 4M~.
    \end{align*}
    For the second term, using Lemma~\ref{lem:regret-perbucket}, using the Cauchy-Schwarz inequality,
    \begin{align*}
        \sum_{s=1}^{S}  \sum_{t \in \Psi_s(T)\sm[\tau]} \Delta_t
        &\le \sum_{s=1}^S \del{18\alpha \sqrt{L|\Psi_s(T)| d \ln(LT/d)} + \fr{320 
                M\alpha^2}{\kappa}d \ln(LT/d)}
        \\&\le 18\alpha \sqrt{L d \ln(LT/d)} \sqrt{S\sum_{s=1}^S |\Psi_s(T)|} + S\cd\fr{320 
            M\alpha^2}{\kappa}d \ln(LT/d)
        \\&\le 18\alpha \sqrt{L d \ln(LT/d) }\cd \sqrt{T\log_2(T)} + \log_2(T)\cd\fr{320 
            M\alpha^2}{\kappa}d \ln(LT/d)~.
    \end{align*}
    Using $L = 1/4$ and $\alpha \ge 7$, the terms involving $\sqrt{T}$ is: $\sqrt{dT} + 2L\sqrt{T} + 18\alpha \sqrt{L d 
        T\ln(LT/d)\log_2(T)} \le 10 \alpha\sqrt{dT\ln(T/d)\log_2(T)}$.
    This concludes the proof.
\end{proof}

\subsection{Auxiliary Results}

\begin{lemma}[Elliptical potential]\label{lem:epl}
    Let $s \in [S]$ and $F>0$. Then,
    \begin{align*}
        \sum_{t \in \Psi_s(T)} \min\cbr{1, F \| \sqrt{\dot\mu(X_t^\T\th^*)} X_t\|^2_{(H^*(\Psi_s(t-1)))^{-1}} }
        \le (2F \vee 1) \cd d\ln\del{\fr{L |\Psi_s(T)|}{d \lammin(H^*(\Psi_s(\tau)))}}~.
    \end{align*}
\end{lemma}
\begin{proof}
    Using Lemma 3 of \citet{jun2017scalable}, we have that $\forall q,x>0, \min\{q,x\} \le \max\cbr{2,q}\ln(1+x)$.
    Thus, 
    \begin{align*}
        \min\cbr{1,  F\| \sqrt{\dot\mu(X_t^\T\th^*)} X_t\|^2_{(H^*(\Psi_s(t)))^{-1}} }
        &=F\min\cbr{\fr1{F},~ \| \sqrt{\dot\mu(X_t^\T\th^*)} X_t\|^2_{(H^*(\Psi_s(t)))^{-1}} }
        \\&\le F \cd \max\cbr{2,\fr{1}{F}} \ln\del{1 + \| \sqrt{\dot\mu(X_t^\T\th^*)} X_t\|^2_{(H^*(\Psi_s(t)))^{-1}}  }
    \end{align*}
    where the last inequality is by $\max\{2, 1/F\} = 2$.
    Then,
    \begin{align*}
        \sum_{t \in \Psi_s(T)} \min\cbr{1, F \| \sqrt{\dot\mu(X_t^\T\th^*)} X_t\|^2_{(H^*(\Psi_s(t)))^{-1}} }
        &\le (2F\vee1) \sum_{t \in \Psi_s(T)} \ln\del{1 + \| \sqrt{\dot\mu(X_t^\T\th^*)} X_t\|^2_{(H^*(\Psi_s(t)))^{-1}}  }
        \\&= (2F\vee1) \ln\del{\fr{|H^*(\Psi_s(T))|}{|H^*(\Psi_s(\tau))|}}
        \\&\le (2F\vee1) \cd d\ln\del{\fr{L|\Psi_s(T)|}{d \cd \lammin(H^*(\Psi_s(\tau)))}}
    \end{align*}
    where the last inequality is by the arithmetic-geometric mean inequality.
\end{proof}

\section{Some Additional Empirical Insights}
In this section we dive a bit more into what RAGE-GLM is doing in its process of sampling. We utilized the Zappos pairwise comparison dataset \cite{finegrained}, focusing on the most ``pointy'' setting. As in the main text, we used the collected pairwise comparisons to learn $\theta^*$. We then sampled 10,000 shoes randomly from the set of 50,000 to be our $\mc{Z}$ set, and an additional randomly chosen 3,000 pairs of pairs of shoes from $\mc{Z}$ to be our $\mc{X}$ set (after being PCA-ed down to 25 dimensions). Given a query from the $x\in \cX$ set we then hallucinate the response using $\P(y = 1) = \mu(x^{\top}\theta^{\ast})$ according to the logistic model.

RAGE-GLM required roughly $5\times 10^6$ samples to complete. In the following figure, the top row shows the allocation of RAGE-GLM over $\cX$ (where we have sorted the elements of $\cX$ by the number of times each element was pulled). As we can see only a few hundred items received any pulls. 

\begin{figure}[h]
    \centering
    \includegraphics[scale=.5]{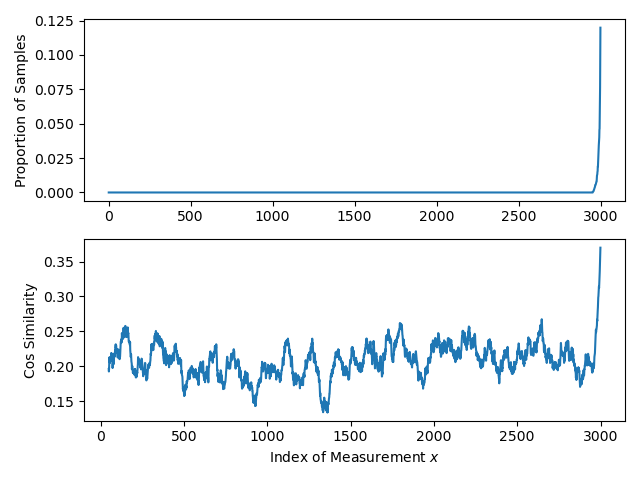}
    \caption{The top row is the allocation of samples over the pairs in $\cX$, the bottom row is the cosine similarity of these items with the difference.}
    \label{fig:allocations}
\end{figure}
 
Within a few tens of thousands of samples RAGE-GLM was able to narrow down to the top two different pair of shoes and then spent it's sampling budget differentiating between them. The second row of Figure~\ref{fig:allocations} shows the absolute value of the cosine similarity between the elements of $\cX$ and the difference between the top two pairs of shoes (we show a windowed average over 50 items to mitigate some spikes). As we can see, the most sampled items have higher average cosine similarity with the difference between the last two items. This has an interpretation in terms of \emph{analogies} - the algorithm focuses on two pairs of shoes in $\cX$ whose difference is most aligned with that of the best two pairs of shoes $\mc{Z}$. 

\begin{figure}[h!]
    \centering
    \includegraphics[scale=.6]{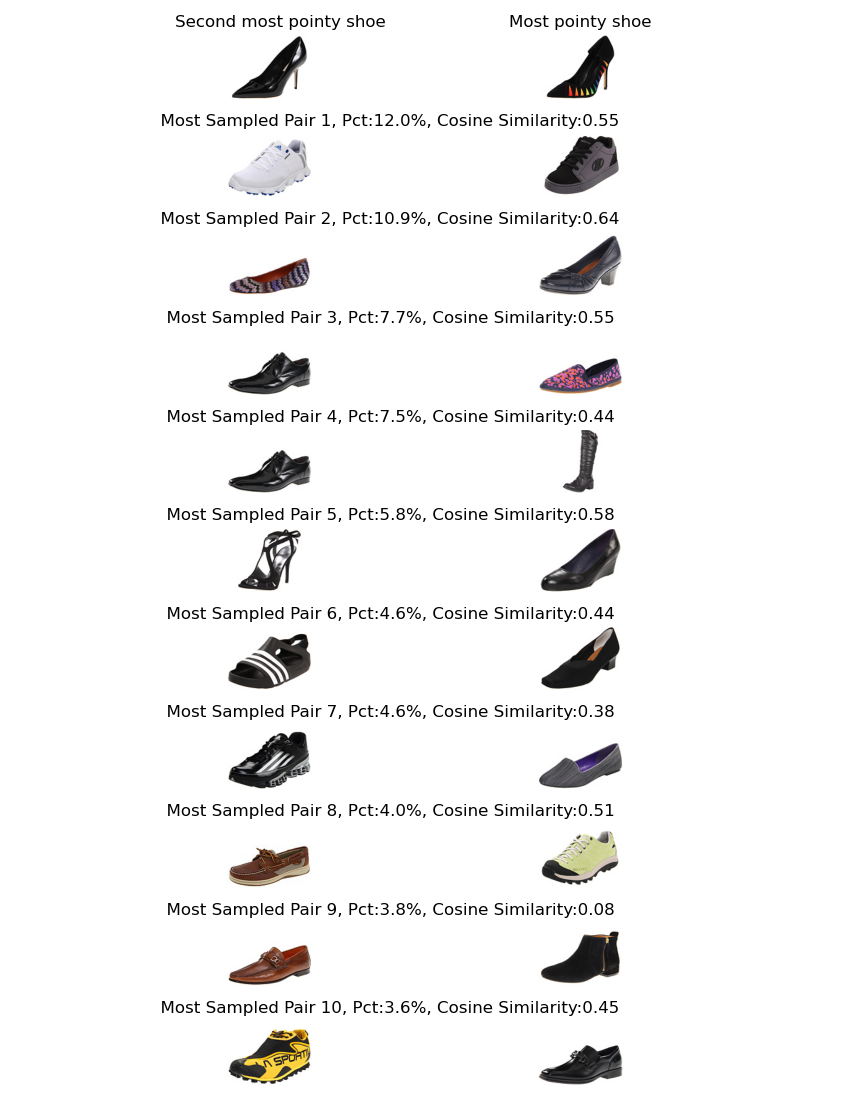}
    \caption{Most queried pairs of pairs of shoes.}
    \label{fig:allocations}
\end{figure}

In Figure~\ref{fig:allocations} we show the top two pairs of shoes along with the 10 pairs that were most queried (corresponding to the allocations above). Even though the pairwise comparisons shown may not be similar to those of the pointed shoes, asking them still gives us information about the difference between the best two, illuminating the advantage of the transductive setting.

%



\end{document}